\documentclass{article}
\PassOptionsToPackage{table,dvipsnames}{xcolor}
\usepackage{iclr2027_conference,times}

\usepackage{amsmath,amsfonts,bm}

\def\eqref#1{equation~\ref{#1}}

\def\1{\bm{1}}

\DeclareMathAlphabet{\mathsfit}{\encodingdefault}{\sfdefault}{m}{sl}
\SetMathAlphabet{\mathsfit}{bold}{\encodingdefault}{\sfdefault}{bx}{n}

\usepackage{url}
\usepackage{comment}
\usepackage{lipsum}
\usepackage{booktabs}
\usepackage{wrapfig,booktabs,multirow}

\usepackage[utf8]{inputenc}
\usepackage[T1]{fontenc}
\usepackage{microtype}

\usepackage{amsmath, amssymb, amsfonts, amsthm, mathtools}
\usepackage{mathrsfs}
\usepackage{nicefrac}

\usepackage{graphicx}
\usepackage{subcaption}
\usepackage{wrapfig}
\usepackage{booktabs}
\usepackage{tabularx}
\usepackage{array}
\usepackage{siunitx}
\usepackage{capt-of}
\usepackage{longtable}

\usepackage{multirow}
\usepackage{booktabs}
\usepackage{graphicx}   
\usepackage{xcolor}
\usepackage{pdflscape}  
\usepackage[table,dvipsnames]{xcolor}
\usepackage{colortbl}

\definecolor{burgundy}{HTML}{77324c}
\definecolor{tolblue}{HTML}{4477AA}
\definecolor{ours}{HTML}{228833}
\definecolor{baseline}{HTML}{BBBBBB}
\definecolor{oracle}{HTML}{4477AA}

\definecolor{tent}{HTML}{77324c}
\definecolor{eata}{HTML}{9B5774}
\definecolor{read}{HTML}{C07B8E}
\definecolor{sumi}{HTML}{5A2638}

\definecolor{sotta}{HTML}{CCBB44}
\definecolor{aeo}{HTML}{AA3377}
\definecolor{pta}{HTML}{882255}

\definecolor{textbody}{HTML}{333333}
\definecolor{textmuted}{HTML}{666666}

\definecolor{rowours}{HTML}{E3F0E5}
\definecolor{roworacle}{HTML}{DCE6F0}
\definecolor{rowdeg}{HTML}{F0E0E6}
\definecolor{rowgray}{HTML}{F5F5F5}

\usepackage{hyperref}
\hypersetup{colorlinks=true}
\definecolor{LinkColor}{HTML}{A04000}   
\definecolor{CiteColor}{HTML}{A04000}   
\definecolor{URLColor}{HTML}{AF601A}    
\definecolor{rowgray}{HTML}{F9ECE8}

\hypersetup{
    colorlinks=true,
    linkcolor=LinkColor,
    citecolor=CiteColor,
    urlcolor=URLColor,
    linktocpage=true,
    unicode=true,
    pdfborder={0 0 0},
}

\hypersetup{
    linkcolor=burgundy,
    citecolor=tolblue,
    urlcolor=burgundy,
    filecolor=burgundy,
    menucolor=burgundy,
    linktocpage=true,
    unicode=true,
    pdfborder={0 0 0},
}

\usepackage{natbib}

\usepackage[capitalize,noabbrev]{cleveref}

\usepackage{enumitem}
\usepackage{pifont}
\usepackage{gensymb}
\usepackage{adjustbox}
\usepackage[normalem]{ulem}
\useunder{\uline}{\ul}{}
\usepackage{makecell}

\usepackage{lipsum}
\usepackage{rotating}
\theoremstyle{plain}
\newtheorem{theorem}{Theorem}[section]
\newtheorem{proposition}[theorem]{Proposition}

\newtheorem{corollary}[theorem]{Corollary}

\theoremstyle{definition}

\theoremstyle{remark}

\usepackage{algorithm}
\usepackage{algpseudocode}

\usepackage{amsmath}      
\usepackage{amssymb}      
\usepackage{mathtools}    

\newcommand{\method}{\textsc{SeMARC}}
\newcommand{\backbone}{\textsc{SeMA}}
\newcommand{\policy}{\textsc{ARC}}

\title{Efficient Multimodal Inference through Adaptive Acquisition and Sequential Fusion}

\author{
\textbf{Payal Mohapatra}$^{1}$\thanks{\small Equal contribution; shared first authorship.}\hspace{0.7em}%
\textbf{Haodong Yang}$^{2}$\footnotemark[1]\hspace{0.7em}%
\textbf{Yueyuan Sui}$^{2}$\hspace{0.4em}%
\textbf{Stephen Xia}$^{2}$\hspace{0.4em}%
\textbf{Benjamin Lundell}$^{1}$\hspace{0.4em}%
\textbf{Qi Zhu}$^{2}$ \\[0.6em]
$^{1}$Arm Inc., USA \quad
$^{2}$Northwestern University, USA
}

\usepackage{wrapfig}
\newsavebox{\diagclassbox}
\newsavebox{\diagavailabilitybox}
\newsavebox{\diagrewardbox}

\iclrfinalcopy 
\begin{document}

\maketitle

\begin{abstract}
Multimodal systems often encode every available input, even when a subset suffices for prediction. Adaptive acquisition can reduce this cost by using predictions from incrementally fused evidence to \emph{decide which modality to encode next and when to stop}. However, sequential fusion makes these predictions order-dependent, so decisions based on them may need to distinguish factorially many histories of the same acquired set. We introduce \method{}, which couples a Sequential Modality Aggregator (\backbone{}) with an Adaptive Runtime Controller (\policy{}) and uses acquired evidence to select each modality \emph{before} its encoder runs. \backbone{} executes only selected encoder and fusion branches, updates a fixed-size state, and predicts after each acquisition without recomputing earlier branches. We supervise every acquisition prefix under randomized modality subsets and orders to encourage consistent predictions across acquisition orders. \policy{} combines a set-dependent marginal-utility prior with residual fitted-Q learning to select the next available modality or stop, without inspecting unacquired inputs or retaining acquisition order. Across six multimodal classification datasets and eleven baselines, \method{} achieves 3.2\% higher macro-F1 and 61.4\% lower total inference GFLOPs on average relative to each dataset's most accurate baseline. End-to-end latency falls by 44.0\% across GPU and CPU and by 47.2\% on Android INT8 relative to the fastest measured baseline, on average. Under varying runtime modality missingness, \method{} still skips available modalities, matching or exceeding the best baseline macro-F1 in 21 of 24 conditions with 14.8\% lower total GFLOPs on average. \method{} thus offers a practical path toward efficient multimodal inference across heterogeneous devices.
\end{abstract}

\section{Introduction}
\label{sec:intro}

Multimodal applications increasingly integrate video, audio, language, inertial measurements, depth, and wireless signals~\citep{liang2021multibench,yang2023mm}. Most frameworks use modality-specific encoders followed by fusion, with substantial research on handling noisy~\citep{zhang2023provable} or missing~\citep{mohapatra2026maestro, mohapatra24_interspeech, mohapatra2023person, wang2023multi, NEURIPS2024_7d62a85e} inputs. However, processing every modality adds computation, energy, and latency as modality counts and deployment scales grow~\citep{sui2026modality}. This raises a complementary question: which available modalities need to be processed for each input? Figure~\ref{fig:teaser} illustrates this opportunity.


\begin{figure}[!htbp]
    \centering
    \vspace{-6pt}
    \includegraphics[width=0.85\linewidth]{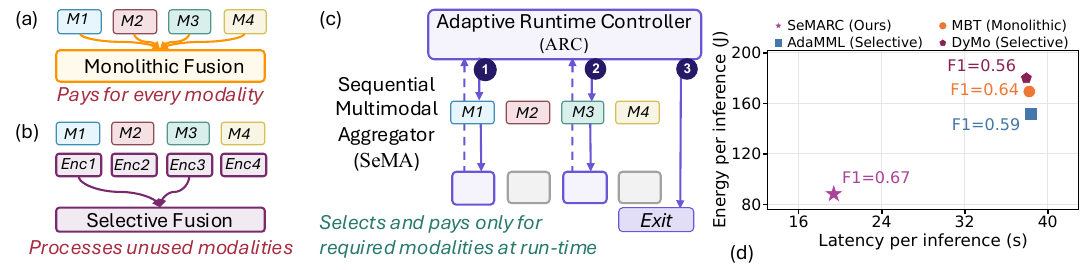}
    \captionsetup{font=small,skip=6pt}
    \vspace{-6pt}
    \caption{\textbf{Select before processing to avoid its cost.} (a) Monolithic fusion processes every modality. (b) Post-encoding selection pays for unused representations. (c) \method{} couples \policy{} and \backbone{} to acquire and incrementally fuse modalities or stop, executing only selected branches. (d) On Android (INT8), \method{} achieves lower latency and energy and higher $F_1$ (annotated) than the shown baselines.}
    \vspace{-8pt}
    \label{fig:teaser}
\end{figure}

Evidence from works on submodular modality selection~\citep{cheng2022greedy} and\begin{wrapfigure}{r}{0.37\textwidth}
    \centering
    \vspace{-14pt}
    \includegraphics[width=\linewidth]{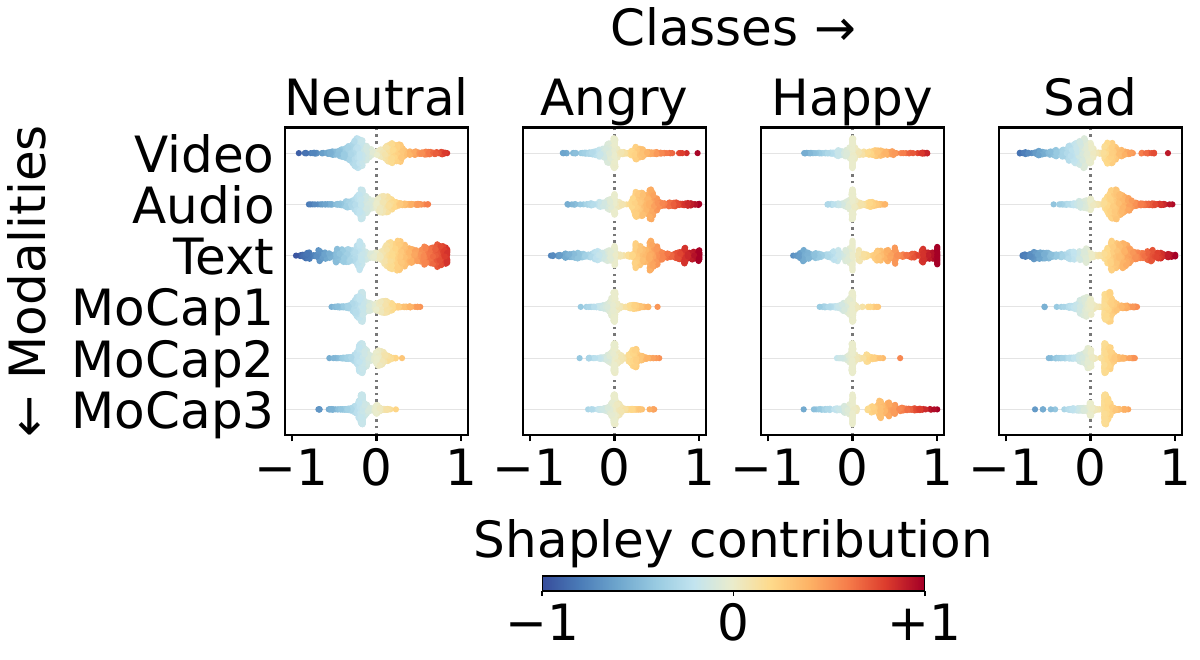}
    \captionsetup{font=scriptsize,skip=3pt}
    \vspace{-16pt}
    \caption{\textbf{Modality contributions vary across samples.} Class-wise Shapley contributions on IEMOCAP motivate sample-adaptive acquisition.}
    \vspace{-17pt}
    \label{fig:shap}
\end{wrapfigure} information-theoretic analysis~\citep{liang2023quantifying} suggests that overlapping task-relevant information can make a subset of modalities sufficient for prediction. Recent modality interaction analyses further reveal sample-level variations~\citep{yang2025efficient}, motivating sample-adaptive modality use. Even a simple Shapley analysis (Figure~\ref{fig:shap}) reveals variation in modality contributions across classes and among samples within each class. However, these analyses typically assess modality value post hoc: target-conditioned interaction estimates and correctness-based Shapley scores rely on ground-truth labels~\citep{yang2025efficient, covert2020understanding}, while prediction-based attributions require access to the inputs being explained~\citep{covert2022learning}. This motivates adaptive runtime estimation of modality value from acquired evidence, enabling the system to select the next modality or stop before incurring unnecessary computation.

But translating selective use into computational savings requires attention to when selection occurs (Figure~\ref{fig:teaser}a--c) and how the fusion is carried out subsequently. Selective fusion based on encoded features or lightweight modality previews can reduce downstream computation~\citep{xue2023dynamic,panda2021adamml}. However, it still incurs the cost of obtaining those modality features, including from modalities ultimately discarded. Active feature acquisition works address which inputs to query and when to stop~\citep{shim2018joint}, but reducing acquisition cost alone does not ensure reuse of multimodal computation across successive modalities. Motivated by these complementary directions, we propose to couple modality acquisition decisions with incremental multimodal fusion: selecting each modality before executing its encoder and fusion branch for every sample, reusing accumulated computation as evidence arrives, and stopping when further processing is no longer worthwhile.

We propose \textbf{\method{}}, a framework for sample-adaptive modality acquisition and incremental multimodal fusion (Figures~\ref{fig:teaser}c and~\ref{fig:sys_diag}). The system must choose the next modality or stop while reusing earlier computation. Yet incremental fusion can yield different predictions across orders of the same acquired evidence, and explicitly accounting for this dependence can require distinguishing factorially many acquisition histories. Our \textbf{Sequential Modality Aggregator (\backbone{})} reuses a compact bottleneck state and predicts after each acquisition. Training combines a modality-dropout curriculum, randomized orders, and prediction supervision with distillation at each step, encouraging consistent predictions across orders of the same acquired evidence. We relate prediction disagreement to prefix risk and use the approximate information state framework to characterize when a compact summary can serve as an approximately Markovian state. Our \textbf{Adaptive Runtime Controller (\policy{})} uses current predictions, acquired-set identity, and modality availability without explicitly retaining acquisition order. Trained on the frozen \backbone{}, it combines an offline, set-dependent marginal-utility prior with standard residual fitted-$Q$ learning to select the next modality or stop. At runtime, each acquisition executes only the selected encoder and fusion branch, updating the prediction for the next decision without rerunning earlier branches or processing unselected modalities. Our contributions are threefold:

\begin{itemize}[leftmargin=12pt, itemsep=0pt, topsep=0pt, partopsep=0pt, label=\raisebox{0.25ex}{\tiny$\blacktriangleright$}]
\item \textbf{Coupled incremental fusion and acquisition.} We introduce \method{}, which couples incremental bottleneck fusion (\backbone{}) with an acquisition controller (\policy{}) through prefix predictions trained to be consistent across acquisition orders. Acquisition and stopping decisions govern incremental encoder and fusion execution, reusing earlier computation and leaving unselected branches unexecuted.


\item \textbf{Design rationale and supporting analysis.} We examine order sensitivity in sequential fusion and motivate randomized-order prefix supervision. Supporting analysis relates population prediction risk to order disagreement and separately identifies the additional assumptions required for compact-state control.

\item \textbf{Broad empirical validation.} Across six datasets and eleven baselines, \method{} averages 61.4\% lower total GFLOPs and 3.2\% higher macro-F1 relative to each dataset's most accurate baseline, with practical latency gains on GPU, CPU, and Android. Under 10 to 40\% test-time missingness, it retains the best macro-F1 in 21 of 24 conditions while skipping modalities that are present, at 14.8\% lower total GFLOPs than the best baseline, so the savings come from selection rather than from missing inputs alone. 
Further component replacements, ablations, and diagnostics examine the proposed coupled design and acquisition behavior.
\end{itemize}

\section{Related Works} \label{sec:related_works}

\noindent\textbf{Efficient and incomplete multimodal inference.} Multimodal Bottleneck Transformer (MBT)~\citep{nagrani2021attention} limits cross-modal communication through bottleneck tokens; ShaSpec~\citep{wang2023multi} and MAESTRO~\citep{mohapatra2026maestro} handle incomplete inputs; MMEE~\citep{hamed2024multimodal} exits at early layer; DyMo~\citep{du2026inference} selects reliable recovered modalities. AdaMML~\citep{panda2021adamml} selects modalities from inexpensive previews; DynMM~\citep{xue2023dynamic} gates modality/fusion paths; HCMS~\citep{weng2023hcms} gates costly modalities in an audio-first hierarchy; multimodal early exiting~\citep{hamed2024multimodal} adapts depth. Our \method{} selects each subsequent available modality from acquired evidence before encoding and fusion, reusing accumulated state.

\noindent\textbf{Static and adaptive modality acquisition.} Greedy selection~\citep{cheng2022greedy,he2024efficient} uses approximate submodularity for cardinality-constrained, dataset-level modality subsets. Joint-AFA~\citep{shim2018joint} jointly learns cost-sensitive acquisition, stopping, and classification with order-invariant set encoding. A2MT~\citep{kossen2022active} schedules observations per timestep using predictive feedback over a fixed horizon, balancing prediction quality and acquisition costs while excluding computation costs. Concurrent, independently developed EAMA~\citep{peronimore} iteratively acquires unobserved modality subsets by estimated utility, stopping when none offers positive net utility. Complementing these methods, \method{} couples acquisition and incremental fusion per sample, using each updated prefix prediction to choose one whole modality branch or \textsc{Stop} without rerunning earlier branches.

\noindent\textbf{Sequential fusion and order robustness.} MoDN~\citep{trottet2023modular} and MultiModN~\citep{swamy2023multimodn} reuse state for intermediate predictions; MultiModN skips absent modules, evaluates a fixed order, and discusses permutation training. Deep Sets~\citep{NIPS2017_f22e4747} constructs invariant set functions; Janossy pooling~\citep{murphy2018janossy} studies symmetrization of order-sensitive functions, including random permutation training. \method{} trains sequential fusion over randomized modality subsets and orders with prefix supervision, relates expected prediction disagreement to excess prefix cross-entropy risk, and evaluates only the realized acquisition path at inference. Separately, we specialize the approximate information state framework~\citep{subramanian2022ais} to characterize the additional reward and action-conditioned transition assumptions needed for control from compact predictive summaries.

\section{Methodology}\label{sec:method}


We aim to improve inference efficiency by selectively processing available modalities while preserving predictive performance. For each sample $\bm{x}$, let $\mathcal{A}\subseteq\{1,\ldots,M\}$ denote the available modalities. A controller decides which available modality to process next and when to stop, using only acquired evidence and modality availability. These decisions precede the execution of each selected branch, avoiding the cost of encoding and fusing unselected modalities.

Incremental fusion enables these sequential decisions to reuse previous computation, but can make predictions depend on acquisition order. Ideally, different processing orders of the same sample and acquired modality set should yield consistent predictions. This motivates co-designing the fusion network and controller around an order-robust summary of the acquired evidence.

\method{} addresses these requirements by coupling \backbone{}, which incrementally fuses selected modalities and is trained for order-robust prediction after each acquisition, with \policy{}, which uses the resulting summary and modality availability to select the next modality or stop. As shown in Figure~\ref{fig:sys_diag}, each decision precedes the selected branch's execution, leaving unselected encoder and fusion branches unexecuted.

\begin{figure}
    \centering
    \captionsetup{font=small,skip=4pt}
    \includegraphics[width=\linewidth]{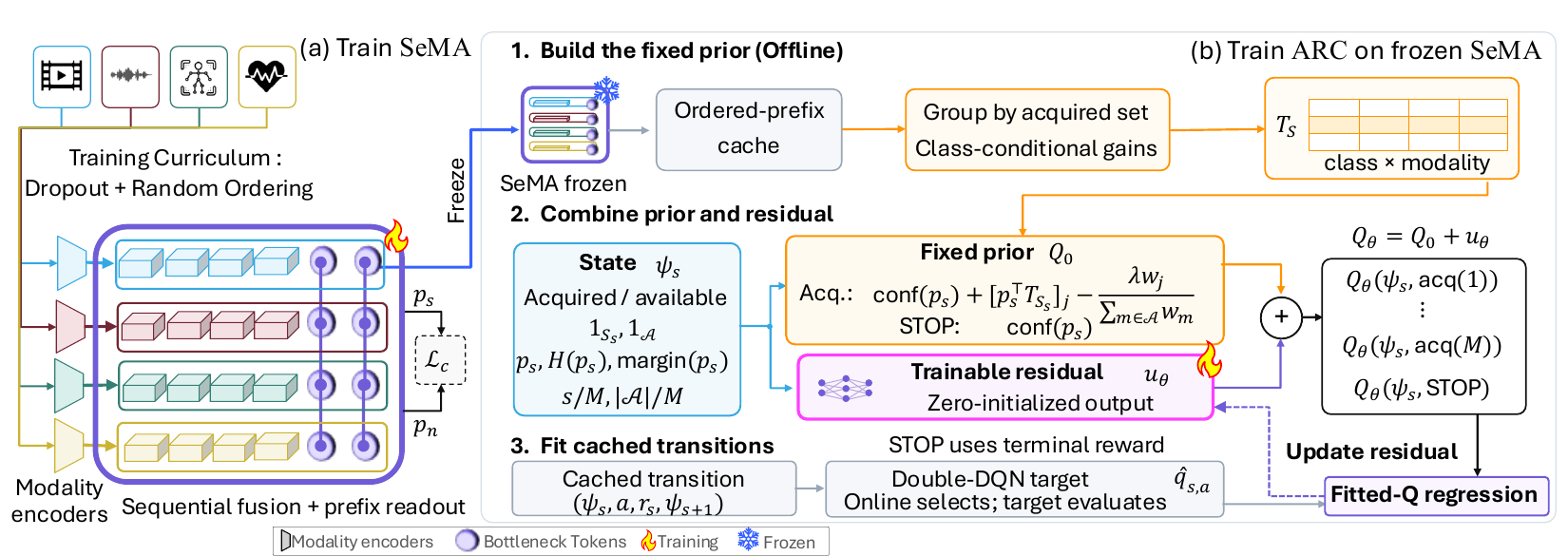}
    \caption{\textbf{\method{} training.} (a) \backbone{} learns prefix predictions $p_s$ under modality dropout and randomized orders, with classification supervision and same-forward distillation from detached $p_n$, the prediction after all $n$ retained modalities. (b) 1. Group frozen-backbone ordered-prefix predictions by acquired set $S$ to estimate class-conditional marginal gains $\bm T_S$. 2. Combine the fixed cost-aware prior $Q_0$ and trainable residual $u_\theta$ to score acquisition and stopping actions using acquired evidence and availability. 3. Fit cached-action values to Double-DQN targets, updating only the residual; STOP uses its terminal reward.}
    \label{fig:sys_diag}
    \vspace{-15pt}
\end{figure}

\subsection{Modality-branched Sequential Bottleneck Aggregator Network (\backbone{})}
\label{sec:m_arch}


\backbone{} adapts the attention-bottleneck mechanism of MBT~\citep{nagrani2021attention} for sequential fusion of selected modalities, producing a prediction after each acquisition. Let $\sigma=(\sigma_1,\ldots,\sigma_n)$ denote an acquisition order of distinct modalities from $\mathcal A$. After step $s$, $\sigma_{1:s}$ denotes the ordered prefix and $S_s=\{\sigma_1,\ldots,\sigma_s\}$ the corresponding acquired set.

\noindent\textbf{Modality-specific bottleneck attention.} Each modality $j$ has an input projection $\operatorname{W}_j$, an encoder $\operatorname{Enc}_j$, and a fusion module $\mathcal F_j$. At step $s$, only the selected branch $j=\sigma_s$ executes, producing $\bm h_j=\operatorname{Enc}_j(\bm x_j\operatorname{W}_j)$. Information from earlier acquisitions is carried by $\bm z_{s-1}\in\mathbb R^{K\times d}$, a compact representation comprising $K$ bottleneck tokens of feature dimension $d$; the initial tokens $\bm z_0$ are learned. We prepend the bottleneck tokens to the modality tokens; this configuration yields higher macro-F1 than suffix placement in our pilot analyses (discussed more in Figure~\ref{fig:ablation}). The fusion module jointly attends over these bottleneck tokens and the current modality tokens, producing $[\widetilde{\bm z}_s;\bm h'_s] =\mathcal F_j([\bm z_{s-1};\bm h_j])$, where $[\,\cdot;\cdot\,]$ denotes token concatenation. It applies $L_f$ pre-normalized transformer layers followed by layer normalization, yielding a bottleneck proposal $\widetilde{\bm z}_s$ and fused modality tokens $\bm h'_s$.

\noindent\textbf{Sequential bottleneck aggregation and prefix prediction.} The first bottleneck proposal initializes $\bm z_1=\widetilde{\bm z}_1$. Subsequent proposals update the running mean: $\bm z_s=(1-\frac1s)\bm z_{s-1}+\frac1s\widetilde{\bm z}_s=\frac1s\sum_{r=1}^s\widetilde{\bm z}_r$, for $s\geq2$. The bottleneck retains $K$ tokens regardless of the number of acquired modalities\footnote{The mean update does not impose permutation invariance: each proposal depends on the preceding bottleneck state and can therefore change with acquisition order.}. To predict, we average the layer-normalized tokens in $\bm h'_s$ to obtain a modality summary $\bm r_s$.
Their running mean is $\bar{\bm h}_s=\frac1s\sum_{r=1}^s\bm r_r$. A shared classifier $g$ produces logits $\bm o_s=g(\bar{\bm h}_s)$ and probabilities $p_s=\operatorname{softmax}(\bm o_s)$. Here, $p_s$ predicts from the realized prefix $\sigma_{1:s}$, with predicted class $\hat y_s=\arg\max_k p_s(k)$. Both $\bm z_s$ and $\bar{\bm h}_s$ are updated from their previous aggregates and the current branch outputs, without re-executing previously processed encoder or fusion branches. The controller can therefore use each prefix prediction to continue acquisition or stop, leaving unselected branches unexecuted.

\noindent\textbf{Training across subsets and orders.} We combine modality  dropout with randomized-order prefix supervision. We linearly increase the base dropout probability after an initial warm-up with all modalities, then hold it fixed. Each forward pass samples a nonempty retained set of $n$ modalities and a uniform random permutation $\sigma$ of that set. Every prefix receives classification supervision. We additionally use the prediction after processing all retained modalities as a same-forward distillation target, $\bar p=\operatorname{sg}(p_n)$, where $\operatorname{sg}$ denotes stop-gradient. With cross-entropy loss $\ell$, we minimize
\begin{equation}\label{eq:loss}
    \mathcal L_c
    =\ell(\bm o_n,y)
    +\frac{\lambda_{\mathrm{ds}}}{n}
      \sum_{s=1}^n\ell(\bm o_s,y)
    +\frac{\lambda_{\mathrm{kd}}}{n}
      \sum_{s=1}^n\operatorname{KL}(\bar p\,\|\,p_s).
\end{equation}
The three terms supervise the full retained prefix, supervise every prefix, and distill the same-forward full-prefix prediction, respectively. The weights $\lambda_{\mathrm{ds}}$ and $\lambda_{\mathrm{kd}}$ control prefix supervision and distillation. Conditional on the retained set and prefix length $s$, the acquired set $S_s$ is uniform over its $s$-element subsets (Appendix~\ref{app:order}). Together, modality dropout and order resampling expose the backbone to different acquisition paths.

\noindent\emph{\textbf{{Order sensitivity: analysis and empirical evidence.}}}
Different acquisition orders of the same modality subset reveal the same evidence, but sequential fusion can introduce order-dependent predictions. Randomized ordering and prefix supervision train these paths toward a common conditional label distribution, while modality dropout exposes this training to different retained sets.

\begin{proposition}[Prefix risk bounds order disagreement]
\label{thm:order_robustness}
Fix an available set $\mathcal A$ and a nonempty acquired set $S\subseteq\mathcal A$. Let $p_S^\sigma$ denote the prediction after processing $S$ in order $\sigma$, and let $\epsilon_S$ be its population prefix cross-entropy above the Bayes minimum, averaged over samples and uniformly sampled orders independent of the samples. For two such independent orders $\sigma$ and $\tau$ evaluated on the same acquired evidence,
\vspace{-5pt}
\begin{equation}\label{eq:order_robustness}
    \mathbb E\!\left[
        \operatorname{TV}(p_S^\sigma,p_S^\tau)
    \right]
    \leq\sqrt{2\epsilon_S},
\end{equation}
where $\operatorname{TV}$ is total variation distance.
\end{proposition}
\begin{wrapfigure}{r}{0.33\textwidth}
    \centering
    \vspace{-36pt}
    \includegraphics[width=\linewidth]
        {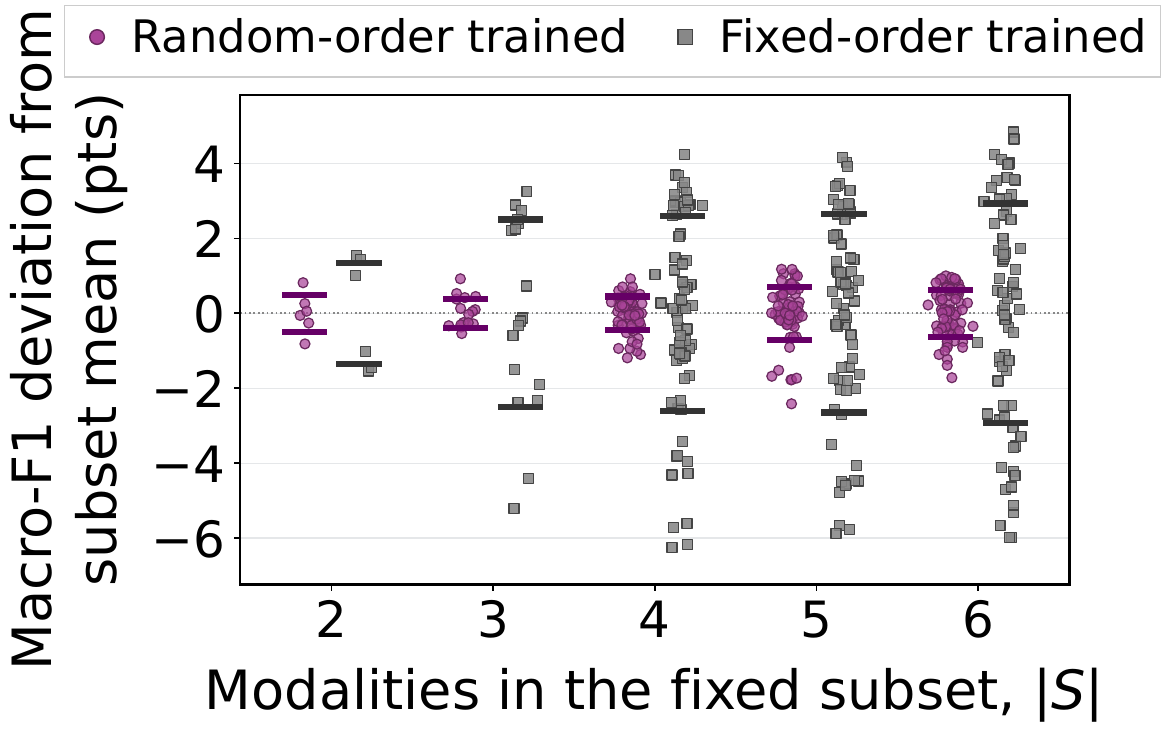}
    \captionsetup{font=small,skip=4pt}
    \vspace{-12pt}
    \caption{\small Macro-F1 deviation from the mean across acquisition
    orders of each fixed modality subset.}
    \vspace{-20pt}
    \label{fig:rand_curr}
\end{wrapfigure} \noindent\textbf{Implication.} Low excess prefix loss makes the predictions approach a common, order-independent posterior on average. Prefix supervision thus targets the excess-risk term governing disagreement across acquisition paths, while inference evaluates only the realized path.
Same-forward distillation complements this mechanism when the full-prefix targets are stable across orders and the prefix-to-target KL divergence is small (Complete proof in Appendix~\ref{app:order}).

\noindent\textbf{Empirical evidence.}
Figure~\ref{fig:rand_curr} shows that randomized-order training yields smaller macro-F1 deviations across different orders of a fixed acquired subset than fixed-order training, which illustrates our approximate robustness to prefix ordering.



\subsection{Prior-Augmented Adaptive Run-time Controller (\policy{})}
\label{sec:m_rl}

An acquired set of $s$ modalities admits $s!$ acquisition orders, which order-sensitive fusion may map to different predictive summaries. The approximate order robustness of \backbone{} (Proposition~\ref{thm:order_robustness}) motivates a compact state for \policy{} comprising acquired-set identity, modality availability, and the predictive summary, without explicitly encoding acquisition order. We formulate this sequential decision problem and use the approximate  information state (AIS) framework to analyze the conditions under which this compact state supports effective control.

\noindent\textbf{Sequential acquisition and controller state.}
We freeze \backbone{} before training \policy{}. Each sample $\bm x$ with a nonempty available set $\mathcal A$ defines an episode. After prefix $\sigma_{1:s}$, $s\geq1$, the controller observes
\begin{equation}\label{eq:rlstate}
\bm\psi_s
=
\big[
    \mathbf{1}_{S_s}
    \mathbin{\Vert}
    \mathbf{1}_{\mathcal A}
    \mathbin{\Vert}
    p_s
    \mathbin{\Vert}
    H(p_s)
    \mathbin{\Vert}
    \operatorname{margin}(p_s)
    \mathbin{\Vert}
    \tfrac{s}{M}
    \mathbin{\Vert}
    \tfrac{|\mathcal A|}{M}
\big].
\end{equation}
Here $\mathbf{1}_S\in\{0,1\}^M$ is the modality mask of $S$, $\Vert$ denotes concatenation, $H(p_s)$ is predictive entropy, and $\operatorname{margin}(p_s)$ is the difference between the two largest probabilities. The final entries normalize acquired and available modality counts. All features use only acquired evidence and availability.

Admissible actions are $\mathcal U_s=\{\operatorname{acq}(j): j\in\mathcal A\setminus S_s\}\cup\{\textsc{stop}\}$. Acquiring $j$ executes its branch, updates the backbone aggregates, and produces $\bm\psi_{s+1}$ with $S_{s+1}=S_s\cup\{j\}$. Stopping returns $\hat y_s$ and ends the episode; it is the only admissible action when $S_s=\mathcal A$. We use $\bm\psi_s$ as the state of a finite-horizon MDP approximation for learning acquisition and stopping decisions. The AIS analysis below characterizes the loss from this history compression under bounded reward and transition approximation errors.

\smallskip
\noindent\textbf{Reward and acquisition cost.} Stopping after $s$ acquisitions gives terminal reward $R_s$ and episodic return $G_s$:
\vspace{-10pt}
\begin{equation}\label{eq:rlreturn}
R_s=\beta(y)\,\mathbb{1}[\hat y_s=y],
\qquad
G_s=R_s-\lambda
\frac{\sum_{m\in S_s}w_m}{\sum_{m\in\mathcal A}w_m}.
\end{equation}
Here $\beta(y)\propto1/\operatorname{freq}(y)$ is normalized to unit mean, $w_m>0$ is modality $m$'s fixed relative acquisition cost, and $\lambda\geq0$ controls the acquisition penalty. Acquiring modality $j$ incurs a penalty of $\lambda w_j/\sum_{m\in\mathcal A}w_m$. We use uniform costs, $w_m=1$, as our default setting for comparability with other methods, giving a per-modality penalty of $\lambda/|\mathcal A|$, and examine sensitivity to nonuniform costs in Section~\ref{sec:ablate}. The controller is trained to maximize expected undiscounted return, using labels only during offline training.

\smallskip
\noindent\emph{\textbf{Theoretical support: Approximate Information State (AIS).}} We analyze this MDP approximation through the AIS framework. The modality masks determine admissible actions and acquisition costs exactly; the remaining approximation errors concern expected stopping rewards and action-conditioned successor-state distributions. We specialize Theorem~9 of \citet{subramanian2022ais} to our finite-horizon, undiscounted acquisition problem.


\begin{corollary}[Conditional AIS bound for modality acquisition]
\label{thm:ais_control}
Let $H_s$ denote the complete acquired history and $\beta_{\max}=\max_k\beta(k)$. Consider a model on $\bm\psi_s$ that preserves modality availability, admissible actions, acquired-set updates, and nonnegative acquisition costs, with stopping rewards in $[0,\beta_{\max}]$. Assume, uniformly over all reachable histories, that its stopping-reward function approximates $\mathbb E[R_s\mid H_s]$ within $\epsilon_{\mathrm r}$ and, for every admissible acquisition, its successor-state distribution approximates the conditional distribution of $\bm\psi_{s+1}$ given $H_s$ and that action within $\epsilon_{\mathrm p}$ in total variation. Define the normalized state-compression error $\delta=\max\{\epsilon_{\mathrm r}/\beta_{\max},\,\epsilon_{\mathrm p}\}$. If $\widehat\pi$ is optimal for this model and executed through $\bm\psi_s$, then
\begin{equation}\label{eq:ais_control}
0\leq V^\star(H_s)-V^{\widehat\pi}(H_s)
\leq 2\left[\epsilon_{\mathrm r}+(|\mathcal A|-s)\beta_{\max}\epsilon_{\mathrm p}\right]
\leq 2\beta_{\max}(|\mathcal A|-s+1)\delta,
\end{equation}
where $V^\star$ and $V^{\widehat\pi}$ are expected remaining returns in the original process under the optimal history-dependent policy and $\widehat\pi$, respectively. Thus, $\epsilon_{\mathrm r}$ measures error in the terminal reward, while each of the at most $|\mathcal A|-s$ remaining acquisitions contributes transition error on the reward scale through $\beta_{\max}\epsilon_{\mathrm p}$\footnote{Formal definitions and the derivation using an equivalent fixed-horizon representation of stopping are given in Appendix~\ref{app:ais}.}.
\end{corollary}

\noindent\textbf{Implication.} Together, Proposition~\ref{thm:order_robustness} and Corollary~\ref{thm:ais_control} identify state compression as the interface between backbone and controller: \backbone{} suppresses acquisition-order dependence, while the AIS result characterizes when the resulting compact state $\bm\psi_s$ supports near-optimal control. Under these conditions, the acquisition problem can be solved through Bellman planning on $\bm\psi_s$ rather than the complete acquisition history. This motivates the value-based controller developed next: we construct an informative prior over acquisition values and then refine those values with fitted-$Q$ learning.


\begingroup
\setlength{\abovedisplayskip}{4pt}
\setlength{\belowdisplayskip}{4pt}
\setlength{\abovedisplayshortskip}{0pt}
\setlength{\belowdisplayshortskip}{4pt}
\setlength{\jot}{2pt}

\smallskip
\noindent\textbf{Construction of the $Q$-prior.} We cache frozen \backbone{} predictions for every ordered modality prefix and group them by acquired set. Let $\Pi(S)$ denote the cached orders of $S$. For sample $i$, define the order-averaged correctness and the marginal gain of modality $j\notin S$ as
\begin{equation}\label{eq:coalition_value}
v_i(S)=\frac{1}{|\Pi(S)|}\sum_{\sigma\in\Pi(S)}\mathbb{1}[\hat y_i(S,\sigma)=y_i],
\qquad
\Delta_{i,j}(S)=v_i(S\cup\{j\})-v_i(S).
\end{equation}
Averaging these gains within each class gives a table $\bm T_S\in\mathbb R^{C\times M}$: $[\bm T_S]_{k,j}=\mathbb E_i[\Delta_{i,j}(S)\mid y_i=k]$, $j\notin S$, where $\mathbb E_i$ denotes the empirical average over cached samples. With $\operatorname{conf}(p_s)=\max_k p_s(k)$, the fixed prior is
\vspace{-10pt}
\begin{equation}\label{eq:qprior}
\begin{aligned}
Q_0(\bm\psi_s,\operatorname{acq}(j))
&=\operatorname{conf}(p_s)+[p_s^\top\bm T_{S_s}]_j-\frac{\lambda w_j}{\sum_{m\in\mathcal A}w_m},\\
Q_0(\bm\psi_s,\textsc{stop})
&=\operatorname{conf}(p_s).
\end{aligned}
\end{equation}
The predictive distribution weights class-specific gains, providing a set-dependent initialization without requiring labels or unacquired inputs at inference.
\par
\endgroup


\smallskip
\noindent\emph{\textbf{Interpretation of the acquisition prior.}} The tables $\bm T_S$ retain the class-conditional marginal contributions averaged in Shapley attribution~\citep{shapley1953value}, while preserving their dependence on the acquired set. Weighting these gains by $p_s$ provides a sample-conditioned estimate of order-averaged correctness gain. The prior initializes action values, while residual learning targets the class-weighted, multistep return. Appendix~\ref{app:qprior} establishes the Shapley connection and states the assumptions required for an exact conditional expected-gain interpretation.

\begingroup
\setlength{\abovedisplayskip}{4pt}
\setlength{\belowdisplayskip}{4pt}
\setlength{\abovedisplayshortskip}{0pt}
\setlength{\belowdisplayshortskip}{4pt}

\smallskip
\noindent\textbf{Residual fitted-$Q$ learning.} We keep $Q_0$ fixed and learn an additive correction:
\begin{equation}\label{eq:residual}
Q_\theta(\bm\psi_s,a)=Q_0(\bm\psi_s,a)+u_\theta(\bm\psi_s,a).
\end{equation}
The residual's output layer is initialized to zero, so $Q_\theta=Q_0$ initially. We use fitted-$Q$ regression~\citep{ernst2005tree} with Double-DQN targets~\citep{van2016deep}. Using the reward and acquisition penalty in Equation~\ref{eq:rlreturn}, we construct the Double-DQN target for each legal cached transition as,
\vspace{-10pt}
\begin{equation}\label{eq:target}
\widehat q_{s,a}=
\begin{cases}
R_s, & a=\textsc{stop},\\
-\dfrac{\lambda w_j}{\sum_{m\in\mathcal A}w_m}+Q_{\bar\theta}(\bm\psi_{s+1},a^+), & a=\operatorname{acq}(j),
\end{cases}
\end{equation}
where $a^+=\arg\max_{a'\in\mathcal U_{s+1}}Q_\theta(\bm\psi_{s+1},a')$. The online network selects the successor action, and the target network, with parameters $\bar\theta$, evaluates it. Over the legal cached transitions $\mathcal D_Q$, we minimize $\mathcal L_Q(\theta)=\mathbb E_{\mathcal D_Q}[(Q_\theta(\bm\psi_s,a)-\widehat q_{s,a})^2]$, where the expectation is the empirical average and targets are held constant during each regression update. Only the residual parameters are optimized; network, cache-splitting, and target-update details are deferred to the appendix.
\par
\endgroup

\subsection{\method{} Operation at Run Time}
\label{sec:summary}


Given the available set, we first acquire the modality with the highest offline-estimated population Shapley value among the modalities. After an initial acquisition from the nonempty available set $\mathcal A$, \backbone{} emits $p_s$ and \policy{} constructs $\bm\psi_s$ to select $a_s^\star=\arg\max_{a\in\mathcal U_s}Q_\theta(\bm\psi_s,a)$. Acquiring another modality executes only its encoder and fusion branch, updates the stored bottleneck and readout aggregates, and repeats the loop. Stopping returns $\hat y_s$ and is the only admissible action when all available modalities are acquired. This adapts computation per sample without rerunning earlier branches or executing unselected ones.
\section{Results and Discussion}
\label{sec:experiments}
We evaluate \method{} against eleven multimodal baselines on six datasets, covering full and simulated partial modality availability and heterogeneous compute platforms.

\smallskip

\noindent\textbf{Datasets.} We evaluate \method{} on six multimodal time-series classification benchmarks: \textbf{IEMOCAP}~\citep{busso2008iemocap}, \textbf{MM-Fi}~\citep{yang2023mm}, \textbf{CMI}~\citep{cmi}, \textbf{CZU-MHAD}~\citep{chao2022czu}, \textbf{EAV}~\citep{lee2024eav}, and \textbf{UTD-MHAD}~\citep{chen2015utd}. Training, validation, and test subjects are disjoint within each evaluation split. EAV uses three independently randomized cross-subject splits; test subjects may overlap across runs. Appendix~\ref{app:data_prep} provides full preparation details and the other dataset protocols, including rebalanced splits.

\smallskip

\noindent\textbf{Baselines.} We compare \method{} with eleven baselines grouped by selection timing. \emph{Monolithic fusion}: \textbf{MAESTRO}, \textbf{ShaSpec}, \textbf{MBT} and \textbf{DecAlign}~\citep{qian2026decalign} all fuse every available input. \emph{Fixed-sequential fusion}: \textbf{MultiModN} processes externally available modalities sequentially without learning their order or stopping. \emph{Selective fusion}: \textbf{Greedy Modality Selection (GMS)} chooses a dataset-level subset under a fixed budget; \textbf{AdaMML} selects per segment using multimodal previews; \textbf{DynMM} routes through data-dependent fusion paths; \textbf{DyMo} selects among recovered modalities; and an early exit baseline \textbf{MMEE} which exits along a predetermined computation graph. None selects the next unseen modality from accumulated prefix evidence with learned stopping. \emph{Proactive sequential acquisition}: \textbf{Joint-AFA} learns feature acquisition and stopping; we treat each modality as a feature group. Baselines use published hyperparameters, with search for unspecified settings. We report mean macro-F1, GFLOPs, latency, and energy in the main text; full statistics and implementation details are given in Appendices~\ref{app:add_results} and~\ref{app:baselines}, respectively. Per-sample compute covers the full path from raw modality streams to logits or, when explicitly labeled $\mathrm{GF}_{\mathrm f}$, only the post modality-encoding to logits stage; Appendix~\ref{app:metrics} specifies selector accounting. 

\subsection{Predictive Performance and Inference Efficiency}

\noindent\textbf{Higher F1 at lower average compute.} \method{} achieves the highest mean macro-F1 on four of six datasets (Table~\ref{tab:full-results}). On CMI, the F1 is 0.556 versus 0.561 for DecAlign, while total compute falls from 23.33 to 0.43 GFLOPs. On MM-Fi, total computation falls from 18,098 GFLOPs for AdaMML to 738 GFLOPs for \method{}, a reduction of approximately 95.9\%; their post-extraction costs are 0.17 and 0.07 GFLOPs, respectively. By selecting each modality before executing its branch, \method{} avoids extraction for unselected inputs and uses approximately 47.9\% of the available modalities on average. On average, \method{} improves macro-F1 by approximately 3.2\% and reduces total GFLOPs by 61.4\% relative to each dataset's highest-F1 baseline.

\begin{table*}[t]
\centering
\footnotesize
\captionsetup[table]{font=small}
\setlength{\tabcolsep}{1.2pt}
\renewcommand{\arraystretch}{1.12}
\caption{\small Macro-F1 and mean total GFLOPs (GF) over three folds. Parentheses show post-encoder fusion cost ($\mathrm{GF}_{\mathrm f}$). \textcolor{burgundy}{\textbf{Bold}} and \underline{underlining} mark the best and second-best values per metric. Complete results: Appendix~\ref{app:full_results}.}
\label{tab:full-results}
\vspace{-8pt}
\resizebox{\textwidth}{!}{%
\begin{tabular}{@{}ll*{5}{cc!{\color{black!25}\vrule width 0.35pt}}cc@{}}
\toprule
\multirow{2}{*}{\textbf{Type}} & \multirow{2}{*}{\textbf{Method}}
& \multicolumn{2}{c}{\textbf{IEMOCAP}}
& \multicolumn{2}{c}{\textbf{MM-Fi}}
& \multicolumn{2}{c}{\textbf{CMI}}
& \multicolumn{2}{c}{\textbf{CZU-MHAD}}
& \multicolumn{2}{c}{\textbf{EAV}}
& \multicolumn{2}{c}{\textbf{UTD-MHAD}} \\
\cmidrule(lr){3-4} \cmidrule(lr){5-6} \cmidrule(lr){7-8}
\cmidrule(lr){9-10} \cmidrule(lr){11-12} \cmidrule(lr){13-14}
& & F1 $\uparrow$ & $\mathrm{GF}\,(\mathrm{GF}_{\mathrm{f}})\downarrow$
  & F1 $\uparrow$ & $\mathrm{GF}\,(\mathrm{GF}_{\mathrm{f}})\downarrow$
  & F1 $\uparrow$ & $\mathrm{GF}\,(\mathrm{GF}_{\mathrm{f}})\downarrow$
  & F1 $\uparrow$ & $\mathrm{GF}\,(\mathrm{GF}_{\mathrm{f}})\downarrow$
  & F1 $\uparrow$ & $\mathrm{GF}\,(\mathrm{GF}_{\mathrm{f}})\downarrow$
  & F1 $\uparrow$ & $\mathrm{GF}\,(\mathrm{GF}_{\mathrm{f}})\downarrow$ \\

\midrule

\multirow{4}{*}{\textbf{Monolithic}} & MAESTRO & .51$_{\pm.04}$ & 6696 (.23) & .18$_{\pm.09}$ & 18098 (.28) & .47$_{\pm.02}$ & 0.28 (.28) & .25$_{\pm.16}$ & \textcolor{burgundy}{\textbf{6091}} (.30) & .53$_{\pm.02}$ & 6230 (.15) & .70$_{\pm.02}$ & 7306 (.19) \\

 & ShaSpec & .61$_{\pm.03}$ & 6697 (1.51) & .66$_{\pm.08}$ & 18099 (1.68) & .53$_{\pm.01}$ & 2.14 (2.14) & .59$_{\pm.06}$ & 6093 (2.17) & .60$_{\pm.05}$ & 6231 (.97) & .71$_{\pm.13}$ & 7307 (1.29) \\

 & MBT & \underline{.64}$_{\pm.02}$ & 6696 (.41) & .77$_{\pm.19}$ & 18098 (.46) & \underline{.55}$_{\pm.01}$ & 0.55 (.55) & .82$_{\pm.11}$ & \textcolor{burgundy}{\textbf{6091}} (.56) & .61$_{\pm.03}$ & 6230 (.26) & \underline{.81}$_{\pm.06}$ & 7306 (.35) \\

 & DecAlign & .48$_{\pm.01}$ & 6707 (11.41) & .28$_{\pm.19}$ & 18106 (8.45) & \textcolor{burgundy}{\textbf{.56}}$_{\pm.02}$ & 23.33 (23.33) & .36$_{\pm.10}$ & 6114 (23.37) & .62$_{\pm.02}$ & 6232 (1.77) & .75$_{\pm.08}$ & 7310 (4.21) \\

\arrayrulecolor{black!20}\midrule\arrayrulecolor{black}
\textbf{Fixed-seq.} & MultiModN & .58$_{\pm.06}$ & 6696 (.76) & .40$_{\pm.06}$ & 18098 (.84) & .45$_{\pm.01}$ & 1.07 (1.07) & .47$_{\pm.09}$ & \underline{6092} (1.08) & .65$_{\pm.04}$ & 6230 (.49) & .66$_{\pm.10}$ & 7306 (.64) \\
\arrayrulecolor{black!20}\midrule\arrayrulecolor{black}

\multirow{5}{*}{\textbf{Selective}} & GMS & .54$_{\pm.08}$ & 4465 (.23) & .40$_{\pm.11}$ & 18092 (.49) & .48$_{\pm.01}$ & \underline{0.11} (\underline{.11}) & .44$_{\pm.10}$ & \textcolor{burgundy}{\textbf{6091}} (\textcolor{burgundy}{\textbf{.13}}) & \underline{.66}$_{\pm.04}$ & \textcolor{burgundy}{\textbf{139}} (\textcolor{burgundy}{\textbf{.09}}) & .61$_{\pm.10}$ & 5185 (\underline{.17}) \\

 & AdaMML & .59$_{\pm.03}$ & 6696 (\underline{.20}) & \textcolor{burgundy}{\textbf{.96}}$_{\pm.02}$ & 18098 (\underline{.17}) & .52$_{\pm.01}$ & 0.51 (.51) & \underline{.83}$_{\pm.09}$ & \textcolor{burgundy}{\textbf{6091}} (\underline{.17}) & .64$_{\pm.07}$ & 6230 (.17) & .75$_{\pm.06}$ & 7306 (.19) \\

 & DynMM & .62$_{\pm.01}$ & 6697 (1.78) & .71$_{\pm.30}$ & 18099 (1.40) & .54$_{\pm.01}$ & 3.13 (3.13) & .75$_{\pm.08}$ & \underline{6092} (1.54) & \underline{.66}$_{\pm.01}$ & 6230 (.54) & .78$_{\pm.05}$ & 7308 (1.91) \\
 & DyMo & .56$_{\pm.05}$ & 6696 (.48) & .54$_{\pm.03}$ & 18098 (.52) & .37$_{\pm.01}$ & 0.65 (.65) & .65$_{\pm.12}$ & \textcolor{burgundy}{\textbf{6091}} (.68) & .60$_{\pm.04}$ & 6230 (.28) & .60$_{\pm.06}$ & 7306 (.38) \\

 & MMEE & .61$_{\pm.01}$ & 6696 (.38) & .60$_{\pm.05}$ & 18098 (.59) & .48$_{\pm.01}$ & 0.70 (.70) & .57$_{\pm.09}$ & \textcolor{burgundy}{\textbf{6091}} (.73) & .60$_{\pm.07}$ & 6230 (.18) & .65$_{\pm.04}$ & 7306 (.44) \\

\arrayrulecolor{black!20}\midrule\arrayrulecolor{black}

\textbf{Proactive} & Joint-AFA & .57$_{\pm.04}$ & \textcolor{burgundy}{\textbf{1655}} (\textcolor{burgundy}{\textbf{.13}}) & .11$_{\pm.10}$ & \underline{12059} (.21) & .38$_{\pm.03}$ & \textcolor{burgundy}{\textbf{0.07}} (\textcolor{burgundy}{\textbf{.07}}) & .28$_{\pm.05}$ & \textcolor{burgundy}{\textbf{6091}} (\textcolor{burgundy}{\textbf{.13}}) & .54$_{\pm.03}$ & 3321 (.16) & .30$_{\pm.04}$ & \textcolor{burgundy}{\textbf{755}} (\textcolor{burgundy}{\textbf{.04}}) \\

\arrayrulecolor{black}\midrule
\rowcolor{tolblue!10}
\multicolumn{2}{l}{\textbf{\method{} (Ours)}} & \textcolor{burgundy}{\textbf{.67}}$_{\pm.02}$ & \underline{3311} (.28) & \underline{.95}$_{\pm.04}$ & \textcolor{burgundy}{\textbf{738}} (\textcolor{burgundy}{\textbf{.07}}) & \textcolor{burgundy}{\textbf{.56}}$_{\pm.01}$ & 0.43 (.43) & \textcolor{burgundy}{\textbf{.91}}$_{\pm.10}$ & \textcolor{burgundy}{\textbf{6091}} (.23) & \textcolor{burgundy}{\textbf{.69}}$_{\pm.02}$ & \underline{719} (\underline{.10}) & \textcolor{burgundy}{\textbf{.83}}$_{\pm.10}$ & \underline{4706} (.19) \\

\bottomrule
\end{tabular}%
}
\vspace{-22pt}
\end{table*}

\begin{wrapfigure}{r}{0.6\textwidth}
    \centering
    \captionsetup{font=small}
    \includegraphics[width=\linewidth]
        {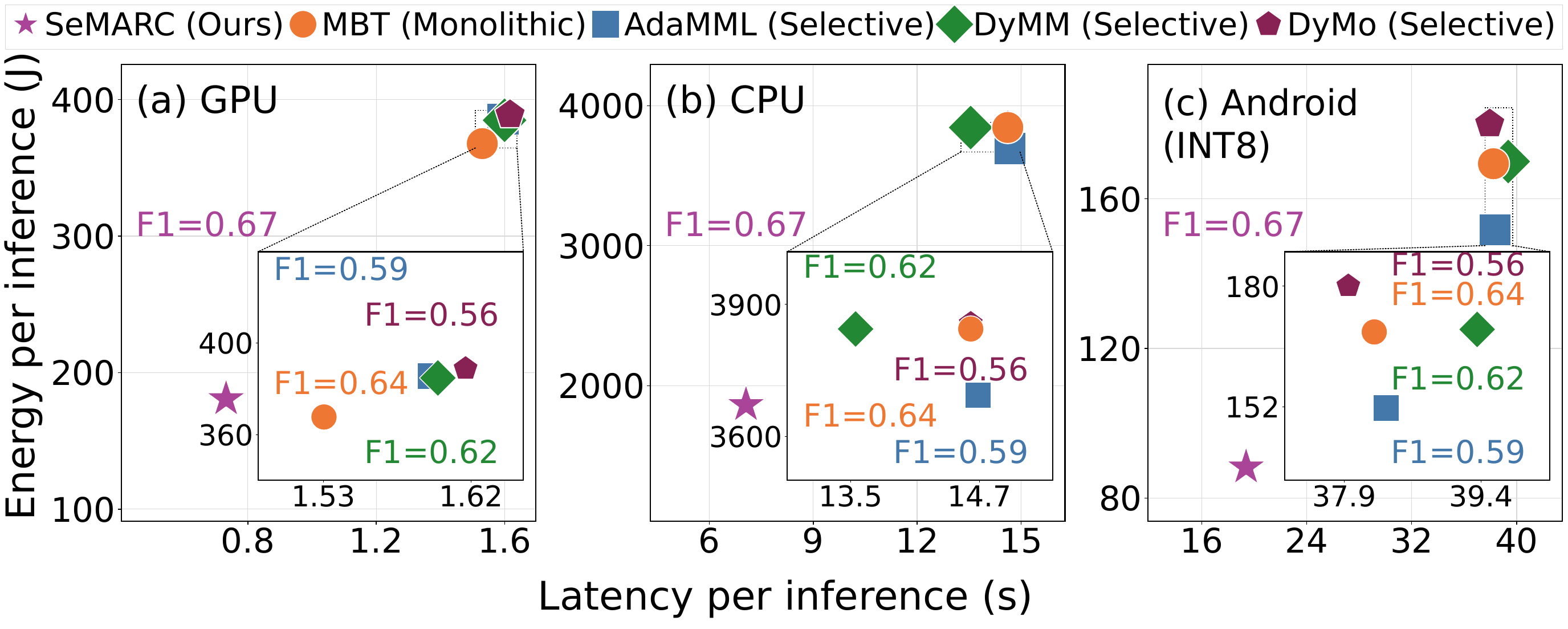}
    \vspace{-22pt}
    \caption{\small Energy--latency trade-offs on IEMOCAP across GPU, CPU, and Android INT8. Lower-left is preferred.}
    \label{fig:energy_latency}
    \vspace{-15pt}
\end{wrapfigure} 

\noindent\textbf{Lower latency and energy across platforms.} On IEMOCAP, \method{} achieves the highest displayed macro-F1 (.67) at about half the latency of the fastest baseline on evaluated platforms. GPU and Android energy are also 50.9\% and 41.9\% lower than their respective lowest-energy baselines. These gains extend across the benchmark suite: \method{} is fastest in 16 of 18 dataset-platform conditions, averaging 44.7\% lower latency against each condition's fastest baseline. Complete results and measurement details are given in Appendix~\ref{app:energy-latency-complete}.

\begin{wrapfigure}{r}{0.5\textwidth}
\vspace{-15pt}
\centering
\captionsetup{font=small}
\includegraphics[width=\linewidth]{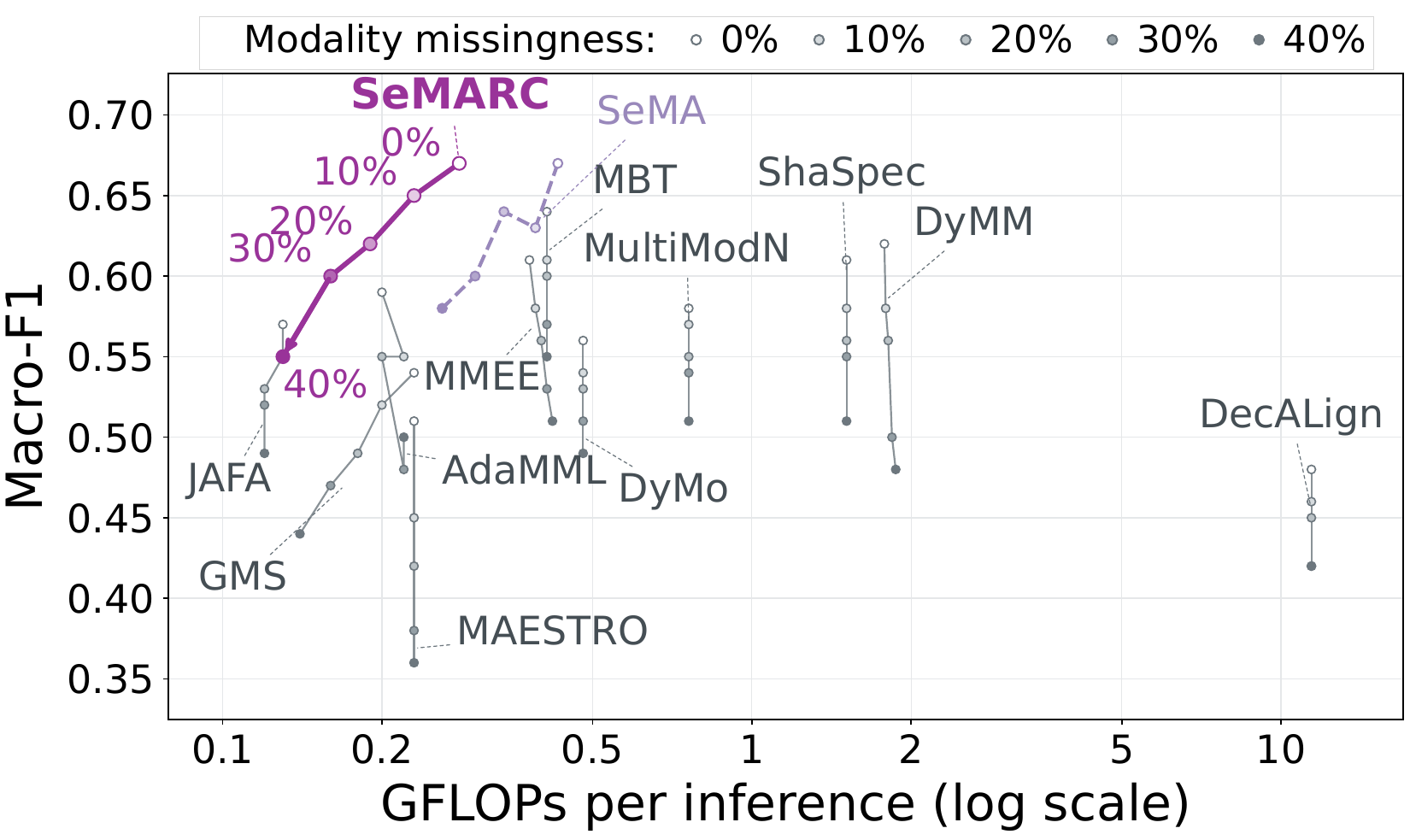}
\vspace{-20pt}
\caption{IEMOCAP macro-F1 versus post-extraction GFLOPs ($\mathrm{GF}_{\mathrm f}$) under 0--40\% test-time modality missingness. Upper-left is preferred.}
\label{fig:miss_gflops}
\vspace{-15pt}
\end{wrapfigure} \noindent\textbf{Efficiency gains persist under modality missingness.} Figure~\ref{fig:miss_gflops} plots IEMOCAP macro-F1 against post-extraction GFLOPs ($\mathrm{GF}_{\mathrm f}$). \method{} is Pareto non-dominated on IEMOCAP at every tested missingness level. Adding \policy{} to \backbone{} reduces post-extraction GFLOPs by 34.9--50.0\% and total GFLOPs, including extraction, by 50.5--63.9\% across 0--40\% missingness on IEMOCAP, with a decrease of at most .03 absolute macro-F1 points (Appendix~\ref{app:missingness-complete}). This supports our premise that sample-adaptive acquisition can substantially reduce computation while largely preserving the backbone's predictive quality. Across six datasets at 10--40\% missingness, \method{} matches or exceeds the highest baseline macro-F1 in 21 of 24 conditions and is Pareto non-dominated in 23 of 24 using either total GFLOPs or post-encoder fusion GFLOPs ($\mathrm{GF}_{\mathrm f}$). Relative to each condition's best baseline, \method{} averages 3.6\% higher macro-F1, 14.8\% lower total GFLOPs, and 38.4\% lower post-extraction GFLOPs. Complete results are given in Appendix~\ref{app:missingness-complete}.

\vspace{-10pt}
\subsection{Further Design Analysis, Ablations and Diagnostics}\label{sec:ablate}

\noindent\textbf{Component replacements.} Table~\ref{tab:component-swap} compares backbone and controller replacements across six datasets. (i) Adding \policy{} to MultiModN improves efficiency, reducing mean latency by 36.36\% with a small decrease in mean macro-F1 from .533 to .527. However, this combination remains less accurate and slower than \method{}. (ii) With \backbone{} fixed, \policy{} achieves higher mean macro-F1 (.77) than EAMA (.74), Joint-AFA (.75), and GMS (.73). It reduces mean latency by 43.46\% relative to Joint-AFA and 32.62\% relative to GMS, while EAMA is faster (463.50 versus 520.27\,ms). GMS provides a static-selection comparison
on the same backbone, helping isolate the contribution of adaptive
acquisition. At the reported operating points, \method{} achieves
both higher mean macro-F1 and lower mean latency across the six
datasets (Table~\ref{tab:component-swap}). \begin{wraptable}{r}{0.32\textwidth}
\centering
\captionsetup{
    font=scriptsize,
    skip=3pt,
    justification=raggedright,
    singlelinecheck=false
}
\caption{Mean macro-F1 and GPU latency across six datasets.}
\vspace{-2pt}
\label{tab:component-swap}

{\fontsize{7.2}{8.0}\selectfont
\setlength{\tabcolsep}{2.5pt}
\renewcommand{\arraystretch}{0.98}
\begin{tabular*}{\linewidth}{@{\extracolsep{\fill}}lcc@{}}
\toprule
\textbf{Configuration}
& \textbf{F1}$\uparrow$
& \textbf{$t$ (ms)}$\downarrow$ \\
\midrule
\rowcolor{tolblue!10}
\textbf{\method{}}      & .77 & 520.27 \\
\arrayrulecolor{black!20}\midrule
\arrayrulecolor{black}
MultiModN + \policy{}   & .53 & 1078.90 \\
\arrayrulecolor{black!20}\midrule
\arrayrulecolor{black}
\backbone{} + EAMA      & .74 & 463.50 \\
\backbone{} + Joint-AFA & .75 & 920.32 \\
\backbone{} + GMS       & .73 & 772.08 \\
\bottomrule
\end{tabular*}
\vspace{-13pt}
}
\end{wraptable}This supports the
cross-dataset average benefit of our controller over this static
selector. EAMA and Joint-AFA act as adaptive selector baselines. These results support our coupled design: a strong sequential predictor with acquisition and stopping decisions informed by accumulated evidence. They also motivate exploring richer subset-acquisition strategies coupled with \backbone{}. Complete results are given in Appendix~\ref{app:comp-replace-complete}.

\begin{wrapfigure}{r}{0.5\textwidth}
    \vspace{-15pt}
    \centering
    \captionsetup{font=small}
    \includegraphics[width=\linewidth]{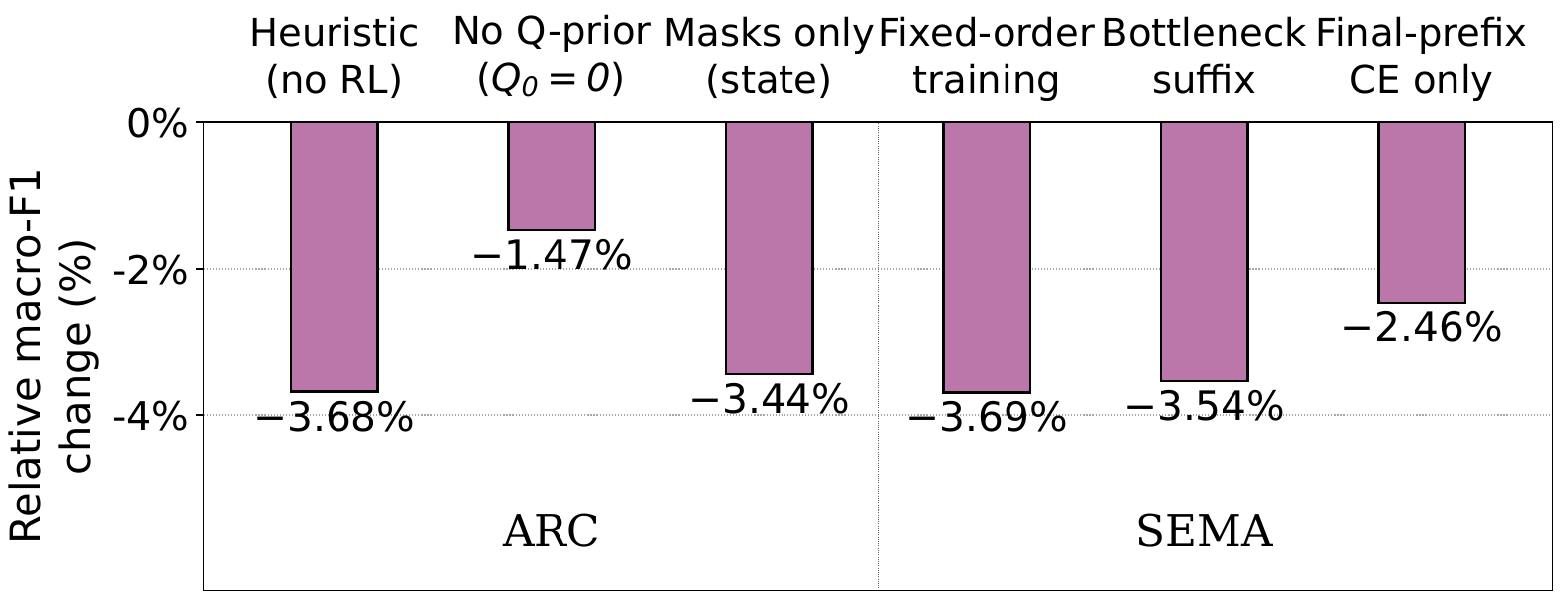}
    \vspace{-15pt}
    \caption{Relative macro-F1 changes for selected ablations on IEMOCAP under full modality availability. Complete results: Appendix~\ref{app:ablate-complete}.}
    \label{fig:ablation}
    \vspace{-12pt}
\end{wrapfigure} \noindent\textbf{Ablation study.} Figure~\ref{fig:ablation} and Table~\ref{tab:iemocap-ablation} (Appendix~\ref{app:ablate-complete}) highlight three findings. (i) Removing the Q-prior and randomized-order training reduces macro-F1 by 1.47\% and 3.69\%, respectively. (ii) Bottleneck-suffix yields 3.54\% lower macro-F1, while removing prefix supervision and distillation together, retaining only final-prefix CE, reduces macro-F1 by 2.46\%. (iii) Disabling $\mathrm{STOP}$ preserves approximately .67 macro-F1 but increases modality utilization from .65 to 1.00. Together, these results support the coupled design of \method{}: \backbone{} reuses previous  computation and is trained for order-robust prediction after each acquisition, while \policy{} uses the resulting evidence summary to select the next modality or stop, avoiding unnecessary encoder and fusion execution.
\vspace{-5pt}
\newlength{\diagwidth}

\begin{figure*}[!htbp]
\centering
\setlength{\parskip}{0pt}
\setlength{\diagwidth}{0.96\textwidth}
\captionsetup{
    font=small,
    position=bottom,
    skip=1pt,
    belowskip=-2pt
}

\begin{minipage}[t]{0.33\diagwidth}
\centering\vspace{0pt}
\includegraphics[width=\linewidth]
    {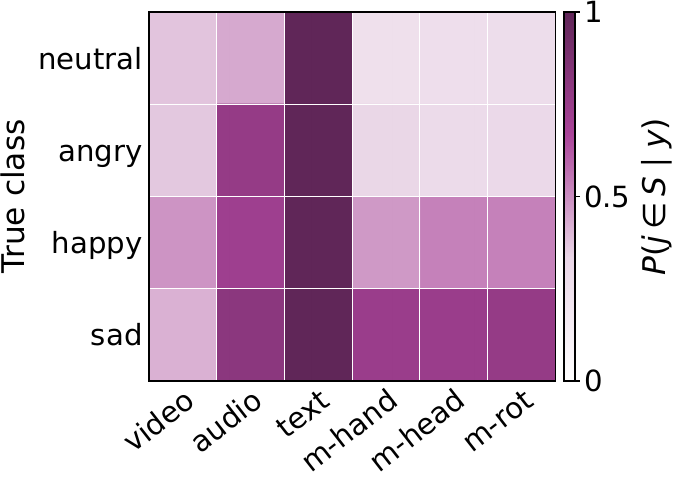}
\caption{Modality acquisition frequencies by true class.}
\label{fig:diag-class}
\end{minipage}\hfill
\begin{minipage}[t]{0.33\diagwidth}
\centering\vspace{0pt}
\includegraphics[width=\linewidth]
    {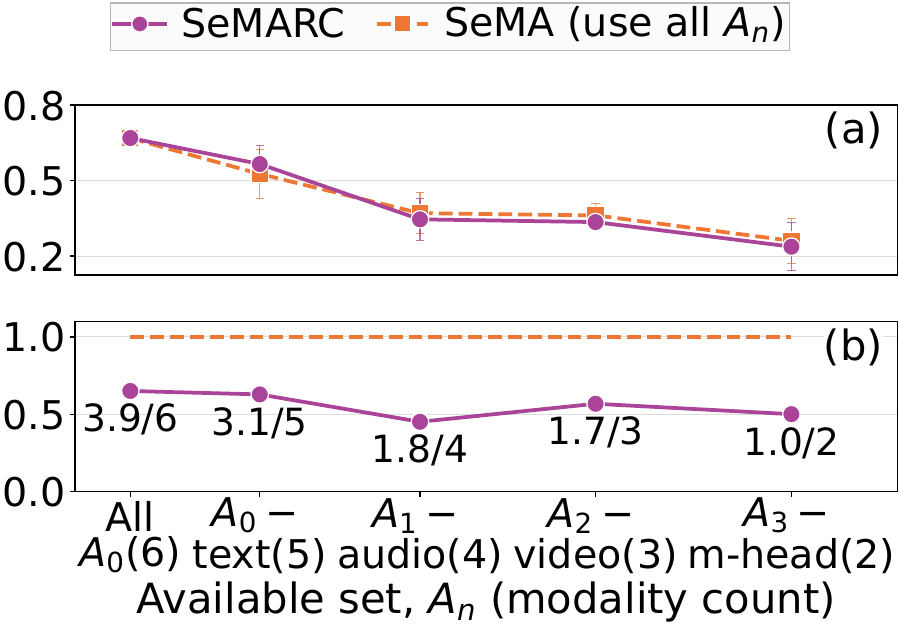}
\vspace{-12pt}
\caption{Availability cascade: (a) macro-F1; (b) fraction of available modalities acquired.}
\label{fig:diag-availability}
\end{minipage}\hfill
\begin{minipage}[t]{0.33\diagwidth}
\centering\vspace{0pt}
\includegraphics[width=\linewidth]
    {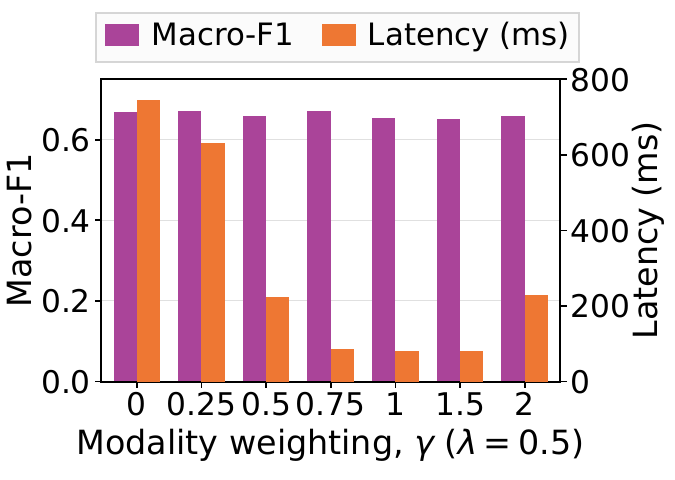}
\caption{Modality-weight sensitivity at $\lambda=0.5$.}
\label{fig:diag-reward}
\end{minipage}
\end{figure*}
\vspace{-10pt}
\noindent\textbf{Diagnostics.} Our IEMOCAP diagnostics show three behaviors. (i) Under full availability, \policy{} selects text first using only the availability mask. When samples are grouped by true class, modalities' acquisition frequencies vary, supporting our premise of adapting modality selection to each sample's acquired evidence (Figure~\ref{fig:diag-class}). (ii) Successively removing the preceding first pick elicits availability-aware fallback. \method{} remains close to \backbone{} processing all remaining modalities while leaving approximately 35--55\% of available modalities unacquired (Figure~\ref{fig:diag-availability}), demonstrating voluntary savings beyond missingness. (iii) At fixed $\lambda$, varying $\gamma$ to iterate the modality-weight exponent $w_j$ in Equation~\ref{eq:target} preserves approximately 0.67 macro-F1 while reducing reported latency by 88.3\% (from 745 to 87\,ms), illustrating the accuracy–latency trade-off (Figure~\ref{fig:diag-reward}). Complete results are provided in Appendix~\ref{app:reward-sensitivity}.


\vspace{-10pt}
\section{Conclusion}
\vspace{-10pt}
We introduced \method{}, coupling incremental fusion with adaptive acquisition and stopping so that acquired evidence determines subsequent computation. Across six datasets, \method{} achieves 3.2\% higher macro-$F1$ and 61.4\% lower total GFLOPs on average than each dataset's most accurate baseline under full availability, with practical latency gains across devices and continued computational savings under modality missingness.

\noindent\textbf{Limitations and future work.}
Our evaluation focuses on per-sample classification with fixed modality costs and batch-one inference. Our compact-state analysis provides conditional guarantees under reward and transition approximation assumptions; empirically assessing these assumptions and extending the analysis to the fitted neural controller remain directions for future work. Future work includes adapting acquisition to externally supplied, time-varying deployment costs, designing batching and scheduling for higher throughput, and incorporating memory across samples to support long-horizon tasks.

\bibliography{iclr2027_conference, semarc_ais_references}
\bibliographystyle{iclr2027_conference}

\appendix
\onecolumn
\appendix

\section{Implementation Details of the \method{} policy}\label{app:rl_policy}

The backbone is frozen before policy fitting. Cached ordered-prefix predictions
define the transition samples, the class-conditional gain tables define the
prior $Q_0$, and a residual network is fitted with Double-DQN targets
(Section~\ref{sec:m_rl}). All of the following is per dataset and per fold.

\paragraph{Backbone configuration.}
Table~\ref{tab:backbone-config} lists the dataset-specific \backbone{} configurations. All use $d=128$, 8 heads, $K=16$ bottleneck tokens, dense attention, running-mean aggregation (Section~\ref{sec:m_arch}), and randomized modality order. We train with AdamW (weight decay 0), a cosine learning-rate schedule from the listed initial rate to $10^{-6}$, and gradient clipping at norm 1. Batch-level modality dropout retains at least one modality; its probability is 0 for the first $E_{\mathrm w}$ epochs, increases linearly over the next $E_{\mathrm r}$ epochs, and remains at $p_{\max}$ thereafter. Each forward pass uniformly permutes the retained modalities. We optimize Equation~\ref{eq:loss} with
$\lambda_{\mathrm{ds}}=1.0$, $\lambda_{\mathrm{kd}}=0.5$,
and uniform prefix weights. For each forward pass, let $n$
be the number of retained modalities. The distillation target is
$\bar p=\operatorname{sg}(p_n)$, the detached prediction after
processing all $n$ retained modalities in the same sampled order
and forward pass. Each prefix prediction is distilled toward this
target, with gradients stopped through $\bar p$. The loss $\ell$ is cross-entropy; on IEMOCAP, it uses class weights $N/(C N_c)$ computed from the training labels. We train one checkpoint per fold with the listed seed and retain the epoch with the best validation macro-F1.

\begin{table}[htbp]
\centering
\small
\captionsetup{font=normalsize,skip=2pt}
\caption{\backbone{} training configuration per dataset ($E$: epochs;
$E_{\mathrm w}/E_{\mathrm r}$: modality-dropout warm-up and ramp epochs;
$p_{\max}$: final modality-dropout probability; $L_{\mathrm e}/L_{\mathrm f}$:
encoder and fusion blocks per branch; $T$: sequence length).}
\label{tab:backbone-config}
\setlength{\tabcolsep}{4pt}
\begin{tabular}{@{}lccccccccc@{}}
\toprule
Dataset & Seed & $E$ & $E_{\mathrm w}/E_{\mathrm r}$ & $p_{\max}$ & LR & Batch & $L_{\mathrm e}/L_{\mathrm f}$ & $T$ & Folds \\
\midrule
IEMOCAP  & 239 & 40  & 10/20 & 0.4 & $10^{-4}$ & 32 & 3/3 & 100 & 0--2 \\
MM-Fi    & 239 & 100 & 20/40 & 0.2 & $10^{-4}$ & 16 & 2/4 & 128 & 0--2 \\
CMI      & 239 & 40  & 10/20 & 0.4 & $10^{-4}$ & 32 & 3/3 & 128 & 0--2 \\
CZU-MHAD & 239 & 100 & 20/40 & 0.2 & $10^{-4}$ & 32 & 3/3 & 128 & 0--2 \\
EAV      & 239 & 100 & 20/40 & 0.2 & $10^{-4}$ & 32 & 3/3 & 128 & 0--2 \\
UTD-MHAD & 9   & 100 & 20/40 & 0.4 & $10^{-4}$ & 32 & 3/3 & 128 & 0--2 \\
\bottomrule
\end{tabular}
\end{table}

\paragraph{Prefix cache.}
With $M$ modalities, we run the frozen backbone once over every ordered
modality prefix of every sample in each split (train, validation, test) and
store the class posterior after each prefix. The cache is a complete depth-$M$
prefix tree with $\sum_{k=0}^{M} M!/(M-k)!$ nodes: 1957 for IEMOCAP ($M{=}6$),
13\,700 for CMI and CZU-MHAD ($M{=}7$), 326 for MM-Fi ($M{=}5$), 65 for UTD-MHAD
($M{=}4$) and 16 for EAV ($M{=}3$). The cache supplies the backbone posteriors needed to construct states and gain
tables, evaluate Bellman targets, and roll out policies. The residual MLP is
still evaluated, but the backbone is not rerun during policy fitting or
selection. Test-cache entries are reserved for final evaluation.

\paragraph{Data used at each stage.}
The policy is fitted entirely on the validation fold, because the backbone's
training-split predictions are overfit (correct at almost every prefix) and
carry no acquisition signal. The validation samples are split once into a fit
part and a holdout part: the $N$ validation indices are permuted with NumPy's
\texttt{default\_rng(1000)}, the first $\max(5,\lfloor0.2N\rfloor)$ form the
holdout and the rest the fit part, the same split for every seed, price and
variant. The gain tables $\bm{T}_S$ and the residual network are fitted on the
fit part; the holdout part selects the FQI iteration to keep; the seed is
selected on the validation fold; the test fold is touched once, for
the reported numbers. The train split is not used by the policy at all.

\paragraph{State and prior.}
The state $\bm\psi_s\in\mathbb R^{2M+C+4}$ in Equation~\ref{eq:rlstate} contains the acquired mask, availability mask, posterior, entropy, top-two margin, $|S_s|/M$, and $|\mathcal A|/M$. The gain table of a held set $S$ is
\[
[\bm T_S]_{c,j}=\mathbb E_i\!\left[v_i(S\cup\{j\})-v_i(S)\mid y_i=c\right],
\qquad
v_i(S)=\frac{1}{|\Pi(S)|}\sum_{\pi\in\Pi(S)}\mathbb 1[\hat y_i(S,\pi)=y_i].
\]
Here $v_i(S)$ is the fraction of cached orderings of $S$ under which the backbone classifies sample $i$ correctly. Correctness is evaluated separately for each ordering and then averaged. Tables are estimated on the validation fit part, with one $C\times M$ table for each of the $2^M$ subsets. This differs from EAMA's set-posterior cache, which averages the posteriors themselves.

For the uniform-cost setting, the prior in Equation~\ref{eq:qprior} is
\[
Q_0^{(\lambda)}(\bm\psi_s,j)=\max_c p_{s,c}+[p_s^\top\bm T_{S_s}]_j-\frac{\lambda}{|\mathcal A|},
\qquad
Q_0^{(\lambda)}(\bm\psi_s,\textsc{stop})=\max_c p_{s,c},
\]
with acquired and unavailable modalities masked. The prior cost, the acquisition cost in the Bellman targets and the prior used at inference all use the same per-fold price $\lambda_{\mathrm{train}}$, so the deployed action values are exactly the fitted ones, $Q_\theta=Q_0^{(\lambda_{\mathrm{train}})}+R_\theta$.

\paragraph{Residual network.}
$R_\theta$ is a three-layer MLP, $(2M+C+4)\to 64\to 64\to M{+}1$ with ReLU
activations and no normalization. Its output layer is zero-initialized, so
$Q_\theta=Q_0$ exactly before fitting and the prior is only ever corrected,
never replaced. It has about 6\,k parameters on IEMOCAP.

\paragraph{Fitted-Q iteration.}
We run 8 FQI iterations. Each iteration recomputes the Bellman target for every
reachable (prefix, sample, legal action) triple in the fit part with a
Double-DQN rule: the next action is the arg-max of the online network
(residual plus prior), its value is read from a target copy, and the target
copy is overwritten with the online weights at the end of the iteration. The
terminal reward uses normalized inverse-frequency weights. Let $N_{\mathrm{fit}}$ be the number of validation-fit samples and let $\mathrm{freq}(c)$ be the proportion with label $c$ in that split. We use
\[
w_i=\frac{1/\mathrm{freq}(y_i)}{N_{\mathrm{fit}}^{-1}\sum_{k=1}^{N_{\mathrm{fit}}}1/\mathrm{freq}(y_k)},
\qquad
R_{\textsc{stop},i}=w_i\,\mathbb 1[\hat y_i=y_i].
\]
The sample-mean weight is one, keeping the reward scale comparable to the acquisition price across datasets with different class counts. This weights minority-class correctness without directly optimizing macro-F1. Each acquisition in the FQI target costs
$\lambda_{\mathrm{train}}/|\mathcal{A}|$, and the return is undiscounted.
Within an iteration the residual is regressed onto (target $-$ prior) with a
masked mean-squared error over the reachable triples for 40 full-batch epochs
of Adam (learning rate $10^{-3}$), 320 epochs in total. After each iteration
the greedy policy is rolled out on the validation holdout and the iteration
with the best holdout macro-F1 is kept.

\paragraph{Availability augmentation.}
To fit a policy that also works when modalities are missing, every sample's
available set is drawn as Bernoulli(0.7) per modality (at least one modality
kept) and the draw is repeated at every FQI iteration, so each sample is seen
under several availability patterns and the network sees the same sample with
and without a given modality. The holdout is scored under one fixed draw.
Unavailable modalities are removed from the action set and from the state's
availability mask; the training cost $\lambda_{\mathrm{train}}/|\mathcal{A}|$ is normalized by the
number that are actually available.

\paragraph{Deployed settings.}
Table~\ref{tab:policy-config} lists, for every fold, the settings of the
deployed uniform-cost policy reported in Table~\ref{tab:full-results}. The
training price $\lambda_{\mathrm{train}}$ enters the prior and the Bellman
targets during fitting and the prior at inference, and actions are selected
greedily on every fold.
Seeds are selected on the validation sets.

\begin{table}[t]
\centering
\small
\captionsetup{font=normalsize,skip=2pt}
\caption{Per-fold settings of the deployed uniform-cost \policy{}. The
prior is priced at $\lambda_{\mathrm{train}}$ during fitting and at inference on every fold.}
\label{tab:policy-config}
\setlength{\tabcolsep}{4pt}
\begin{tabular}{@{}lcccc@{}}
\toprule
Dataset & $M$ & Fold & Seed & $\lambda_{\mathrm{train}}$ \\
\midrule
IEMOCAP  & 6 & 0 / 1 / 2 & 1 / 1 / 0 & 0.5 / 0.5 / 0.5 \\
MM-Fi    & 5 & 0 / 1 / 2 & 2 / 0 / 0 & 0.5 / 0.5 / 1.0 \\
CMI      & 7 & 0 / 1 / 2 & 1 / 2 / 0 & 0.5 / 0.5 / 0.5 \\
CZU-MHAD & 7 & 0 / 1 / 2 & 0 / 1 / 2 & 0.5 / 0.5 / 0.5 \\
EAV      & 3 & 0 / 1 / 2 & 2 / 1 / 2 & 0.5 / 0.5 / 0.5 \\
UTD-MHAD & 4 & 3 / 4 / 5 & 0 / 1 / 2 & 0.5 / 0.5 / 0.5 \\
\bottomrule
\end{tabular}
\end{table}

All policies are fitted in under a minute per fold on one GPU, since fitting
runs on cached tables; at inference the policy adds one table lookup and one
MLP evaluation per step, about 0.1\,ms.

\section{Complete Dataset Preparation}\label{app:data_prep}

\subsection{Dataset Preparation}
\label{app:datasets}

This section details the six datasets used for evaluation: modalities and their native sampling rates, the feature each modality is mapped to before fusion, class distributions, and the cross-subject splits. The train, validation and test partitions are subject-disjoint. Backbone and policy fitting and model selection use the designated training and validation data. Quantization calibration uses val-split examples, as specified in Appendix~\ref{app:deployment}.

\smallskip

\noindent\textbf{Common preprocessing.} Each modality is converted to a sequence $(T_m, d_m)$, either directly from the raw stream (sensor and motion-capture channels) or through a frozen extractor (DINOv3 ViT-L/16~\citep{simeoni2025dinov3} for image and depth frames, XLS-R~\citep{babu2021xls} or wav2vec 2.0~\citep{baevski2020wav2vec} for audio, BERT-base~\citep{devlin2019bert} for text, and a frozen, randomly initialized PointNet~\citep{qi2017pointnet} descriptor for point clouds). Preprocessing includes setting non-finite entries to zero, normalization by the largest per-timestep $\ell_2$ norm within each sequence, and temporal alignment to a common length $T$. The default temporal alignment uses zero-padding when shorter and linear interpolation when longer; the CZU-MHAD resampling and normalization order is specified below. $T=100$ on IEMOCAP and $T=128$ elsewhere. Per-modality extraction times for the explicitly labelled cost-weighted variants are measured on the same hardware for every dataset; the default policy uses uniform acquisition costs.

\paragraph{Extractor checkpoints and feature layers.}
All extractors are frozen. Image, audio, and text extractors use publicly released pretrained checkpoints without fine-tuning; the PointNet descriptors use fixed random weights, as specified below. The same frozen extractor weights produce the training features and server-side measurements and serve as the source for the ONNX graphs deployed on the phone. The extraction environment uses PyTorch~2.5.1, torchaudio~2.5.1, and Transformers~5.3.0.

\emph{Images and depth.} DINOv3 ViT-L/16 uses the official repository's \texttt{dinov3\_vitl16\_pretrain\_lvd1689m} weights, pretrained on LVD-1689M. Each frame is represented by the 1024-dimensional CLS token after the final LayerNorm (\texttt{x\_norm\_clstoken}). Frames are resized directly to $224\times224$ with bilinear interpolation, without center cropping, and ImageNet-normalized. RGB video frames are used as decoded. The 16-bit depth maps from MM-Fi and UTD-MHAD are clipped to 0.1--6\,m, min--max scaled, and replicated to three channels; CZU-MHAD's 8-bit depth frames are replicated to three channels.

\emph{Audio.} IEMOCAP uses XLS-R~1B, \texttt{torchaudio.pipelines.}\allowbreak\texttt{WAV2VEC2\_XLSR\_1B} (\texttt{facebook/wav2vec2-xls-r-1b}), on 16\,kHz mono audio cut or zero-padded to 12\,s. Features come from the last Transformer layer, giving $599\times1280$ values per clip. EAV uses wav2vec~2.0 Base, \texttt{torchaudio.pipelines.}\allowbreak\texttt{WAV2VEC2\_BASE} (\texttt{facebook/wav2vec2-base}), pretrained on LibriSpeech without fine-tuning. Its 20-second, 16\,kHz mono clips produce $999\times768$ last-layer features.

\emph{Text.} Hugging Face's \texttt{bert-base-uncased} provides the 768-dimensional last hidden state of every WordPiece token, including \texttt{[CLS]} and \texttt{[SEP]}, with at most 256 tokens.

\emph{Point clouds.} MM-Fi mmWave and LiDAR use a PointNet shared MLP with hidden widths $64\!\to\!128\!\to\!128$, BatchNorm, ReLU, and max-pooling over points. The weights are randomly initialized with seed~42 and frozen, producing a fixed 128-dimensional random descriptor for each frame's point set. Inputs are five-dimensional mmWave points $(x,y,z,\mathrm{Doppler},\mathrm{intensity})$ and three-dimensional LiDAR points. This descriptor is not a pretrained PointNet checkpoint.

\emph{Raw streams.} No extractor is applied to IEMOCAP MoCap (hand 18-d, head 6-d, rotated 165-d); EAV EEG (30 channels at 500\,Hz, using each recorded 10{,}000-sample trial without filtering); MM-Fi infrared keypoints (34-d) and Wi-Fi CSI (3420-d); UTD-MHAD skeleton (20 joints $\times$ 3 coordinates, 60-d) and inertial data (6-d); CMI IMU (7-d), thermopile (5-d), and five $8\times8$ time-of-flight arrays (64-d each); or CZU-MHAD skeleton and IMU streams, whose construction is detailed below. Extractor outputs and raw streams then undergo the temporal alignment and normalization described in this section.

\smallskip
\noindent\textbf{IEMOCAP.} Five dyadic sessions with ten speakers (one female and one male per session), scripted and improvised. We keep the four emotions \textit{neutral}, \textit{angry}, \textit{happy} and \textit{sad}; utterances labelled otherwise (e.g.\ frustrated, excited) are excluded, leaving 4{,}457 utterances. The class distribution is imbalanced: neutral 1{,}704, angry 1{,}090, sad 1{,}077, happy 586. Splits are by session, hence by speaker (Table~\ref{tab:iemocap_split}). Video is cut per utterance from the dialog recording and 100 frames are sampled uniformly; audio is padded or cut to 12\,s before XLSR-1B; text is the reference transcript; the three motion-capture streams are used raw at their native 120\,Hz.

\begin{table}[h]
\centering\small
\begin{tabular}{lrr}
\toprule
Modality & Sampling rate & Variates \\
\midrule
video (DINOv3 ViT-L/16) & 29.97\,fps, 720$\times$480 & 1024 \\
audio (XLSR-1B) & 16\,kHz, mono & 1280 \\
text (BERT-base) & -- & 768 \\
mocap\_hand & 120\,Hz & 18 \\
mocap\_head & 120\,Hz & 6 \\
mocap\_rotated & 120\,Hz & 165 \\
\bottomrule
\end{tabular}
\hspace{1.5em}
\begin{tabular}{lccc}
\toprule
Fold & Train & Val & Test \\
\midrule
0 & Ses01--03 (2{,}755) & Ses04 (793) & Ses05 (909) \\
1 & Ses02,03,05 (2{,}722) & Ses01 (942) & Ses04 (793) \\
2 & Ses01,04,05 (2{,}644) & Ses02 (813) & Ses03 (1{,}000) \\
\bottomrule
\end{tabular}
\caption{IEMOCAP modalities (left) and cross-session splits (right). Output classes: neutral, angry, happy, sad.}
\label{tab:iemocap_split}
\end{table}

\smallskip
\noindent\textbf{MM-Fi.} Forty subjects in four environments perform 27 actions (14 daily, 13 rehabilitation), one 30\,s recording per subject and action, giving 1{,}080 sequences and exactly 40 per class. All sensors are logged with timestamps and the release is resampled to a unified 10\,Hz frame index, so every modality carries 297 synchronized frames per sequence. We use five modalities (Table~\ref{tab:mmfi}): 2D keypoints from the infrared camera, WiFi CSI amplitude with the official NaN interpolation and per-frame min--max normalization, the mmWave and LiDAR point clouds encoded per frame by the fixed random PointNet descriptor, and the depth map. The stock split leaves one validation subject, too few to fit or select a policy, so validation is enlarged by moving whole subjects out of the test set; train subjects are untouched and all three sets remain subject-disjoint (Table~\ref{tab:mmfi}).

\begin{table}[h]
\centering\small
\begin{tabular}{lp{5.2cm}r}
\toprule
Modality & Native sensing & Variates \\
\midrule
infra1 & RealSense D435 IR, 30\,Hz; 17 keypoints $(u,v)$ & 34 \\
wifi & TP-Link N750, 1$\times$3 antennas, 114 subcarriers; 1000\,Hz raw, 10 packets per frame & 3420 \\
mmwave (PointNet) & TI IWR6843AOP, 30\,Hz; $N\times5$ (x,y,z,Doppler,intensity) & 128 \\
lidar (PointNet) & Ouster OS1-32, 20\,Hz; $N\times3$ & 128 \\
depth (DINOv3) & RealSense D435, 30\,fps, 640$\times$480 16-bit & 1024 \\
\bottomrule
\end{tabular}

\vspace{0.5em}
\begin{tabular}{lccc}
\toprule
Fold & Train & Val & Test \\
\midrule
0 & S14--S39 (702) & S40, S01--S06 (189) & S07--S13 (189) \\
1 & S01--S13, S27--S39 (702) & S40, S14--S19 (189) & S20--S26 (189) \\
2 & S01--S12, S14--S26 (675) & S13, S27--S32 (189) & S33--S40 (216) \\
\bottomrule
\end{tabular}
\caption{MM-Fi modalities and cross-subject splits. All modalities share the 10\,Hz unified frame index (297 frames per 30\,s sequence). Output classes: 27 actions (stretching, chest expansion, twists, mark time, limb extension, lunges, squat, raising and waving hands, picking up, throwing, kicking, body extension, jumping, bowing).}
\label{tab:mmfi}
\end{table}

\smallskip

\noindent\textbf{CZU-MHAD.} Five subjects perform 22 actions with 9--13 repetitions per subject--action pair. We form the index from the union of \texttt{train.csv}, \texttt{val.csv}, and \texttt{test.csv}, deduplicate by \texttt{file\_stem}, and require the depth, skeleton, and sensor \texttt{.mat} files for each trial. Of 1166 depth files, only \texttt{qyh\_a10\_t11} lacks the corresponding skeleton and sensor files and is excluded. No other trials are removed. The resulting 1165 trials contain 52--53 samples per class and 198, 220, 220, 241, and 286 samples for subjects zyh, cyy, myj, qyh, and cx, respectively. Each split holds out one subject for testing and another for validation, giving test-set sizes of 198, 220, and 241 (Table~\ref{tab:czu}).

\paragraph{CZU-MHAD feature construction.}
Depth recordings contain 31--132 frames of $424\times512$ uint8 images. We uniformly sample 100 frames, repeating frames when necessary, extract a 1024-dimensional DINOv3-L CLS vector per frame, and interpolate the resulting $100\times1024$ sequence to $T=128$. Skeleton recordings contain 53--118 frames with 100 columns: 25 Kinect~V2 joints, each represented by $(x,y,z)$ and a microsecond timestamp. The ten IMU streams each contain approximately 1500--2100 rows with seven columns: three accelerometer axes, three gyroscope axes, and a microsecond timestamp. For each IMU pair, the second stream is linearly resampled to the first stream's length and the two seven-column streams are concatenated, producing one of five 14-dimensional sensor streams. Timestamps remain input features in both the skeleton and IMU representations. All streams are linearly resampled to $T=128$ and then divided by their maximum row $\ell_2$ norm within the sequence.

\begin{table}[h]
\centering\small
\begin{tabular}{lrr}
\toprule
Modality & Input / rate & Variates \\
\midrule
depth (DINOv3) & 100 sampled frames & 1024 \\
skeleton & 25 joints & 100 \\
sensor\_head & $\approx$338\,Hz, IMUs 0--1 & 14 \\
sensor\_torso & $\approx$338\,Hz, IMUs 2--3 & 14 \\
sensor\_rarm & $\approx$338\,Hz, IMUs 4--5 & 14 \\
sensor\_larm & $\approx$338\,Hz, IMUs 6--7 & 14 \\
sensor\_leg & $\approx$338\,Hz, IMUs 8--9 & 14 \\
\bottomrule
\end{tabular}
\hspace{1.5em}
\begin{tabular}{lccc}
\toprule
Fold & Train & Val & Test \\
\midrule
0 & cx, cyy, myj (726) & qyh (241) & zyh (198) \\
1 & cx, qyh, zyh (725) & cyy (220) & myj (220) \\
2 & cyy, myj, zyh (638) & cx (286) & qyh (241) \\
\bottomrule
\end{tabular}
\caption{CZU-MHAD input representations (left) and cross-subject splits (right). Each paired-IMU input has $2\times(3+3+1)=14$ features, including timestamps; all streams use $T=128$. Output classes: right/left high wave, right/left horizontal wave, hammer, grasp, draw fork (right/left), draw circle (right/left), forward kick (right/left), side kick (right/left), clap, bend down, wave up and down, sur place, body turn (left/right), lateral movement (left/right).}
\label{tab:czu}
\end{table}

\smallskip

\noindent\textbf{EAV.} Forty-two subjects each complete 200 twenty-second trials, five emotions $\times$ two tasks (listening, speaking) $\times$ 20 repetitions. Audio exists only for the speaking trials, so the trimodal set used here is the 100 speaking trials per subject, 4{,}200 in total and exactly balanced at 840 per class. EEG is used raw; audio is resampled from 44.1\,kHz to 16\,kHz and cut or padded to 20\,s before wav2vec2-base; 100 video frames are sampled uniformly from each clip and encoded by DINOv3. We use three independently randomized cross-subject train/validation/test splits, with seeds 42, 123, and 456 and target proportions 70/15/15. Each split assigns 29 subjects to training, six to validation, and seven to testing, yielding 2900/600/700 speaking trials. Every partition is exactly class-balanced, with 580/120/140 trials per class. Subjects are disjoint within each split, but test sets overlap across splits. Table~\ref{tab:eav-splits} lists the held-out subject IDs; the training subjects are the complement among IDs 1--42. The union of the test sets contains 14 subjects, leaving 28 subjects never used for testing. Thus these are three independent cross-subject splits, rather than a partition of subjects into non-overlapping test folds. We retain fold indices 0--2 as split identifiers in the result tables.

\begin{table}[h]
\centering\small
\begin{tabular}{lrr}
\toprule
Modality & Sampling rate & Variates \\
\midrule
eeg & 500\,Hz, 30 channels & 30 \\
audio (wav2vec2-base) & 44.1\,kHz $\rightarrow$ 16\,kHz, mono & 768 \\
video (DINOv3) & 30\,fps, 640$\times$480 & 1024 \\
\bottomrule
\end{tabular}
\caption{EAV modalities. Output classes: neutral, anger, happiness, sadness, calmness (840 trials each).}
\label{tab:eav}
\end{table}

\begin{table}[t]
\centering\small
\setlength{\tabcolsep}{5pt}
\caption{EAV: three independent cross-subject splits. Validation and test IDs are listed explicitly; the remaining subjects among IDs 1--42 form the training set. Each split contains 29/6/7 train/validation/test subjects. Test subjects can recur across splits.}
\label{tab:eav-splits}
\begin{tabular}{@{}rrll@{}}
\toprule
Split & Seed & Validation subjects & Test subjects \\
\midrule
0 & 42 & 4, 11, 24, 35, 38, 42 & 8, 15, 19, 21, 23, 29, 39 \\
1 & 123 & 1, 10, 15, 32, 33, 34 & 3, 18, 20, 23, 29, 38, 39 \\
2 & 456 & 3, 6, 12, 14, 15, 21 & 16, 18, 26, 28, 29, 38, 39 \\
\bottomrule
\end{tabular}
\end{table}

\smallskip

\noindent\textbf{CMI.} The Child Mind Institute wrist-sensor dataset contains 8{,}151 gesture sequences from 81 participants recorded with the Helios device: an IMU providing three-axis acceleration and an orientation quaternion, five thermopiles, and five $8\times8$ time-of-flight proximity sensors. Sequences average 70.5 samples (29--700); the device sampling rate is not published with the data, so durations in seconds are not reported. Missing time-of-flight readings, encoded as $-1$, are set to zero. Eighteen gestures are labelled, eight body-focused repetitive behaviours and ten everyday actions, with class imbalance. Participants are sorted by identifier and split into three groups of 27 (Table~\ref{tab:cmi}). Class counts are computed after deduplicating \texttt{train.csv} by \texttt{sequence\_id}. Label IDs follow lexicographic order of the gesture names and match the cached targets. Table~\ref{tab:cmi-class-counts} reports all 18 classes; the three test folds contain 2748, 2697 and 2706 sequences and partition all 8151 sequences.

\begin{table}[h]
\centering\small
\begin{tabular}{lrr}
\toprule
Modality & Sensor & Variates \\
\midrule
imu & 3-axis accelerometer + quaternion & 7 \\
thm & 5 thermopiles & 5 \\
tof\_1 -- tof\_5 & five $8\times8$ time-of-flight grids & 64 each \\
\bottomrule
\end{tabular}
\hspace{1.5em}
\begin{tabular}{lccc}
\toprule
Fold & Train & Val & Test \\
\midrule
0 & B[9:]$\cup$C (4{,}485) & B[:9] (918) & A (2{,}748) \\
1 & A$\cup$C[9:] (4{,}536) & C[:9] (918) & B (2{,}697) \\
2 & A[9:]$\cup$B (4{,}528) & A[:9] (917) & C (2{,}706) \\
\bottomrule
\end{tabular}
\caption{CMI modalities (left) and cross-subject splits over participant groups A, B, C of 27 each (right). Output classes: above ear pull hair, cheek pinch skin, eyebrow pull hair, eyelash pull hair, forehead pull hairline, forehead scratch, neck pinch skin, neck scratch (BFRB-like); drink from bottle/cup, feel around in tray, glasses on/off, pinch knee/leg skin, pull air toward face, scratch knee/leg, text on phone, wave hello, write name in air, write name on leg.}
\label{tab:cmi}
\end{table}

\begin{table}[t]
\centering\footnotesize
\setlength{\tabcolsep}{3pt}
\renewcommand{\arraystretch}{1.05}
\caption{CMI class distribution after deduplication by sequence ID. BFRB denotes body-focused repetitive behaviour. IDs follow lexicographic gesture-name order. The three test-fold counts sum to the full class count; percentages are rounded independently.}
\label{tab:cmi-class-counts}
\begin{tabular}{@{}rp{5.8cm}lrrrrr@{}}
\toprule
ID & Gesture & Type & Total & Share (\%) & Fold 0 & Fold 1 & Fold 2 \\
\midrule
0 & Above ear -- pull hair & BFRB & 638 & 7.8 & 216 & 210 & 212 \\
1 & Cheek -- pinch skin & BFRB & 637 & 7.8 & 215 & 210 & 212 \\
2 & Drink from bottle/cup & non-BFRB & 161 & 2.0 & 54 & 53 & 54 \\
3 & Eyebrow -- pull hair & BFRB & 638 & 7.8 & 216 & 210 & 212 \\
4 & Eyelash -- pull hair & BFRB & 640 & 7.9 & 216 & 212 & 212 \\
5 & Feel around in tray and pull out an object & non-BFRB & 161 & 2.0 & 54 & 53 & 54 \\
6 & Forehead -- pull hairline & BFRB & 640 & 7.9 & 216 & 212 & 212 \\
7 & Forehead -- scratch & BFRB & 640 & 7.9 & 216 & 212 & 212 \\
8 & Glasses on/off & non-BFRB & 161 & 2.0 & 54 & 53 & 54 \\
9 & Neck -- pinch skin & BFRB & 640 & 7.9 & 216 & 212 & 212 \\
10 & Neck -- scratch & BFRB & 640 & 7.9 & 216 & 212 & 212 \\
11 & Pinch knee/leg skin & non-BFRB & 161 & 2.0 & 54 & 53 & 54 \\
12 & Pull air toward your face & non-BFRB & 477 & 5.9 & 160 & 159 & 158 \\
13 & Scratch knee/leg skin & non-BFRB & 161 & 2.0 & 54 & 53 & 54 \\
14 & Text on phone & non-BFRB & 640 & 7.9 & 216 & 212 & 212 \\
15 & Wave hello & non-BFRB & 478 & 5.9 & 161 & 159 & 158 \\
16 & Write name in air & non-BFRB & 477 & 5.9 & 160 & 159 & 158 \\
17 & Write name on leg & non-BFRB & 161 & 2.0 & 54 & 53 & 54 \\
\midrule
\multicolumn{3}{l}{Total} & 8151 & 100.0 & 2748 & 2697 & 2706 \\
\bottomrule
\end{tabular}
\end{table}

\smallskip
\noindent\textbf{UTD-MHAD.} Eight subjects (four female, four male) perform 27 actions four times each; after removing three corrupted trials, 861 samples remain, 31--32 per class and 107--108 per subject. Skeleton, depth and RGB come from a Kinect~v1; the release carries depth and skeleton at an effective $\approx$20\,fps and RGB at 15\,fps. The inertial sensor is worn on the right wrist for actions 1--21 and the right thigh for actions 22--27. We report the rebalanced cross-subject protocol with four training, two validation and two test subjects per fold (Table~\ref{tab:utd}); it keeps the original test subjects and replaces the original single validation subject, which was too small for policy selection, with two.

\begin{table}[h]
\centering\small
\begin{tabular}{lrr}
\toprule
Modality & Sampling rate & Variates \\
\midrule
skeleton & Kinect, 20 joints $\times$ (x,y,z) & 60 \\
inertial & 50\,Hz, 3-axis accel + 3-axis gyro & 6 \\
depth (DINOv3) & Kinect, 320$\times$240 & 1024 \\
rgb (DINOv3) & 15\,fps, 640$\times$480 & 1024 \\
\bottomrule
\end{tabular}
\hspace{1.5em}
\begin{tabular}{lccc}
\toprule
Fold & Train & Val & Test \\
\midrule
3 & s3--s6 (431) & s7, s8 (215) & s1, s2 (215) \\
4 & s1, s2, s5, s6 (430) & s7, s8 (215) & s3, s4 (216) \\
5 & s1--s4 (431) & s5, s6 (215) & s7, s8 (215) \\
\bottomrule
\end{tabular}
\caption{UTD-MHAD modalities (left) and 4/2/2 cross-subject splits (right). Output classes: swipe left, swipe right, wave, clap, throw, arm cross, basketball shoot, draw X, draw circle (CW), draw circle (CCW), draw triangle, bowling, boxing, baseball swing, tennis swing, arm curl, tennis serve, push, knock, catch, pickup and throw, jog, walk, sit to stand, stand to sit, lunge, squat.}
\label{tab:utd}
\end{table}

\section{Complete Baseline Implementation Details}\label{app:baselines}

\subsection{Baseline Implementations and Adaptations}
\label{app:baseline_adaptations}

None of the baselines was designed for the setting evaluated here: per-sample acquisition of whole modalities from raw sensor streams, under cross-subject splits and test-time missingness. Each therefore needed some adaptation. We describe, per method, what is carried over from the original, what was changed, and why. The following descriptions distinguish the original methods from adaptations to our inputs and evaluation protocol. These adaptations can affect performance.

\smallskip

\noindent\textbf{Common training protocol.} All learned baselines are trained from scratch on the same per-modality features, the same cross-subject folds and the same data loader as our backbone, with the following default configuration, subject to the method-specific exceptions below: $d_{\text{model}}=128$, eight heads, six encoder layers per modality, Adam at $10^{-4}$, batch size 32, the same epochs as \method{}, and model selection on validation macro-F1. Every baseline that has a notion of missing modalities is trained with the same reported modality-dropout schedule as \method{}. Where a method has its own missingness mechanism (ShaSpec, MAESTRO), that mechanism is kept and the curriculum drives it. 

\smallskip

\noindent\textbf{Execution under missingness.} Every baseline skips raw-data extraction for unavailable modalities. This is distinct from skipping a downstream encoder: the non-adaptive MBT, MAESTRO, ShaSpec, DecAlign, MultiModN and MMEE implementations retain their modality encoders and process zero-filled inputs at absent positions. Among AdaMML, DynMM and DyMo, AdaMML and DyMo can skip encoders, whereas DynMM cannot. MultiModN skips a missing modality's state update but still executes its encoder. Acquisition methods mask unavailable actions. Modality utilization is normalized by the number of modalities available to the current sample, rather than the nominal total $M$.

\emph{Vanilla MBT} is a plain Multimodal Bottleneck Transformer: the same shared-bottleneck-token fusion our backbone uses, with four bottleneck tokens introduced at layer three, but with no sequential ordering, no early exit and no policy. It is the non-adaptive counterpart of our fusion at identical width and depth.

\emph{MAESTRO} is run from the authors' released model with every layer, buffer and parameter name unchanged, so the published checkpoints load strictly. It keeps its own ProbSparse attention, mixture-of-experts feed-forward and modality-dropout schedule with a learned dynamic factor; the only change is the input adapter that splits our concatenated per-modality tensor.

\emph{ShaSpec} uses the original implementation, including its shared and modality-specific representations and missing-modality mechanism. The surrogate shared representation is the mean of the available shared features.

\emph{DecAlign} is generalised from its original three modalities (text, audio, video) to $M$ modalities. The heterogeneity alignment scales accordingly: the multi-marginal Sinkhorn transport runs over an $M$-dimensional cost tensor, and the directed cross-attention branch has $M(M-1)$ pairs. Everything else, including the three-path fusion and the decoupling loss, is as published.

\emph{MultiModN} is a faithful port of the sequential modular network: a shared state vector updated one modality at a time by an un-gated MLP that concatenates the modality feature with the running state, and a shared decoder readable after every prefix. Its per-modality encoder is the same six-layer Transformer as the other baselines, and all prefixes are supervised during training. A missing modality skips its state update, as in the original mechanism; raw-data extraction is skipped, while its downstream encoder still processes the zero-filled input. Because the model is order-sensitive, we report it in the order chosen on validation by exhaustive search over all $M!$ orders.

\smallskip
\noindent\textbf{Modality-selective baselines.} These make a per-sample decision about which modalities to process, and are measured live: the same forward pass that produces the prediction is what is timed and whose executed-path compute is reported. Under missingness, raw-data extraction is skipped for absent modalities, and unavailable actions are masked. Downstream encoder skipping is method-specific: AdaMML and DyMo can skip encoders, whereas DynMM retains encoder execution.

\emph{AdaMML} keeps the official policy network: a lightweight network previews a cheap summary of the input and emits a per-modality keep/drop decision by hard, straight-through Gumbel-softmax with an annealed temperature. Two changes, both forced by the setting. The decision is made once per sample rather than per segment through an LSTM, since our inputs are single windows rather than long videos. The main network is a set of per-modality encoders with logit-level late fusion under learned weights, so that a dropped modality's encoder is genuinely skipped at inference and the saving is real. Training follows the original three phases: encoder warm-up, joint policy and encoder training, and fine-tuning.

\emph{DynMM} mirrors the official affect-recognition implementation: nested modality-subset experts, a per-sample gate trained with the FLOP-resource regulariser as its only auxiliary loss, and no distillation. Each expert is a set of MultiBench-style Transformer encoders with concatenation fusion. We use the two-branch configuration, a single-modality expert and a full expert, with the resource loss weight at 0.1 and a soft gate during training that is made hard at inference. 

\emph{DyMo} uses the original implementation, retaining its missing-modality recovery and per-sample selection mechanisms on the repository's cross-modal Transformer backbone. The configuration uses prototype-distance quality with cosine temperature 0.1, per-subset Gaussian probabilities, the prototype training loss, per-epoch prototype recomputation as class means of projected class-token features, and greedy per-sample selection with reward $R(u)=\text{quality}(S\cup\{u\})\,G-\text{quality}(S)$, with $K=4$ and a 16-dimensional projection. Missing raw streams are not extracted, and the implementation can skip downstream encoders.

\smallskip
\noindent\textbf{Early-exit baseline.} \emph{MMEE} ports the anytime early-exit mechanism, off-ramp heads on intermediate fusion layers with weighted multi-exit training, onto a generic concatenate-then-fuse Transformer. It exits in the depth axis, complementary to our exit in the modality axis, at the first head whose maximum softmax reaches 0.9. All available modalities are extracted, while missing raw streams are skipped. Downstream modality encoders still execute on zero-filled inputs for absent modalities. Its usage is 1.0 relative to the available modalities; early exit reduces fusion depth rather than selecting modalities.

\smallskip
\noindent\textbf{Selection controllers on the shared backbone.} Table~\ref{tab:comp-replace-complete} compares controllers coupled to the same frozen \backbone{}. EAMA and Joint-AFA are placed before the frozen backbone. These comparisons evaluate the shared-backbone adaptations, with the original action spaces retained at modality granularity.

\emph{Greedy submodular selection} is adapted to the frozen backbone. The reported implementation greedily adds the modality with the largest estimated gain using $H(Y)-N^{-1}\sum_n H(p(y\mid x^n_S))$ on validation samples, averaging over cached orders. This predictive-entropy utility differs from the converged cross-entropy estimate described for the own-backbone implementation below and is an approximation rather than an identical reproduction. The budget is selected on validation. The GMS comparison evaluates the static subset selected under its configured budget as a static-selection baseline for \method{}.

\emph{Joint-AFA} adapts the authors' released implementation from scalar features to modalities and keeps its original placement: the agent observes only the inputs acquired so far. It uses the same $M+1$ discrete actions as \policy{}: acquire modality $j$, for any modality not yet acquired, or $\mathrm{STOP}$. Acquired modalities are masked. The agent never reads a \backbone{} prediction during acquisition. Each set element combines a one-hot modality index with a per-modality linear projection of that modality's raw preview: its input stream mean-pooled over time and z-scored with validation statistics. The attentive LSTM set encoder reads the acquired elements, and Joint-AFA's dueling head scores the $M+1$ actions. We retain the log-likelihood terminal reward, per-acquisition cost, $\epsilon$-greedy schedule from 1.0 to 0.1, and target network. After $\mathrm{STOP}$, or once every modality is acquired, the frozen \backbone{} fuses the final set once. Its prediction, which also defines the terminal reward, is read from the prefix cache of that set. The agent is trained on the validation split with a per-acquisition cost of 0.05.

\emph{EAMA} follows Algorithm~1 of EAMA: the selector acts on the raw inputs of the observed modalities, and the classifier is applied once, to the final acquired set. At stage $s$, let $p_s$ be the observed modalities and $U_s$ the unobserved ones. Each acquisition action is a nonempty subset $q\subseteq U_s$, yielding $2^{|U_s|}-1$ candidates: 63 initially on IEMOCAP and 127 on CMI and CZU-MHAD. The stage optimizer maximizes
\[
J_s(q)=\operatorname{score}(q\mid p_s)-\lambda\sum_{j\in q}z_j,
\qquad \varnothing\ne q\subseteq U_s,
\]where the value model predicts $\operatorname{score}=P(\mathrm{gain})-P(\mathrm{loss})$ from the raw previews of $p_s$, defined as for Joint-AFA and averaged after per-modality linear projections, together with unit-norm embeddings of $p_s$ and $q$. It never reads a \backbone{} prediction. Our datasets have no always-available context modality, so $p_1=\varnothing$ and the first stage selects the same subset for every sample. $\mathrm{STOP}$ is implicit: acquisition ends if the best returned objective is nonpositive or no modality remains unobserved. The reported controller uses the paper's hybrid optimizer, randomized double greedy and greedy maximization, retaining the better candidate. We also solve the stage objective exactly by enumeration as an optimization upper bound; this does not replace the hybrid optimizer in the reported EAMA row. When acquisition ends, the frozen \backbone{} fuses the final set once, and its prediction is read from the cache of that set. Thus EAMA can add several modalities in one decision, whereas \policy{} and Joint-AFA add exactly one. We retain the three-class value target, computed from the correctness of this readout on $p_s$ and $p_s\cup q$, Algorithm~2 view sampling, and unit-norm subset embeddings. The value model is fitted on one part of the validation split; early stopping and the per-fold choice of $\lambda$ use the other. The unqualified EAMA row uses uniform costs ($z_j=z$ for all modalities). The separately labelled weighted-cost row uses modality-specific extraction-cost fractions.

\emph{MultiModN+ARC} attaches our acquisition policy to the MultiModN backbone instead of ours, with the prefix cache rebuilt from MultiModN's own per-prefix decoder. This tests whether the policy's gains transfer to a backbone that was not trained to be order-robust.

\smallskip
\noindent\textbf{Own-backbone variants.} For the main table, greedy submodular selection and JAFA are also run end to end with their own classifiers trained from scratch, so that each method owns its full pipeline as its paper assumes. Both use the same neutral per-modality encoder, two 1D convolutions with global average pooling, which is EAMA's stated time-series encoder, the natural sequence analogue of JAFA's per-element embedder, and a mild strengthening of the MLP Cheng et al.\ use. For greedy selection, $H(Y\mid S)$ is estimated by the converged training loss of a model trained on $S$, as in the paper's Section~4.2, with one model per budget. JAFA is trained as the released code trains it: classifier pre-training on Bernoulli(0.5) subsets, then joint training with $\epsilon$-greedy episodes, terminal reward $\log p(y\mid\text{acquired})$, cost $0.05$ per acquisition, smooth-L1 loss, periodic target copies and double DQN. Under missingness, a selected modality that is absent is not acquired and its encoder is skipped; JAFA's episodes mask unavailable actions.

\smallskip
\noindent\textbf{Ours without ARC.} The backbone alone, consuming every available modality in the fold's validation leave-one-out order and predicting once. This provides a full-modality reference for evaluating the policy’s accuracy–cost trade-off.

\smallskip
\noindent\textbf{Missingness protocol.} Test-time missingness is simulated by dropping each modality of each sample independently with probability $\rho\in\{0.1,0.2,0.3,0.4\}$, re-instating one uniformly chosen modality for any sample that would lose all of them, with a fixed seed per fold. Every method sees the same masks. All methods skip raw-data extraction for absent modalities. Fixed-width baselines process zero-filled features in their downstream encoders; AdaMML and DyMo can skip encoders, whereas DynMM cannot. Selective baselines and our policy are forbidden from choosing an unavailable modality. MU uses the number of available modalities for each sample as its denominator.

\section{Evaluation Metrics}\label{app:metrics}
\paragraph{Predictive performance.}
Macro-F1 is computed separately within each test fold using scikit-learn~1.5.2, \texttt{f1\_score(average='macro')}, with no explicit \texttt{labels} argument and the default \texttt{zero\_division='warn'}. The label set for fold $f$ is $\mathcal C_f=\operatorname{unique}(y_{\mathrm{true},f}\cup y_{\mathrm{pred},f})$. Thus,
\[
\mathrm{F1}_{\mathrm{macro},f}=\frac{1}{|\mathcal C_f|}
\sum_{c\in\mathcal C_f}\frac{2\mathrm{TP}_{f,c}}{2\mathrm{TP}_{f,c}+\mathrm{FP}_{f,c}+\mathrm{FN}_{f,c}}.
\]
The default zero-division rule assigns zero to an undefined class F1 and emits a warning. It is not triggered here: each included class occurs in either the true labels or the predictions, so its F1 denominator is positive. Classes absent from both are omitted from that fold's macro average. Reported results average the three fold-level scores equally rather than pooling predictions across folds. Unless explicitly labelled accuracy, predictive results report macro-F1.

\paragraph{Computation.}
All GFLOPs use \texttt{fvcore.FlopCountAnalysis}, whose convention counts one multiply-add as one FLOP; these values are therefore MAC-equivalent operation counts. Total raw-to-logit cost is reported as extraction cost plus the executed downstream network cost $\mathrm{GF}_{\mathrm f}$. Extraction is counted per unit and multiplied by the number of units actually read for the sample: DINOv3 costs 60.9 GFLOPs per frame, XLS-R-1B costs 603.6 GFLOPs per 12-second segment, and wav2vec2-base costs 139.3 GFLOPs per 20-second segment. BERT cost depends on token count, and PointNet cost uses the mean number of points per frame.

Fusion cost is counted per executed branch step and accumulated along the realized prefix. An extraction or fusion step that is skipped contributes no operations. This follows actual execution: a non-adaptive encoder that still processes a zero-filled input remains counted, even though the absent raw stream is not extracted. The \policy{} selector evaluates an approximately 6k-parameter MLP once per step, costing approximately $10^{-5}$ GFLOPs. AdaMML, DynMM and DyMo gating operations are inside their traced graphs and are included by fvcore. Reported total GFLOPs include extraction, downstream encoding and fusion, and controller computation along the executed path.

\paragraph{Modality utilization.}
For every method, $\mathrm{MU}=N^{-1}\sum_i |S_i|/|\mathcal A_i|$, where $\mathcal A_i$ is the nonempty set of modalities available to sample $i$ under the current missingness mask and $S_i\subseteq\mathcal A_i$ is the set acquired or retained by the method. Methods that use every available modality have $\mathrm{MU}=1$ even when some modalities are missing. At full availability, $|\mathcal A_i|=M$. MU measures modality use, not the number of downstream encoders executed; a method can retain fewer modalities without skipping its encoders.
\paragraph{Latency and energy.}
Measured latency is elapsed inference time per sample. GPU and CPU energy measurements use cumulative hardware-counter differences; phone energy integrates battery power and subtracts an idle baseline. Component-replacement latency is measured end-to-end. 

\paragraph{Three-fold summary convention.}
For the full-availability experiment, the appendix fold entries are the numerical source of truth. For EAV, the three entries correspond to independent cross-subject splits with overlapping test subjects; their standard deviation measures variation across these splits. We compute $\bar x=\frac{1}{3}\sum_{r=0}^{2}x_r$ and the F1 sample standard deviation $s=\sqrt{\frac{1}{2}\sum_{r=0}^{2}(x_r-\bar x)^2}$. Means are equally weighted across folds. Calculations retain the precision of the listed fold entries until final display rounding: F1, its standard deviation, and post-extraction GFLOPs use two decimal places; total GFLOPs use two decimals below 100 and integers otherwise. Halfway values round upward. Rankings use distinct displayed values, with ties sharing marks. Six-dataset full-availability gains use the recomputed means before display rounding; if baseline F1 means tie, the lower-total-GF baseline is selected. No hidden precision is reconstructed from rounded fold values.

\paragraph{Aggregate comparisons.}
For a reference value $b$ and our value $o$, a relative F1 gain is $100(o/b-1)\%$; a cost reduction is $100(1-o/b)\%$. F1 differences in percentage points are $100(o-b)$. Dataset means weight the six datasets equally. Hardware reductions compare against the best measured baseline for each metric in each dataset--platform condition before averaging.






\section{Complete Results} \label{app:add_results}

\subsection{Complete results for Table~\ref{tab:full-results}} \label{app:full_results}

\paragraph{Per-fold results.}

Tables~\ref{tab:perfold-iemocap}--\ref{tab:perfold-utd_mhad} report, for each of the six datasets, the per-fold numbers behind the fold means of Table~\ref{tab:full-results}: macro-F1, modality usage MU (the fraction of modalities acquired per test sample, averaged over the test set) and end-to-end GFLOPs per sample, with the fusion cost $\mathrm{GF}_{\mathrm{f}}$ in parentheses. Every value is measured under full modality availability on the same cross-subject test folds as the main table; the main table and the 0\% missingness entries are now recomputed directly from these three-fold values. End-to-end GFLOPs count the whole pipeline from raw data to logits: the encoder cost of every modality a method actually reads, plus its fusion cost. The results show the fact that \method{} keeps the best or second-best F1 on almost every fold rather than winning on average through one favourable split. Within each column, bold and underlining mark the best and second-best distinct displayed values (highest F1; lowest MU, GF and $\mathrm{GF}_{\mathrm{f}}$); ties share markings.

\begin{table*}[!htbp]
\centering\footnotesize\setlength{\tabcolsep}{2.5pt}\renewcommand{\arraystretch}{1.12}
\caption{\small \textbf{IEMOCAP}, per fold, full modality availability: macro-F1 ($\uparrow$), modality usage MU ($\downarrow$) and end-to-end GFLOPs ($\downarrow$; fusion cost $\mathrm{GF}_{\mathrm{f}}$ in parentheses). Fold means and sample standard deviations of these values are the cells of Table~\ref{tab:full-results}; F1 is shown to three decimals. Per column, \textcolor{burgundy}{\textbf{bold}} and \underline{underlining} mark the best and second-best distinct displayed values (F1 highest; MU, GF and $\mathrm{GF}_{\mathrm{f}}$ lowest, GF compared on the displayed total); ties share markings.}
\label{tab:perfold-iemocap}
\vspace{-6pt}
\resizebox{\textwidth}{!}{%
\begin{tabular}{@{}ll*{3}{ccc}@{}}
\toprule
\multirow{2}{*}{\textbf{Type}} & \multirow{2}{*}{\textbf{Method}} & \multicolumn{3}{c}{\textbf{Fold 0}} & \multicolumn{3}{c}{\textbf{Fold 1}} & \multicolumn{3}{c}{\textbf{Fold 2}} \\
\cmidrule(lr){3-5} \cmidrule(lr){6-8} \cmidrule(lr){9-11}
& & F1 $\uparrow$ & MU $\downarrow$ & $\mathrm{GF}\,(\mathrm{GF}_{\mathrm{f}})\downarrow$ & F1 $\uparrow$ & MU $\downarrow$ & $\mathrm{GF}\,(\mathrm{GF}_{\mathrm{f}})\downarrow$ & F1 $\uparrow$ & MU $\downarrow$ & $\mathrm{GF}\,(\mathrm{GF}_{\mathrm{f}})\downarrow$ \\
\midrule
\multirow{4}{*}{\textbf{Monolithic}} & MAESTRO & .472 & 1.0 & 6696 (.23) & .559 & 1.0 & 6696 (.23) & .496 & 1.0 & 6696 (.23) \\
 & ShaSpec & .623 & 1.0 & 6697 (1.51) & \underline{.636} & 1.0 & 6697 (1.51) & .584 & 1.0 & 6697 (1.51) \\
 & MBT & \underline{.656} & 1.0 & 6696 (.41) & .627 & 1.0 & 6696 (.41) & .626 & 1.0 & 6696 (.41) \\
 & DecAlign & .470 & 1.0 & 6707 (11.41) & .484 & 1.0 & 6707 (11.41) & .486 & 1.0 & 6707 (11.41) \\
\arrayrulecolor{black!20}\midrule
\textbf{Fixed-seq.} & MultiModN & .514 & 1.0 & 6696 (.76) & .602 & 1.0 & 6696 (.76) & \underline{.632} & 1.0 & 6696 (.76) \\
\arrayrulecolor{black!20}\midrule
\multirow{5}{*}{\textbf{Selective}} & GMS & .628 & \textcolor{burgundy}{\textbf{.17}} & \textcolor{burgundy}{\textbf{1.51}} (\textcolor{burgundy}{\textbf{.07}}) & .471 & .83 & 6696 (.30) & .515 & 1.0 & 6696 (.31) \\
 & AdaMML & .568 & .50 & 6696 (.20) & .633 & .50 & 6696 (\underline{.20}) & .583 & .50 & 6696 (\underline{.20}) \\
 & DynMM & .625 & 1.0 & 6698 (2.22) & .623 & .70 & 6697 (1.56) & .612 & .70 & 6697 (1.56) \\
 & DyMo & .502 & \underline{.27} & 6696 (.47) & .591 & \underline{.35} & 6696 (.49) & .583 & \underline{.33} & 6696 (.48) \\
 & MMEE & .625 & 1.0 & 6696 (.41) & .597 & 1.0 & 6696 (.37) & .604 & 1.0 & 6696 (.37) \\
\arrayrulecolor{black!20}\midrule
\textbf{Proactive} & Joint-AFA & .613 & .32 & \underline{2097} (\underline{.16}) & .562 & \textcolor{burgundy}{\textbf{.24}} & \textcolor{burgundy}{\textbf{1367}} (\textcolor{burgundy}{\textbf{.11}}) & .542 & \textcolor{burgundy}{\textbf{.22}} & \textcolor{burgundy}{\textbf{1501}} (\textcolor{burgundy}{\textbf{.11}}) \\
\arrayrulecolor{black}\midrule
\rowcolor{tolblue!10}\multicolumn{2}{l}{\textbf{\method{} (Ours)}} & \textcolor{burgundy}{\textbf{.682}} & .59 & 2691 (.26) & \textcolor{burgundy}{\textbf{.682}} & .79 & \underline{3820} (.34) & \textcolor{burgundy}{\textbf{.643}} & .57 & \underline{3423} (.25) \\
\bottomrule
\end{tabular}%
}
\end{table*}

\begin{table*}[!htbp]
\centering\footnotesize\setlength{\tabcolsep}{2.5pt}\renewcommand{\arraystretch}{1.12}
\caption{\small \textbf{MM-Fi}, per fold, full modality availability: macro-F1 ($\uparrow$), modality usage MU ($\downarrow$) and end-to-end GFLOPs ($\downarrow$; fusion cost $\mathrm{GF}_{\mathrm{f}}$ in parentheses). Fold means and sample standard deviations of these values are the cells of Table~\ref{tab:full-results}; F1 is shown to three decimals. Per column, \textcolor{burgundy}{\textbf{bold}} and \underline{underlining} mark the best and second-best distinct displayed values (F1 highest; MU, GF and $\mathrm{GF}_{\mathrm{f}}$ lowest, GF compared on the displayed total); ties share markings.}
\label{tab:perfold-mmfi}
\vspace{-6pt}
\resizebox{\textwidth}{!}{%
\begin{tabular}{@{}ll*{3}{ccc}@{}}
\toprule
\multirow{2}{*}{\textbf{Type}} & \multirow{2}{*}{\textbf{Method}} & \multicolumn{3}{c}{\textbf{Fold 0}} & \multicolumn{3}{c}{\textbf{Fold 1}} & \multicolumn{3}{c}{\textbf{Fold 2}} \\
\cmidrule(lr){3-5} \cmidrule(lr){6-8} \cmidrule(lr){9-11}
& & F1 $\uparrow$ & MU $\downarrow$ & $\mathrm{GF}\,(\mathrm{GF}_{\mathrm{f}})\downarrow$ & F1 $\uparrow$ & MU $\downarrow$ & $\mathrm{GF}\,(\mathrm{GF}_{\mathrm{f}})\downarrow$ & F1 $\uparrow$ & MU $\downarrow$ & $\mathrm{GF}\,(\mathrm{GF}_{\mathrm{f}})\downarrow$ \\
\midrule
\multirow{4}{*}{\textbf{Monolithic}} & MAESTRO & .232 & 1.0 & 18098 (.28) & .229 & 1.0 & 18098 (.28) & .067 & 1.0 & 18098 (.28) \\
 & ShaSpec & .713 & 1.0 & 18099 (1.68) & .691 & 1.0 & 18099 (1.68) & .570 & 1.0 & 18099 (1.68) \\
 & MBT & .824 & 1.0 & 18098 (.46) & .929 & 1.0 & 18098 (.46) & .557 & 1.0 & 18098 (.46) \\
 & DecAlign & .120 & 1.0 & 18106 (8.45) & .492 & 1.0 & 18106 (8.45) & .236 & 1.0 & 18106 (8.45) \\
\arrayrulecolor{black!20}\midrule
\textbf{Fixed-seq.} & MultiModN & .429 & 1.0 & 18098 (.84) & .437 & 1.0 & 18098 (.84) & .323 & 1.0 & 18098 (.84) \\
\arrayrulecolor{black!20}\midrule
\multirow{5}{*}{\textbf{Selective}} & GMS & .424 & .40 & 18089 (.46) & .490 & .60 & \underline{18089} (.48) & .274 & 1.0 & 18098 (.53) \\
 & AdaMML & \textcolor{burgundy}{\textbf{.960}} & .40 & 18098 (\underline{.17}) & \underline{.977} & .40 & 18098 (.17) & \textcolor{burgundy}{\textbf{.937}} & .40 & 18098 (.17) \\
 & DynMM & .856 & .74 & 18099 (1.79) & .904 & .78 & 18099 (1.88) & .367 & \textcolor{burgundy}{\textbf{.20}} & 18098 (.53) \\
 & DyMo & .574 & .42 & 18098 (.52) & .527 & .41 & 18098 (.52) & .512 & .42 & 18098 (.52) \\
 & MMEE & .597 & 1.0 & 18098 (.59) & .643 & 1.0 & 18098 (.59) & .546 & 1.0 & 18098 (.59) \\
\arrayrulecolor{black!20}\midrule
\textbf{Proactive} & Joint-AFA & .003 & \textcolor{burgundy}{\textbf{.20}} & \textcolor{burgundy}{\textbf{0.35}} (.35) & .195 & \textcolor{burgundy}{\textbf{.20}} & \underline{18089} (\underline{.12}) & .141 & .23 & \underline{18089} (\underline{.16}) \\
\arrayrulecolor{black}\midrule
\rowcolor{tolblue!10}\multicolumn{2}{l}{\textbf{\method{} (Ours)}} & \underline{.946} & \underline{.21} & \underline{861} (\textcolor{burgundy}{\textbf{.07}}) & \textcolor{burgundy}{\textbf{.989}} & \underline{.22} & \textcolor{burgundy}{\textbf{766}} (\textcolor{burgundy}{\textbf{.07}}) & \underline{.915} & \underline{.21} & \textcolor{burgundy}{\textbf{586}} (\textcolor{burgundy}{\textbf{.07}}) \\
\bottomrule
\end{tabular}%
}
\end{table*}

\begin{table*}[!htbp]
\centering\footnotesize\setlength{\tabcolsep}{2.5pt}\renewcommand{\arraystretch}{1.12}
\caption{\small \textbf{CMI}, per fold, full modality availability: macro-F1 ($\uparrow$), modality usage MU ($\downarrow$) and end-to-end GFLOPs ($\downarrow$; fusion cost $\mathrm{GF}_{\mathrm{f}}$ in parentheses). Fold means and sample standard deviations of these values are the cells of Table~\ref{tab:full-results}; F1 is shown to three decimals. Per column, \textcolor{burgundy}{\textbf{bold}} and \underline{underlining} mark the best and second-best distinct displayed values (F1 highest; MU, GF and $\mathrm{GF}_{\mathrm{f}}$ lowest, GF compared on the displayed total); ties share markings.}
\label{tab:perfold-cmi}
\vspace{-6pt}
\resizebox{\textwidth}{!}{%
\begin{tabular}{@{}ll*{3}{ccc}@{}}
\toprule
\multirow{2}{*}{\textbf{Type}} & \multirow{2}{*}{\textbf{Method}} & \multicolumn{3}{c}{\textbf{Fold 0}} & \multicolumn{3}{c}{\textbf{Fold 1}} & \multicolumn{3}{c}{\textbf{Fold 2}} \\
\cmidrule(lr){3-5} \cmidrule(lr){6-8} \cmidrule(lr){9-11}
& & F1 $\uparrow$ & MU $\downarrow$ & $\mathrm{GF}\,(\mathrm{GF}_{\mathrm{f}})\downarrow$ & F1 $\uparrow$ & MU $\downarrow$ & $\mathrm{GF}\,(\mathrm{GF}_{\mathrm{f}})\downarrow$ & F1 $\uparrow$ & MU $\downarrow$ & $\mathrm{GF}\,(\mathrm{GF}_{\mathrm{f}})\downarrow$ \\
\midrule
\multirow{4}{*}{\textbf{Monolithic}} & MAESTRO & .456 & 1.0 & 0.28 (.28) & .487 & 1.0 & 0.28 (.28) & .464 & 1.0 & 0.28 (.28) \\
 & ShaSpec & .544 & 1.0 & 2.14 (2.14) & .516 & 1.0 & 2.14 (2.14) & .539 & 1.0 & 2.14 (2.14) \\
 & MBT & \underline{.561} & 1.0 & 0.55 (.55) & \underline{.543} & 1.0 & 0.55 (.55) & .546 & 1.0 & 0.55 (.55) \\
 & DecAlign & \textcolor{burgundy}{\textbf{.565}} & 1.0 & 23.33 (23.33) & \underline{.543} & 1.0 & 23.33 (23.33) & \textcolor{burgundy}{\textbf{.575}} & 1.0 & 23.33 (23.33) \\
\arrayrulecolor{black!20}\midrule
\textbf{Fixed-seq.} & MultiModN & .445 & 1.0 & 1.07 (1.07) & .452 & 1.0 & 1.07 (1.07) & .441 & 1.0 & 1.07 (1.07) \\
\arrayrulecolor{black!20}\midrule
\multirow{5}{*}{\textbf{Selective}} & GMS & .493 & 1.0 & \underline{0.12} (\underline{.12}) & .471 & .86 & \underline{0.11} (\underline{.11}) & .484 & .71 & \underline{0.09} (\underline{.09}) \\
 & AdaMML & .516 & .83 & 0.51 (.51) & .505 & .86 & 0.53 (.53) & .525 & .80 & 0.49 (.49) \\
 & DynMM & .546 & .94 & 3.04 (3.04) & .528 & .98 & 3.17 (3.17) & .532 & .98 & 3.17 (3.17) \\
 & DyMo & .371 & \textcolor{burgundy}{\textbf{.28}} & 0.65 (.65) & .362 & \textcolor{burgundy}{\textbf{.30}} & 0.66 (.66) & .375 & \textcolor{burgundy}{\textbf{.28}} & 0.65 (.65) \\
 & MMEE & .486 & 1.0 & 0.70 (.70) & .472 & 1.0 & 0.70 (.70) & .481 & 1.0 & 0.70 (.70) \\
\arrayrulecolor{black!20}\midrule
\textbf{Proactive} & Joint-AFA & .361 & \underline{.36} & \textcolor{burgundy}{\textbf{0.06}} (\textcolor{burgundy}{\textbf{.06}}) & .359 & \underline{.38} & \textcolor{burgundy}{\textbf{0.07}} (\textcolor{burgundy}{\textbf{.07}}) & .419 & \underline{.46} & \textcolor{burgundy}{\textbf{0.08}} (\textcolor{burgundy}{\textbf{.08}}) \\
\arrayrulecolor{black}\midrule
\rowcolor{tolblue!10}\multicolumn{2}{l}{\textbf{\method{} (Ours)}} & .544 & .40 & 0.23 (.23) & \textcolor{burgundy}{\textbf{.562}} & .94 & 0.54 (.54) & \underline{.561} & .91 & 0.53 (.53) \\
\bottomrule
\end{tabular}%
}
\end{table*}

\begin{table*}[!htbp]
\centering\footnotesize\setlength{\tabcolsep}{2.5pt}\renewcommand{\arraystretch}{1.12}
\caption{\small \textbf{CZU-MHAD}, per fold, full modality availability: macro-F1 ($\uparrow$), modality usage MU ($\downarrow$) and end-to-end GFLOPs ($\downarrow$; fusion cost $\mathrm{GF}_{\mathrm{f}}$ in parentheses). Fold means and sample standard deviations of these values are the cells of Table~\ref{tab:full-results}; F1 is shown to three decimals. Per column, \textcolor{burgundy}{\textbf{bold}} and \underline{underlining} mark the best and second-best distinct displayed values (F1 highest; MU, GF and $\mathrm{GF}_{\mathrm{f}}$ lowest, GF compared on the displayed total); ties share markings.}
\label{tab:perfold-czu_mhad}
\vspace{-6pt}
\resizebox{\textwidth}{!}{%
\begin{tabular}{@{}ll*{3}{ccc}@{}}
\toprule
\multirow{2}{*}{\textbf{Type}} & \multirow{2}{*}{\textbf{Method}} & \multicolumn{3}{c}{\textbf{Fold 0}} & \multicolumn{3}{c}{\textbf{Fold 1}} & \multicolumn{3}{c}{\textbf{Fold 2}} \\
\cmidrule(lr){3-5} \cmidrule(lr){6-8} \cmidrule(lr){9-11}
& & F1 $\uparrow$ & MU $\downarrow$ & $\mathrm{GF}\,(\mathrm{GF}_{\mathrm{f}})\downarrow$ & F1 $\uparrow$ & MU $\downarrow$ & $\mathrm{GF}\,(\mathrm{GF}_{\mathrm{f}})\downarrow$ & F1 $\uparrow$ & MU $\downarrow$ & $\mathrm{GF}\,(\mathrm{GF}_{\mathrm{f}})\downarrow$ \\
\midrule
\multirow{4}{*}{\textbf{Monolithic}} & MAESTRO & .182 & 1.0 & \textcolor{burgundy}{\textbf{6091}} (.30) & .436 & 1.0 & \textcolor{burgundy}{\textbf{6091}} (.30) & .135 & 1.0 & \textcolor{burgundy}{\textbf{6091}} (.30) \\
 & ShaSpec & .645 & 1.0 & 6093 (2.17) & .587 & 1.0 & 6093 (2.17) & .524 & 1.0 & 6093 (2.17) \\
 & MBT & \underline{.910} & 1.0 & \textcolor{burgundy}{\textbf{6091}} (.56) & .849 & 1.0 & \textcolor{burgundy}{\textbf{6091}} (.56) & .703 & 1.0 & \textcolor{burgundy}{\textbf{6091}} (.56) \\
 & DecAlign & .352 & 1.0 & 6114 (23.37) & .469 & 1.0 & 6114 (23.37) & .265 & 1.0 & 6114 (23.37) \\
\arrayrulecolor{black!20}\midrule
\textbf{Fixed-seq.} & MultiModN & .563 & 1.0 & \underline{6092} (1.08) & .459 & 1.0 & \underline{6092} (1.08) & .379 & 1.0 & \underline{6092} (1.08) \\
\arrayrulecolor{black!20}\midrule
\multirow{5}{*}{\textbf{Selective}} & GMS & .497 & \textcolor{burgundy}{\textbf{.14}} & \textcolor{burgundy}{\textbf{6091}} (\textcolor{burgundy}{\textbf{.11}}) & .504 & .71 & \textcolor{burgundy}{\textbf{6091}} (.18) & .331 & \textcolor{burgundy}{\textbf{.14}} & \textcolor{burgundy}{\textbf{6091}} (\underline{.11}) \\
 & AdaMML & .842 & \textcolor{burgundy}{\textbf{.14}} & \textcolor{burgundy}{\textbf{6091}} (.17) & \underline{.914} & \textcolor{burgundy}{\textbf{.14}} & \textcolor{burgundy}{\textbf{6091}} (\underline{.17}) & \underline{.733} & \textcolor{burgundy}{\textbf{.14}} & \textcolor{burgundy}{\textbf{6091}} (.17) \\
 & DynMM & .809 & .56 & \underline{6092} (1.85) & .793 & .37 & \underline{6092} (1.26) & .660 & .46 & \underline{6092} (1.52) \\
 & DyMo & .675 & .39 & \textcolor{burgundy}{\textbf{6091}} (.69) & .758 & .31 & \textcolor{burgundy}{\textbf{6091}} (.67) & .527 & .37 & \textcolor{burgundy}{\textbf{6091}} (.69) \\
 & MMEE & .664 & 1.0 & \textcolor{burgundy}{\textbf{6091}} (.73) & .560 & 1.0 & \textcolor{burgundy}{\textbf{6091}} (.73) & .478 & 1.0 & \textcolor{burgundy}{\textbf{6091}} (.73) \\
\arrayrulecolor{black!20}\midrule
\textbf{Proactive} & Joint-AFA & .333 & \underline{.17} & \textcolor{burgundy}{\textbf{6091}} (\underline{.13}) & .248 & \underline{.20} & \textcolor{burgundy}{\textbf{6091}} (\textcolor{burgundy}{\textbf{.13}}) & .254 & \textcolor{burgundy}{\textbf{.14}} & \textcolor{burgundy}{\textbf{6091}} (.12) \\
\arrayrulecolor{black}\midrule
\rowcolor{tolblue!10}\multicolumn{2}{l}{\textbf{\method{} (Ours)}} & \textcolor{burgundy}{\textbf{.960}} & .62 & \textcolor{burgundy}{\textbf{6091}} (.37) & \textcolor{burgundy}{\textbf{.982}} & .35 & \textcolor{burgundy}{\textbf{6091}} (.22) & \textcolor{burgundy}{\textbf{.800}} & \underline{.15} & \textcolor{burgundy}{\textbf{6091}} (\textcolor{burgundy}{\textbf{.10}}) \\
\bottomrule
\end{tabular}%
}
\end{table*}

\begin{table*}[!htbp]
\centering\footnotesize\setlength{\tabcolsep}{2.5pt}\renewcommand{\arraystretch}{1.12}
\caption{\small \textbf{EAV}, per fold, full modality availability: macro-F1 ($\uparrow$), modality usage MU ($\downarrow$) and end-to-end GFLOPs ($\downarrow$; fusion cost $\mathrm{GF}_{\mathrm{f}}$ in parentheses). Fold means and sample standard deviations of these values are the cells of Table~\ref{tab:full-results}; F1 is shown to three decimals. Per column, \textcolor{burgundy}{\textbf{bold}} and \underline{underlining} mark the best and second-best distinct displayed values (F1 highest; MU, GF and $\mathrm{GF}_{\mathrm{f}}$ lowest, GF compared on the displayed total); ties share markings.}
\label{tab:perfold-eav}
\vspace{-6pt}
\resizebox{\textwidth}{!}{%
\begin{tabular}{@{}ll*{3}{ccc}@{}}
\toprule
\multirow{2}{*}{\textbf{Type}} & \multirow{2}{*}{\textbf{Method}} & \multicolumn{3}{c}{\textbf{Fold 0}} & \multicolumn{3}{c}{\textbf{Fold 1}} & \multicolumn{3}{c}{\textbf{Fold 2}} \\
\cmidrule(lr){3-5} \cmidrule(lr){6-8} \cmidrule(lr){9-11}
& & F1 $\uparrow$ & MU $\downarrow$ & $\mathrm{GF}\,(\mathrm{GF}_{\mathrm{f}})\downarrow$ & F1 $\uparrow$ & MU $\downarrow$ & $\mathrm{GF}\,(\mathrm{GF}_{\mathrm{f}})\downarrow$ & F1 $\uparrow$ & MU $\downarrow$ & $\mathrm{GF}\,(\mathrm{GF}_{\mathrm{f}})\downarrow$ \\
\midrule
\multirow{4}{*}{\textbf{Monolithic}} & MAESTRO & .516 & 1.0 & 6230 (\underline{.15}) & .549 & 1.0 & 6230 (.15) & .537 & 1.0 & 6230 (.15) \\
 & ShaSpec & .638 & 1.0 & 6231 (.97) & .550 & 1.0 & 6231 (.97) & .615 & 1.0 & 6231 (.97) \\
 & MBT & .639 & 1.0 & 6230 (.26) & .588 & 1.0 & 6230 (.26) & .604 & 1.0 & 6230 (.26) \\
 & DecAlign & .634 & 1.0 & 6232 (1.77) & .613 & 1.0 & 6232 (1.77) & .599 & 1.0 & 6232 (1.77) \\
\arrayrulecolor{black!20}\midrule
\textbf{Fixed-seq.} & MultiModN & \underline{.692} & 1.0 & 6230 (.49) & .610 & 1.0 & 6230 (.49) & .642 & 1.0 & 6230 (.49) \\
\arrayrulecolor{black!20}\midrule
\multirow{5}{*}{\textbf{Selective}} & GMS & .654 & \textcolor{burgundy}{\textbf{.33}} & \textcolor{burgundy}{\textbf{139}} (\textcolor{burgundy}{\textbf{.09}}) & .619 & \textcolor{burgundy}{\textbf{.33}} & \textcolor{burgundy}{\textbf{139}} (\textcolor{burgundy}{\textbf{.09}}) & \textcolor{burgundy}{\textbf{.702}} & \textcolor{burgundy}{\textbf{.33}} & \textcolor{burgundy}{\textbf{139}} (\textcolor{burgundy}{\textbf{.09}}) \\
 & AdaMML & .690 & .67 & 6230 (.17) & .561 & .67 & 6230 (.17) & .680 & .67 & 6230 (.17) \\
 & DynMM & .654 & \textcolor{burgundy}{\textbf{.33}} & 6230 (.52) & \underline{.662} & \textcolor{burgundy}{\textbf{.33}} & 6230 (.52) & .678 & \underline{.37} & 6230 (.57) \\
 & DyMo & .645 & \underline{.52} & 6230 (.28) & .581 & .52 & 6230 (.28) & .567 & .53 & 6230 (.28) \\
 & MMEE & .618 & 1.0 & 6230 (.18) & .521 & 1.0 & 6230 (\underline{.11}) & .647 & 1.0 & 6230 (.26) \\
\arrayrulecolor{black!20}\midrule
\textbf{Proactive} & Joint-AFA & .585 & .73 & \underline{4394} (.18) & .526 & .52 & 2767 (.15) & .523 & .52 & 2802 (.15) \\
\arrayrulecolor{black}\midrule
\rowcolor{tolblue!10}\multicolumn{2}{l}{\textbf{\method{} (Ours)}} & \textcolor{burgundy}{\textbf{.699}} & \textcolor{burgundy}{\textbf{.33}} & \textcolor{burgundy}{\textbf{139}} (\textcolor{burgundy}{\textbf{.09}}) & \textcolor{burgundy}{\textbf{.668}} & \underline{.39} & \underline{870} (\underline{.11}) & \underline{.701} & .39 & \underline{1149} (\underline{.11}) \\
\bottomrule
\end{tabular}%
}
\end{table*}

\begin{table*}[!htbp]
\centering\footnotesize\setlength{\tabcolsep}{2.5pt}\renewcommand{\arraystretch}{1.12}
\caption{\small \textbf{UTD-MHAD}, per fold, full modality availability: macro-F1 ($\uparrow$), modality usage MU ($\downarrow$) and end-to-end GFLOPs ($\downarrow$; fusion cost $\mathrm{GF}_{\mathrm{f}}$ in parentheses). Fold means and sample standard deviations of these values are the cells of Table~\ref{tab:full-results}; F1 is shown to three decimals. Per column, \textcolor{burgundy}{\textbf{bold}} and \underline{underlining} mark the best and second-best distinct displayed values (F1 highest; MU, GF and $\mathrm{GF}_{\mathrm{f}}$ lowest, GF compared on the displayed total); ties share markings.}
\label{tab:perfold-utd_mhad}
\vspace{-6pt}
\resizebox{\textwidth}{!}{%
\begin{tabular}{@{}ll*{3}{ccc}@{}}
\toprule
\multirow{2}{*}{\textbf{Type}} & \multirow{2}{*}{\textbf{Method}} & \multicolumn{3}{c}{\textbf{Fold 0}} & \multicolumn{3}{c}{\textbf{Fold 1}} & \multicolumn{3}{c}{\textbf{Fold 2}} \\
\cmidrule(lr){3-5} \cmidrule(lr){6-8} \cmidrule(lr){9-11}
& & F1 $\uparrow$ & MU $\downarrow$ & $\mathrm{GF}\,(\mathrm{GF}_{\mathrm{f}})\downarrow$ & F1 $\uparrow$ & MU $\downarrow$ & $\mathrm{GF}\,(\mathrm{GF}_{\mathrm{f}})\downarrow$ & F1 $\uparrow$ & MU $\downarrow$ & $\mathrm{GF}\,(\mathrm{GF}_{\mathrm{f}})\downarrow$ \\
\midrule
\multirow{4}{*}{\textbf{Monolithic}} & MAESTRO & .683 & 1.0 & 7306 (\underline{.19}) & .721 & 1.0 & 7306 (.19) & .689 & 1.0 & 7306 (.19) \\
 & ShaSpec & .809 & 1.0 & 7307 (1.29) & \textcolor{burgundy}{\textbf{.765}} & 1.0 & 7307 (1.29) & .566 & 1.0 & 7307 (1.29) \\
 & MBT & \underline{.873} & 1.0 & 7306 (.35) & \underline{.759} & 1.0 & 7306 (.35) & \underline{.805} & 1.0 & 7306 (.35) \\
 & DecAlign & .838 & 1.0 & 7310 (4.21) & .709 & 1.0 & 7310 (4.21) & .692 & 1.0 & 7310 (4.21) \\
\arrayrulecolor{black!20}\midrule
\textbf{Fixed-seq.} & MultiModN & .764 & 1.0 & 7306 (.64) & .578 & 1.0 & 7306 (.64) & .636 & 1.0 & 7306 (.64) \\
\arrayrulecolor{black!20}\midrule
\multirow{5}{*}{\textbf{Selective}} & GMS & .675 & 1.0 & 7306 (.26) & .667 & .50 & 4124 (\underline{.13}) & .502 & .50 & \underline{4124} (\underline{.13}) \\
 & AdaMML & .792 & .50 & 7306 (\underline{.19}) & .683 & .50 & 7306 (.19) & .788 & .50 & 7306 (.19) \\
 & DynMM & .828 & 1.0 & 7308 (1.91) & .721 & 1.0 & 7308 (1.91) & .777 & 1.0 & 7308 (1.91) \\
 & DyMo & .634 & \underline{.43} & 7306 (.38) & .638 & \underline{.42} & 7306 (.38) & .538 & \underline{.43} & 7306 (.38) \\
 & MMEE & .631 & 1.0 & 7306 (.44) & .692 & 1.0 & 7306 (.44) & .622 & 1.0 & 7306 (.44) \\
\arrayrulecolor{black!20}\midrule
\textbf{Proactive} & Joint-AFA & .325 & \textcolor{burgundy}{\textbf{.26}} & \textcolor{burgundy}{\textbf{0.02}} (\textcolor{burgundy}{\textbf{.02}}) & .331 & \textcolor{burgundy}{\textbf{.41}} & \textcolor{burgundy}{\textbf{2246}} (\textcolor{burgundy}{\textbf{.09}}) & .256 & \textcolor{burgundy}{\textbf{.25}} & \textcolor{burgundy}{\textbf{19.20}} (\textcolor{burgundy}{\textbf{.02}}) \\
\arrayrulecolor{black}\midrule
\rowcolor{tolblue!10}\multicolumn{2}{l}{\textbf{\method{} (Ours)}} & \textcolor{burgundy}{\textbf{.931}} & .56 & \underline{5204} (.21) & .724 & \textcolor{burgundy}{\textbf{.41}} & \underline{3488} (.15) & \textcolor{burgundy}{\textbf{.826}} & .58 & 5426 (.21) \\
\bottomrule
\end{tabular}%
}
\end{table*}

\subsection{Class-conditional acquisition on the remaining five datasets}\label{app:class_cond_all}

Figure~\ref{fig:diag-class} in the main text shows that \method{}'s acquisitions separate by true class on \textbf{IEMOCAP}, in a form---one bar panel per class---that only stays readable for a small label space. Figure~\ref{fig:class_heat} gives the same quantity, $P(j\in S\mid y)$, for the other five datasets: a class$\times$modality image with a side strip for the per-class mean acquired-set size. They span $C{=}5$--$27$, where a panel per class is no longer legible.

The effect is not an \textbf{IEMOCAP} artifact. Mean acquired-set size varies by class on every dataset---$3.44$--$5.76$ of $M{=}7$ on \textbf{CMI}, $1.63$--$3.67$ of $7$ on \textbf{CZU-MHAD} and $1.92$--$3.50$ of $4$ on \textbf{UTD-MHAD}---so on every stack some classes are settled cheaply and others are not. The universal-probe pattern recurs as the dark column each panel carries: one stream the policy almost never skips, IMU on \textbf{CMI} and inertial on \textbf{UTD-MHAD} as text is on \textbf{IEMOCAP}, with the remaining streams bought class-selectively around it. Two datasets sit at the extremes. \textbf{MM-Fi} is the narrowest at $1.00$--$1.50$ of $M{=}5$: several classes are settled by a single modality, consistent with its policy acquiring only $0.218$ of the available streams in Table~\ref{tab:comp-replace-complete}. \textbf{EAV} is flat for the opposite reason---with $M{=}3$ there is little to choose between, and the acquired set stays between $1.12$ and $1.25$---which is worth reading as a limit of the method rather than a property of the data: adaptive acquisition needs something to trade off, and three streams barely provide it.

\begin{figure}[t]
\centering
\begin{subfigure}{0.32\linewidth}\centering

  \includegraphics[width=\linewidth,height=3.4cm,keepaspectratio]{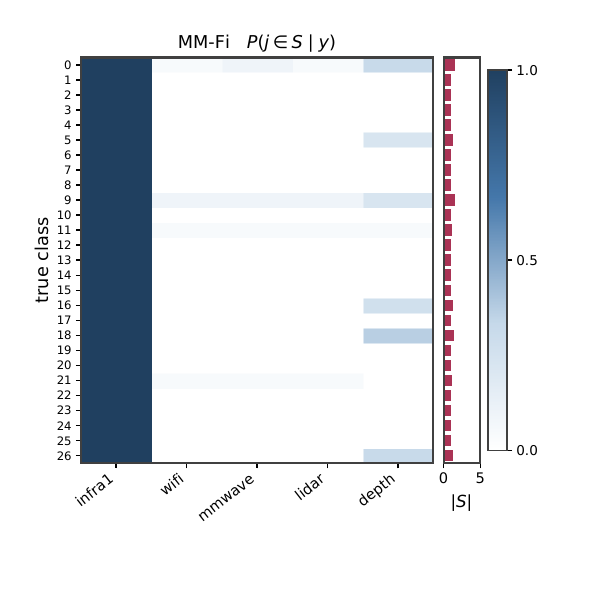}\caption{MM-Fi}\end{subfigure}
\hfill
\begin{subfigure}{0.32\linewidth}\centering
  \includegraphics[width=\linewidth,height=3.4cm,keepaspectratio]{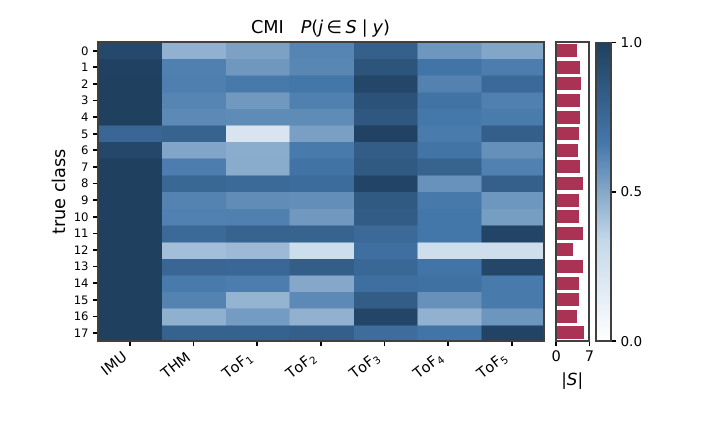}\caption{CMI}\end{subfigure}
\hfill
\begin{subfigure}{0.32\linewidth}\centering
  \includegraphics[width=\linewidth,height=3.4cm,keepaspectratio]{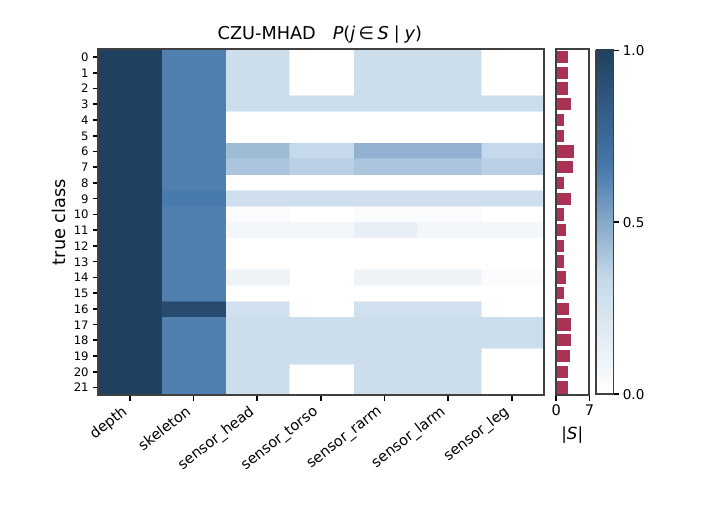}\caption{CZU-MHAD}\end{subfigure}

\vspace{5pt}
\begin{subfigure}{0.32\linewidth}\centering
  \includegraphics[width=\linewidth,height=3.4cm,keepaspectratio]{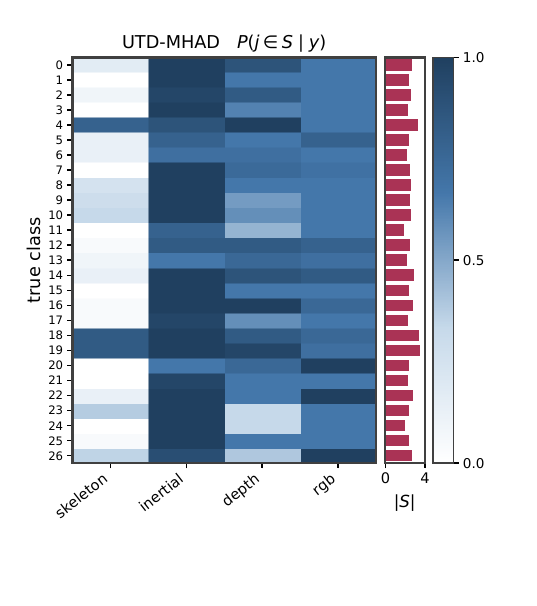}\caption{UTD-MHAD}\end{subfigure}
\hfill
\begin{subfigure}{0.32\linewidth}\centering
  \includegraphics[width=\linewidth,height=3.4cm,keepaspectratio]{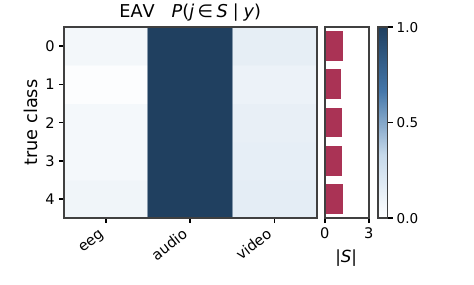}\caption{EAV}\end{subfigure}
\hfill
\begin{subfigure}{0.32\linewidth}\centering\end{subfigure}
\caption{\textbf{Class-conditional acquisition on the five datasets outside the main text.} Each cell is $P(j\in S\mid y)$, the probability that the deployed policy acquires modality $j$ given true class $y$, at full availability with the test folds pooled; darker is more often acquired. The red strip right of each panel is that class's mean acquired-set size $|S|$ on a $0$--$M$ scale. Panels preserve their aspect ratios and fit within their columns; modality counts are $M{=}3$--$7$ and class counts are $C{=}5$--$27$. \textbf{IEMOCAP}, drawn the same way, is Figure~\ref{fig:diag-class}.}
\label{fig:class_heat}
\end{figure}

\section{Complete Energy--Latency Study Results}
\label{app:energy-latency-complete}

Table~\ref{tab:energy-latency-complete} reports the hardware results supporting Figure~\ref{fig:energy_latency}. Across six datasets, \method{} has the lowest reported latency in 10 of 12 GPU/CPU conditions and the lowest reported energy in 8 of 12, with mean reductions of 44.0\% and 42.6\%, respectively, relative to the best of the three adaptive baselines (AdaMML, DynMM and DyMo) for each metric. The exceptions are CMI, whose raw sensor streams are nearly free to extract, and CZU-MHAD, where \method{} acquires the dominant depth stream for every sample and its GPU latency is within 0.3\% of the fastest baseline. Android INT8 results cover all six datasets and are excluded from these GPU/CPU aggregate percentages. On Android INT8, \method{} has the lowest reported latency and energy on all six datasets, with mean reductions of 47.2\% and 49.7\%, respectively, relative to the best of the same three adaptive baselines for each metric. These are equal-weight means of the six dataset-specific percentage reductions; the CMI energy comparison should be interpreted cautiously because of the small measured energies.

\paragraph{Measurement protocol for Table~\ref{tab:energy-latency-complete}.}
GPU measurements use one NVIDIA L40 (48\,GB) in a dual-socket Intel Xeon Platinum 8462Y+ server ($2\times32$ cores), with the machine otherwise idle and the GPU dedicated to the benchmark. CPU measurements use the same server without GPU execution and with PyTorch's default intra-op thread pool. Each method runs in eager PyTorch at batch size~1. For each sample, its required raw files are read and decoded, the corresponding frozen extractors and downstream encoders execute according to the method's selection and missingness rules, and the fusion network and any acquisition loop produce the logits. All encoders are loaded and one-time archive conversions are completed before timing. The first three samples of each fold are used for untimed warm-up on both GPU and CPU.

Latency is wall-clock time from immediately before the first file read until the logits are returned. GPU timing synchronizes CUDA at both boundaries; CPU timing uses the same wall-clock boundaries without CUDA synchronization. GPU energy is the difference between readings of the board's cumulative energy counter, \texttt{nvmlDeviceGetTotalEnergyConsumption}, at these boundaries. It requires no power sampling and includes the GPU's idle draw during the timed interval.

CPU energy is measured on the complete test set of each fold using the difference of four package-plus-DRAM RAPL counters, including idle consumption over the interval. Each baseline energy cell averages these samples within each fold and then averages the folds.

\paragraph{\method{} under the deployed rule.}
\method{}'s GPU and CPU latency is measured end to end under the deployed rule (Section~\ref{sec:summary}) on the complete test folds, using the protocol above. Android deployment uses the same policy rule. For fairness we use the same mobile device for all measurements. Running takes very long due to the low processing speed. Thus, measurements on the phone are for the first 50 test samples of each fold. On every platform, the prior uses the same per-fold $\lambda_{\mathrm{train}}$ as policy fitting. 

Server prediction metrics are evaluated on the complete test set. Latency and energy are measured end to end for each baseline. For example, the IEMOCAP folds contain 909, 793 and 1000 test samples. Each reported mean averages samples within a fold and then weights the three folds equally. Fold-to-fold latency standard deviations are below 10\% of the mean except on CZU-MHAD and UTD-MHAD, whose folds differ in clip length; per-fold measurements are retained with the raw result files. Component-replacement and cost-sensitivity latency measurements use the same L40 server and time each configuration end to end (Appendices~\ref{app:comp-replace-complete} and~\ref{app:reward-sensitivity}).

\subsection{On-Device Deployment and Quantization}
\label{app:deployment}

\paragraph{Hardware and runtime.}
On-device measurements use a OnePlus phone with a Qualcomm Snapdragon~8~Gen~3 (SM8650: one Cortex-X4, five Cortex-A720 and two Cortex-A520 cores, Adreno~750 GPU, and Hexagon NPU), 16\,GB RAM, and Android~16. One Kotlin application performs the complete pipeline at batch size~1: raw-file decoding, feature extraction, modality acquisition, fusion, and prediction. Inference uses ONNX Runtime for Android 1.19 (\texttt{onnxruntime-android} from Maven), with the CPU execution provider and four intra-op threads. NNAPI, GPU and NPU execution providers are disabled. Encoder sessions are created once and remain resident, excluding session creation from timed samples.

Latency is measured with \texttt{System.nanoTime}. Phone energy is calculated from the battery fuel gauge by integrating instantaneous current times voltage at 10\,Hz and subtracting an idle baseline calibrated for 3\,s before each run, with the phone unplugged and the screen on. For datasets whose samples are shorter than the gauge update interval, all measured samples are timed as one block and total latency and energy are divided by the sample count. To keep measurement time manageable, we perform measurements on a subset of the test samples. Every method uses the same fixed first 50 test samples of each fold, with per-sample records retained in the app's result files. On-device macro-F1 is computed on these samples; server GPU/CPU macro-F1 uses the complete test folds. The quantization pilot below reports accuracy and is distinct from macro-F1. Thus, server and phone prediction metrics describe different evaluation populations.

\paragraph{Quantization.}
All deployed methods use post-training quantization without fine-tuning. Deployment uses ONNX Runtime dynamic W8A8 quantization, with per-channel INT8 weights and runtime per-tensor INT8 activations for compute-dominant matrix multiplications, including linear and attention projections. Normalization, softmax, activation functions, and small \texttt{Gemm}/\texttt{Conv} layers remain fp32. SmoothQuant~\citep{xiao2023smoothquant} redistributes activation outliers into weights through an equivalent fp32 transformation.

\paragraph{Calibration data and label use.}
Calibration inputs come from the validation partitions. The initial IEMOCAP calibration for XLS-R, BERT, DINOv3, and the fusion graph uses eight samples from the val manifest, sampled with seed~0 and stratified by label (two samples per emotion). Fusion graphs for the other five datasets use the last eight samples in each val-fold manifest. The shared DINOv3 image encoder is subsequently recalibrated on a mixture of IEMOCAP video frames and frames from the first three staged val samples per source on UTD-MHAD (depth and RGB), CZU-MHAD (depth), and MM-Fi (depth), using 12 frames per source.

SmoothQuant calibration itself uses channel-wise maxima of absolute input activations for each matrix multiplication and weight ranges, without a supervised loss or parameter fine-tuning. Dynamic W8A8 computes activation scales per tensor at runtime. 

Table~\ref{tab:quant-pilot} motivates this choice. Whole-stack accuracy on IEMOCAP varies by less than 0.5 percentage points across precisions, but deployment efficiency and feature fidelity differ substantially. fp16 causes ViT-L attention overflow and largely falls back to fp32 on the CPU; weight-only INT8 preserves fidelity but provides little speedup. W8A8 enables fast INT8 kernels, yet severely distorts DINOv3 and BERT features. SmoothQuant improves their cosine similarity to fp32 from 0.618 to 0.990 and from 0.938 to 0.979, respectively, while retaining accelerated inference. Fusion graphs operating on $\ell_2$-normalized features already preserve outputs under either variant (mean cosine 0.9999; argmax agreement 1.000). Across the adaptive methods, SmoothQuant increases end-to-end prediction agreement with fp32 from 81.2\% to 93.8\%, with exact agreement for \method{} and DyMo (24 fold-0 IEMOCAP samples run end to end through the ONNX graphs, \method{} under the deployed rule). We therefore deploy SmoothQuant W8A8 without fine-tuning for every method.

\begin{table}[htbp]
\centering\small
\setlength{\tabcolsep}{4pt}
\caption{Quantization pilot on IEMOCAP. \textbf{Top:} cross-subject test accuracy (three folds) for fp32, fp16, and two INT8 settings. QAT is included only as an experimental comparison; the deployed PTQ setting uses no fine-tuning. fp16 uses bf16 for the ViT-L encoder. \textbf{Bottom:} the shared encoders on Android phone with ONNX Runtime: size, output-feature cosine against fp32 on real inputs (min.\ per-token cosine in parentheses) and latency relative to fp32. W8A8 = plain dynamic INT8; SQ70 = SmoothQuant($\alpha{=}0.7$) then the same quantizer, the deployed format.}
\label{tab:quant-pilot}
\begin{tabular}{@{}lcccc@{}}
\toprule
\textbf{Method} & \textbf{fp32} & \textbf{fp16} & \textbf{INT8 (QAT)} & \textbf{INT8 (PTQ)} \\
\midrule
\backbone{}+\policy{} (Ours) & .696 & .697 & .696 & .696 \\
AdaMML & .623 & .624 & .624 & .619 \\
DynMM & .663 & .661 & .663 & .663 \\
DyMo & .640 & .639 & .641 & .642 \\
\midrule\midrule
\textbf{Encoder / format} & \textbf{MB} & \textbf{cos (min tok.)} & \multicolumn{2}{c}{\textbf{latency vs fp32}} \\
\midrule
DINOv3-L/16 \ fp32 & 1213 & 1.000 (1.000) & \multicolumn{2}{c}{1.00$\times$} \\
\phantom{DINOv3-L/16} \ weight-only INT8 & 308 & 1.000 (--) & \multicolumn{2}{c}{$\approx$1.0$\times$} \\
\phantom{DINOv3-L/16} \ W8A8 & 308 & 0.618 (0.528) & \multicolumn{2}{c}{0.24$\times$} \\
\phantom{DINOv3-L/16} \ \textbf{SQ70} & 309 & \textbf{0.990 (0.973)} & \multicolumn{2}{c}{\textbf{0.21$\times$}} \\
XLSR-1B \ fp32 & 3852 & 1.000 (1.000) & \multicolumn{2}{c}{1.00$\times$} \\
\phantom{XLSR-1B} \ W8A8 & 1023 & 0.995 (0.252) & \multicolumn{2}{c}{0.30$\times$} \\
\phantom{XLSR-1B} \ \textbf{SQ70} & 1025 & \textbf{0.996 (0.384)} & \multicolumn{2}{c}{\textbf{0.25$\times$}} \\
BERT-base \ fp32 & 436 & 1.000 (1.000) & \multicolumn{2}{c}{1.00$\times$} \\
\phantom{BERT-base} \ W8A8 & 182 & 0.938 (0.117) & \multicolumn{2}{c}{0.44$\times$} \\
\phantom{BERT-base} \ \textbf{SQ70} & 182 & \textbf{0.979 (0.460)} & \multicolumn{2}{c}{\textbf{0.40$\times$}} \\
\bottomrule
\end{tabular}
\end{table}

\begin{table*}[htbp]
\centering
\scriptsize
\setlength{\tabcolsep}{2.5pt}
\renewcommand{\arraystretch}{1.05}
\caption{Latency and energy across compute platforms.
Each entry reports latency (ms)/energy (J); lower is better.
\textcolor{burgundy}{\textbf{Bold}} and \underline{underlining}
denote the best and second-best values within each dataset and
platform, ranked separately for latency and energy at full precision.
Latency is measured end to end for every method and platform. \method{} uses the deployed policy. Server prediction metrics on GPU and CPU use the complete test folds; phone metrics use the same 50 samples per fold for every method. GPU and CPU energy include idle consumption over the timed interval, whereas phone energy subtracts a calibrated idle baseline.}
\label{tab:energy-latency-complete}
\resizebox{\textwidth}{!}{%
\begin{tabular}{@{}l*{5}{c!{\color{black!25}\vrule width 0.35pt}}c@{}}
\toprule
\textbf{Method}
& \textbf{IEMOCAP}
& \textbf{MM-Fi}
& \textbf{CMI}
& \textbf{CZU-MHAD}
& \textbf{EAV}
& \textbf{UTD-MHAD} \\
\midrule
\multicolumn{7}{@{}l}{\emph{GPU}} \\
MBT
& \underline{1530.22}/\underline{367.68}
& 4234.19/804.09
& 15.12/2.15
& 1324.02/260.59
& 2032.22/321.31
& 1036.01/279.43 \\
AdaMML
& 1595.15/385.60
& 4196.11/\underline{789.15}
& \underline{14.61}/\textcolor{burgundy}{\textbf{1.57}}
& \textcolor{burgundy}{\textbf{1312.38}}/262.83
& \underline{2011.12}/\underline{316.99}
& \underline{1018.04}/\underline{275.27} \\
DynMM
& 1599.79/384.79
& \underline{4194.79}/796.49
& 17.98/1.64
& 1317.80/\textcolor{burgundy}{\textbf{259.45}}
& 2017.36/319.48
& 1020.67/275.42 \\
DyMo
& 1616.70/388.86
& 4208.37/799.10
& 35.50/3.75
& 1341.98/264.15
& 2028.04/322.96
& 1031.99/277.82 \\
\rowcolor{tolblue!10}
\textbf{\method{} (Ours)}
& \textcolor{burgundy}{\textbf{732.23}}/\textcolor{burgundy}{\textbf{180.65}}
& \textcolor{burgundy}{\textbf{179.46}}/\textcolor{burgundy}{\textbf{32.84}}
& \textcolor{burgundy}{\textbf{14.39}}/\underline{1.60}
& \underline{1315.73}/\underline{260.54}
& \textcolor{burgundy}{\textbf{224.42}}/\textcolor{burgundy}{\textbf{36.74}}
& \textcolor{burgundy}{\textbf{627.78}}/\textcolor{burgundy}{\textbf{169.99}} \\
\arrayrulecolor{black!20}\midrule
\arrayrulecolor{black}
\multicolumn{7}{@{}l}{\emph{CPU}} \\
MBT
& 14618.00/3844.30
& 19551.00/4509.80
& 35.01/3.80
& 6467.90/1360.10
& 8054.00/1246.80
& 8055.00/2308.20 \\
AdaMML
& 14687.00/\underline{3693.60}
& 19551.00/\underline{4429.50}
& \textcolor{burgundy}{\textbf{23.00}}/\textcolor{burgundy}{\textbf{2.10}}
& \underline{6437.00}/\textcolor{burgundy}{\textbf{1265.90}}
& 8076.00/1246.80
& 8062.00/\underline{2273.80} \\
DynMM
& \underline{13546.00}/3844.30
& \underline{19477.00}/4646.50
& 41.00/4.40
& 6489.00/1420.30
& \underline{7954.00}/1319.00
& \underline{8031.00}/2390.10 \\
DyMo
& 14618.00/3856.80
& 19567.00/4509.80
& 48.00/4.60
& 6498.00/1364.50
& 8054.00/\underline{1244.30}
& 8055.00/2308.20 \\
\rowcolor{tolblue!10}
\textbf{\method{} (Ours)}
& \textcolor{burgundy}{\textbf{7057.57}}/\textcolor{burgundy}{\textbf{1862.06}}
& \textcolor{burgundy}{\textbf{810.81}}/\textcolor{burgundy}{\textbf{186.88}}
& \underline{30.64}/\underline{2.94}
& \textcolor{burgundy}{\textbf{6284.01}}/\underline{1319.57}
& \textcolor{burgundy}{\textbf{628.70}}/\textcolor{burgundy}{\textbf{97.13}}
& \textcolor{burgundy}{\textbf{4429.75}}/\textcolor{burgundy}{\textbf{1269.37}} \\
\arrayrulecolor{black!20}\midrule
\arrayrulecolor{black}
\multicolumn{7}{@{}l}{\emph{Android ONNX Runtime (INT8)}} \\
MBT
& 38229.44/169.40
& 100256.70/351.37
& \underline{45.00}/\underline{0.11}
& \underline{33891.60}/174.72
& 37646.00/163.14
& \underline{42716.00}/\underline{165.44} \\
AdaMML
& 38360.50/\underline{151.76}
& 99529.80/354.73
& 47.00/0.18
& 34356.00/\underline{155.27}
& \underline{37445.30}/167.97
& 47223.00/174.50 \\
DynMM
& 39354.00/169.99
& \underline{96346.60}/381.87
& 59.00/0.21
& 34482.30/155.95
& 37498.70/163.76
& 45987.00/168.25 \\
DyMo
& \underline{37946.00}/180.07
& 99350.50/\underline{350.61}
& 46.00/0.13
& 34530.80/155.87
& 37460.70/\underline{159.65}
& 43746.00/188.68 \\
\rowcolor{tolblue!10}
\textbf{\method{} (Ours)}
& \textcolor{burgundy}{\textbf{19383.10}}/\textcolor{burgundy}{\textbf{88.21}}
& \textcolor{burgundy}{\textbf{2509.60}}/\textcolor{burgundy}{\textbf{12.23}}
& \textcolor{burgundy}{\textbf{39.80}}/\textcolor{burgundy}{\textbf{0.09}}
& \textcolor{burgundy}{\textbf{33058.10}}/\textcolor{burgundy}{\textbf{147.85}}
& \textcolor{burgundy}{\textbf{5909.90}}/\textcolor{burgundy}{\textbf{24.79}}
& \textcolor{burgundy}{\textbf{28305.00}}/\textcolor{burgundy}{\textbf{96.84}} \\
\bottomrule
\end{tabular}%
}
\end{table*}


\section{Complete Missingness Study Results}
\label{app:missingness-complete}

Tables~\ref{tab:missingness-iemocap}--\ref{tab:missingness-utdmhad} report macro-F1 and total GFLOPs, with post-extraction cost $\mathrm{GF}_{\mathrm f}$ in parentheses. The primary-method entries at 0\% missingness use the same recomputed three-fold means and sample standard deviations as Table~\ref{tab:full-results}. The 10--40\% measurements are retained at their reported precision; missing standard deviations are not inferred. Among the 30 dataset--missingness conditions, \method{} attains the highest displayed F1 in 26 (including ties) and is Pareto non-dominated in 29 using total GFLOPs, or 29 using post-extraction GFLOPs. The backbone-only diagnostic is excluded from these rankings. These are displayed-precision counts, not claims about an unavailable higher-precision ranking.

At 40\% missingness, the displayed values imply the following relative F1 changes and total-GFLOP reductions against the highest-F1 baseline (ties broken by lower total GFLOPs): IEMOCAP (0.0\%, 63.9\%), MM-Fi (6.9\%, 52.4\%), CMI (0.0\%, 60.0\%), CZU-MHAD (7.0\%, -2.6\%), EAV (3.7\%, 55.2\%), UTD-MHAD (6.3\%, 19.8\%). A negative reduction denotes a cost increase. Post-extraction costs are separate and must not be substituted for total GFLOPs in these comparisons.

\noindent \textbf{Experimental Setup for Missingness Study.} 
We evaluate every method on the same six datasets and the same cross-subject test folds as Table~\ref{tab:full-results}, under test-time modality missingness at rates $\rho \in \{0.1, 0.2, 0.3, 0.4\}$; $\rho = 0$ is the full-availability setting of the main table. Missingness is simulated per test sample: each of the $M$ modalities is dropped independently with probability $\rho$, and if every modality of a sample is dropped one is reinstated uniformly at random.
The masks are drawn once per fold with a fixed seed and shared by all methods, so every method sees exactly the same set of absent modalities on every sample. All models use the checkpoints of the full-availability experiment, and \method{} keeps its deployed acquisition policy unchanged across all levels. The policy is trained with modality-availability augmentation 

and absent modalities are in-distribution for it at test time. Every method skips raw-data extraction for absent modalities. MBT, MAESTRO, ShaSpec, DecAlign, MultiModN and MMEE still execute their downstream modality encoders with zero-filled inputs for absent positions; MultiModN separately skips the corresponding state updates. AdaMML and DyMo can skip encoders, whereas DynMM cannot. GMS, Joint-AFA and \method{} request only available modalities through their selection or acquisition masks. MU is normalized per sample by the number of available modalities, and is therefore $1$ for methods that retain all available inputs. End-to-end GFLOPs include the raw extraction actually performed and the downstream encoders, fusion, and controller computation executed at batch size~1.

Macro-F1 is computed on the full test set of each fold and averaged over folds; standard deviations are given where all folds were measured under the three-trial protocol.

\begin{table*}[t]
\centering
\scriptsize
\setlength{\tabcolsep}{3.0pt}
\renewcommand{\arraystretch}{1.04}
\caption{Macro-F1 ($\uparrow$; mean$_{\pm\mathrm{std}}$ where available) and end-to-end GFLOPs ($\downarrow$; fusion cost $\mathrm{GF}_{\mathrm{f}}$ in parentheses) on IEMOCAP under 0--40\% test-time modality missingness. Per column, \textcolor{burgundy}{\textbf{bold}} and \underline{underlining} mark the best and second-best distinct displayed values (F1 highest; GF and $\mathrm{GF}_{\mathrm{f}}$ lowest); ties share markings; the backbone reference row is not ranked. The 0\% primary-method entries use the recomputed three-fold summaries.}
\label{tab:missingness-iemocap}
\resizebox{\textwidth}{!}{%
\begin{tabular}{@{}l*{4}{cc!{\color{black!25}\vrule width 0.35pt}}cc@{}}
\toprule
\multirow{2}{*}{\textbf{Method}} & \multicolumn{2}{c}{\textbf{0\%}} & \multicolumn{2}{c}{\textbf{10\%}} & \multicolumn{2}{c}{\textbf{20\%}} & \multicolumn{2}{c}{\textbf{30\%}} & \multicolumn{2}{c}{\textbf{40\%}} \\
\cmidrule(lr){2-3}\cmidrule(lr){4-5}\cmidrule(lr){6-7}\cmidrule(lr){8-9}\cmidrule(lr){10-11}
& F1$\uparrow$ & $\mathrm{GF}\,(\mathrm{GF}_{\mathrm{f}})\downarrow$ & F1$\uparrow$ & $\mathrm{GF}\,(\mathrm{GF}_{\mathrm{f}})\downarrow$ & F1$\uparrow$ & $\mathrm{GF}\,(\mathrm{GF}_{\mathrm{f}})\downarrow$ & F1$\uparrow$ & $\mathrm{GF}\,(\mathrm{GF}_{\mathrm{f}})\downarrow$ & F1$\uparrow$ & $\mathrm{GF}\,(\mathrm{GF}_{\mathrm{f}})\downarrow$ \\
\midrule
MBT & \underline{.64}$_{\pm.02}$ & 6696 (.41) & \underline{.61} & 6026 (.41) & \underline{.60}$_{\pm.02}$ & 5357 (.41) & \underline{.57} & 4688 (.41) & \textcolor{burgundy}{\textbf{.55}}$_{\pm.03}$ & 4022 (.41) \\
MAESTRO & .51$_{\pm.04}$ & 6696 (.23) & .45 & 6026 (.23) & .42$_{\pm.02}$ & 5357 (.23) & .38 & 4688 (.23) & .36$_{\pm.03}$ & 4022 (.23) \\
ShaSpec & .61$_{\pm.03}$ & 6697 (1.51) & .58 & 6027 (1.51) & .56$_{\pm.03}$ & 5358 (1.51) & .55 & 4689 (1.51) & \underline{.51}$_{\pm.03}$ & 4023 (1.51) \\
DecAlign & .48$_{\pm.01}$ & 6707 (11.41) & .46 & 6037 (11.41) & .45$_{\pm.01}$ & 5368 (11.41) & .42 & 4699 (11.41) & .42$_{\pm.01}$ & 4033 (11.41) \\
\arrayrulecolor{black!20}\midrule
\arrayrulecolor{black}
MultiModN & .58$_{\pm.06}$ & 6696 (.76) & .57 & 6027 (.76) & .55$_{\pm.07}$ & 5357 (.76) & .54 & 4688 (.76) & \underline{.51}$_{\pm.06}$ & 4023 (.76) \\
\arrayrulecolor{black!20}\midrule
\arrayrulecolor{black}
GMS & .54$_{\pm.08}$ & 4465 (.23) & .52$_{\pm.06}$ & 4046 (\underline{.20}) & .49$_{\pm.05}$ & 3604 (\underline{.18}) & .47$_{\pm.05}$ & 3146 (\underline{.16}) & .44$_{\pm.04}$ & 2712 (.14) \\
AdaMML & .59$_{\pm.03}$ & 6696 (\underline{.20}) & .55 & 6026 (.22) & .55$_{\pm.04}$ & 5357 (.20) & .48 & 4688 (.22) & .50$_{\pm.02}$ & 4022 (.22) \\
DynMM & .62$_{\pm.01}$ & 6697 (1.78) & .58 & 6028 (1.79) & .56$_{\pm.01}$ & 5358 (1.81) & .50 & 4690 (1.84) & .48$_{\pm.01}$ & 4024 (1.87) \\
DyMo & .56$_{\pm.05}$ & 6696 (.48) & .54 & 6026 (.48) & .53$_{\pm.05}$ & 5357 (.48) & .51 & 4688 (.48) & .49$_{\pm.04}$ & 4022 (.48) \\
MMEE & .61$_{\pm.01}$ & 6696 (.38) & .58 & 6026 (.39) & .56$_{\pm.01}$ & 5357 (.40) & .53 & 4688 (.41) & \underline{.51}$_{\pm.01}$ & 4022 (.42) \\
\arrayrulecolor{black!20}\midrule
\arrayrulecolor{black}
Joint-AFA & .57$_{\pm.04}$ & \textcolor{burgundy}{\textbf{1655}} (\textcolor{burgundy}{\textbf{.13}}) & .55$_{\pm.03}$ & \textcolor{burgundy}{\textbf{1908}} (\textcolor{burgundy}{\textbf{.13}}) & .53$_{\pm.03}$ & \textcolor{burgundy}{\textbf{2067}} (\textcolor{burgundy}{\textbf{.12}}) & .52$_{\pm.01}$ & \underline{2076} (\textcolor{burgundy}{\textbf{.12}}) & .49$_{\pm.02}$ & \underline{2038} (\textcolor{burgundy}{\textbf{.12}}) \\
\arrayrulecolor{black!20}\midrule
\arrayrulecolor{black}
\backbone{} (w/o \policy{}) & .67$_{\pm.01}$ & 6696 (.43) & .63 & 6026 (.39) & .64$_{\pm.01}$ & 5357 (.34) & .60 & 4688 (.30) & .58$_{\pm.03}$ & 4022 (.26) \\
\rowcolor{tolblue!10}
\textbf{\method{} (Ours)} & \textcolor{burgundy}{\textbf{.67}}$_{\pm.02}$ & \underline{3311} (.28) & \textcolor{burgundy}{\textbf{.65}}$_{\pm.02}$ & \underline{2589} (.23) & \textcolor{burgundy}{\textbf{.62}}$_{\pm.00}$ & \underline{2095} (.19) & \textcolor{burgundy}{\textbf{.60}}$_{\pm.01}$ & \textcolor{burgundy}{\textbf{1764}} (\underline{.16}) & \textcolor{burgundy}{\textbf{.55}}$_{\pm.03}$ & \textcolor{burgundy}{\textbf{1453}} (\underline{.13}) \\
\bottomrule
\end{tabular}%
}
\end{table*}

\begin{table*}[t]
\centering
\scriptsize
\setlength{\tabcolsep}{3.0pt}
\renewcommand{\arraystretch}{1.04}
\caption{Macro-F1 ($\uparrow$; mean$_{\pm\mathrm{std}}$ where available) and end-to-end GFLOPs ($\downarrow$; fusion cost $\mathrm{GF}_{\mathrm{f}}$ in parentheses) on MM-Fi under 0--40\% test-time modality missingness. Per column, \textcolor{burgundy}{\textbf{bold}} and \underline{underlining} mark the best and second-best distinct displayed values (F1 highest; GF and $\mathrm{GF}_{\mathrm{f}}$ lowest); ties share markings; the backbone reference row is not ranked. The 0\% primary-method entries use the recomputed three-fold summaries.}
\label{tab:missingness-mmfi}
\resizebox{\textwidth}{!}{%
\begin{tabular}{@{}l*{4}{cc!{\color{black!25}\vrule width 0.35pt}}cc@{}}
\toprule
\multirow{2}{*}{\textbf{Method}} & \multicolumn{2}{c}{\textbf{0\%}} & \multicolumn{2}{c}{\textbf{10\%}} & \multicolumn{2}{c}{\textbf{20\%}} & \multicolumn{2}{c}{\textbf{30\%}} & \multicolumn{2}{c}{\textbf{40\%}} \\
\cmidrule(lr){2-3}\cmidrule(lr){4-5}\cmidrule(lr){6-7}\cmidrule(lr){8-9}\cmidrule(lr){10-11}
& F1$\uparrow$ & $\mathrm{GF}\,(\mathrm{GF}_{\mathrm{f}})\downarrow$ & F1$\uparrow$ & $\mathrm{GF}\,(\mathrm{GF}_{\mathrm{f}})\downarrow$ & F1$\uparrow$ & $\mathrm{GF}\,(\mathrm{GF}_{\mathrm{f}})\downarrow$ & F1$\uparrow$ & $\mathrm{GF}\,(\mathrm{GF}_{\mathrm{f}})\downarrow$ & F1$\uparrow$ & $\mathrm{GF}\,(\mathrm{GF}_{\mathrm{f}})\downarrow$ \\
\midrule
MBT & .77$_{\pm.19}$ & 18098 (.46) & .69 & 16288 (.46) & .63$_{\pm.17}$ & 14480 (.46) & .56 & 12678 (.46) & .52$_{\pm.14}$ & 10896 (.46) \\
MAESTRO & .18$_{\pm.09}$ & 18098 (.28) & .17 & 16288 (.28) & .14$_{\pm.07}$ & 14479 (.28) & .14 & 12677 (.28) & .14$_{\pm.08}$ & 10896 (.28) \\
ShaSpec & .66$_{\pm.08}$ & 18099 (1.68) & .58 & 16289 (1.68) & .56$_{\pm.04}$ & 14481 (1.68) & .47 & 12679 (1.68) & .44$_{\pm.05}$ & 10897 (1.68) \\
DecAlign & .28$_{\pm.19}$ & 18106 (8.45) & .24 & 16296 (8.45) & .24$_{\pm.12}$ & 14488 (8.45) & .21 & 12685 (8.45) & .19$_{\pm.09}$ & 10904 (8.45) \\
\arrayrulecolor{black!20}\midrule
\arrayrulecolor{black}
MultiModN & .40$_{\pm.06}$ & 18098 (.84) & .37 & 16289 (.84) & .35$_{\pm.04}$ & 14480 (.84) & .32 & 12678 (.84) & .29$_{\pm.07}$ & 10896 (.84) \\
\arrayrulecolor{black!20}\midrule
\arrayrulecolor{black}
GMS & .40$_{\pm.11}$ & 18092 (.49) & .36$_{\pm.08}$ & 16429 (.44) & .34$_{\pm.06}$ & 15124 (.39) & .30$_{\pm.06}$ & 13346 (.35) & .28$_{\pm.06}$ & 11782 (.30) \\
AdaMML & \textcolor{burgundy}{\textbf{.96}}$_{\pm.02}$ & 18098 (\underline{.17}) & \underline{.88} & 16288 (\underline{.17}) & \underline{.87}$_{\pm.02}$ & 14479 (\underline{.17}) & \underline{.74} & 12677 (\underline{.17}) & \underline{.72}$_{\pm.03}$ & 10896 (\underline{.17}) \\
DynMM & .71$_{\pm.30}$ & 18099 (1.40) & .66 & 16289 (1.35) & .60$_{\pm.24}$ & 14480 (1.33) & .54 & 12678 (1.33) & .47$_{\pm.17}$ & 10897 (1.30) \\
DyMo & .54$_{\pm.03}$ & 18098 (.52) & .53 & 16288 (.52) & .51$_{\pm.04}$ & 14480 (.52) & .49 & 12678 (.52) & .45$_{\pm.04}$ & 10896 (.52) \\
MMEE & .60$_{\pm.05}$ & 18098 (.59) & .56 & 16288 (.59) & .53$_{\pm.04}$ & 14480 (.59) & .51 & 12678 (.59) & .46$_{\pm.06}$ & 10896 (.59) \\
\arrayrulecolor{black!20}\midrule
\arrayrulecolor{black}
Joint-AFA & .11$_{\pm.10}$ & \underline{12059} (.21) & .11$_{\pm.08}$ & \underline{11099} (.21) & .11$_{\pm.08}$ & \underline{10146} (.21) & .10$_{\pm.07}$ & \underline{9070} (.21) & .10$_{\pm.07}$ & \underline{8208} (.20) \\
\arrayrulecolor{black!20}\midrule
\arrayrulecolor{black}
\backbone{} (w/o \policy{}) & .94$_{\pm.05}$ & 18098 (.48) & .89 & 16288 (.44) & .83$_{\pm.06}$ & 14480 (.39) & .76 & 12677 (.34) & .69$_{\pm.07}$ & 10896 (.30) \\
\rowcolor{tolblue!10}
\textbf{\method{} (Ours)} & \underline{.95}$_{\pm.04}$ & \textcolor{burgundy}{\textbf{738}} (\textcolor{burgundy}{\textbf{.07}}) & \textcolor{burgundy}{\textbf{.91}}$_{\pm.06}$ & \textcolor{burgundy}{\textbf{2592}} (\textcolor{burgundy}{\textbf{.08}}) & \textcolor{burgundy}{\textbf{.88}}$_{\pm.06}$ & \textcolor{burgundy}{\textbf{4188}} (\textcolor{burgundy}{\textbf{.09}}) & \textcolor{burgundy}{\textbf{.83}}$_{\pm.06}$ & \textcolor{burgundy}{\textbf{4659}} (\textcolor{burgundy}{\textbf{.09}}) & \textcolor{burgundy}{\textbf{.77}}$_{\pm.09}$ & \textcolor{burgundy}{\textbf{5189}} (\textcolor{burgundy}{\textbf{.09}}) \\
\bottomrule
\end{tabular}%
}
\end{table*}

\begin{table*}[t]
\centering
\scriptsize
\setlength{\tabcolsep}{3.0pt}
\renewcommand{\arraystretch}{1.04}
\caption{Macro-F1 ($\uparrow$; mean$_{\pm\mathrm{std}}$ where available) and end-to-end GFLOPs ($\downarrow$; fusion cost $\mathrm{GF}_{\mathrm{f}}$ in parentheses) on CMI under 0--40\% test-time modality missingness. Per column, \textcolor{burgundy}{\textbf{bold}} and \underline{underlining} mark the best and second-best distinct displayed values (F1 highest; GF and $\mathrm{GF}_{\mathrm{f}}$ lowest); ties share markings; the backbone reference row is not ranked. The 0\% primary-method entries use the recomputed three-fold summaries.}
\label{tab:missingness-cmi}
\resizebox{\textwidth}{!}{%
\begin{tabular}{@{}l*{4}{cc!{\color{black!25}\vrule width 0.35pt}}cc@{}}
\toprule
\multirow{2}{*}{\textbf{Method}} & \multicolumn{2}{c}{\textbf{0\%}} & \multicolumn{2}{c}{\textbf{10\%}} & \multicolumn{2}{c}{\textbf{20\%}} & \multicolumn{2}{c}{\textbf{30\%}} & \multicolumn{2}{c}{\textbf{40\%}} \\
\cmidrule(lr){2-3}\cmidrule(lr){4-5}\cmidrule(lr){6-7}\cmidrule(lr){8-9}\cmidrule(lr){10-11}
& F1$\uparrow$ & $\mathrm{GF}\,(\mathrm{GF}_{\mathrm{f}})\downarrow$ & F1$\uparrow$ & $\mathrm{GF}\,(\mathrm{GF}_{\mathrm{f}})\downarrow$ & F1$\uparrow$ & $\mathrm{GF}\,(\mathrm{GF}_{\mathrm{f}})\downarrow$ & F1$\uparrow$ & $\mathrm{GF}\,(\mathrm{GF}_{\mathrm{f}})\downarrow$ & F1$\uparrow$ & $\mathrm{GF}\,(\mathrm{GF}_{\mathrm{f}})\downarrow$ \\
\midrule
MBT & \underline{.55}$_{\pm.01}$ & 0.55 (.55) & \underline{.53} & 0.55 (.55) & \textcolor{burgundy}{\textbf{.52}}$_{\pm.01}$ & 0.55 (.55) & .47 & 0.55 (.55) & \textcolor{burgundy}{\textbf{.46}}$_{\pm.02}$ & 0.55 (.55) \\
MAESTRO & .47$_{\pm.02}$ & 0.28 (.28) & .42 & 0.28 (.28) & .37$_{\pm.01}$ & 0.28 (.28) & .32 & 0.28 (.28) & .27$_{\pm.02}$ & 0.28 (.28) \\
ShaSpec & .53$_{\pm.01}$ & 2.14 (2.14) & .52 & 2.14 (2.14) & .50$_{\pm.02}$ & 2.14 (2.14) & \underline{.48} & 2.14 (2.14) & \underline{.45}$_{\pm.01}$ & 2.14 (2.14) \\
DecAlign & \textcolor{burgundy}{\textbf{.56}}$_{\pm.02}$ & 23.33 (23.33) & \textcolor{burgundy}{\textbf{.54}} & 23.33 (23.33) & \textcolor{burgundy}{\textbf{.52}}$_{\pm.01}$ & 23.33 (23.33) & \textcolor{burgundy}{\textbf{.49}} & 23.33 (23.33) & \underline{.45}$_{\pm.01}$ & 23.33 (23.33) \\
\arrayrulecolor{black!20}\midrule
\arrayrulecolor{black}
MultiModN & .45$_{\pm.01}$ & 1.07 (1.07) & .44 & 1.07 (1.07) & .43$_{\pm.01}$ & 1.07 (1.07) & .41 & 1.07 (1.07) & .40$_{\pm.02}$ & 1.07 (1.07) \\
\arrayrulecolor{black!20}\midrule
\arrayrulecolor{black}
GMS & .48$_{\pm.01}$ & \underline{0.11} (\underline{.11}) & .42$_{\pm.01}$ & \underline{0.10} (\underline{.10}) & .37$_{\pm.01}$ & \underline{0.09} (\underline{.09}) & .33$_{\pm.01}$ & \underline{0.07} (\underline{.07}) & .29$_{\pm.02}$ & \textcolor{burgundy}{\textbf{0.06}} (\textcolor{burgundy}{\textbf{.06}}) \\
AdaMML & .52$_{\pm.01}$ & 0.51 (.51) & .49 & 0.45 (.45) & .47$_{\pm.01}$ & 0.48 (.48) & .41 & 0.43 (.43) & .41$_{\pm.02}$ & 0.44 (.44) \\
DynMM & .54$_{\pm.01}$ & 3.13 (3.13) & .51 & 3.12 (3.12) & .49$_{\pm.00}$ & 3.10 (3.10) & .46 & 3.07 (3.07) & .43$_{\pm.01}$ & 3.02 (3.02) \\
DyMo & .37$_{\pm.01}$ & 0.65 (.65) & .38 & 0.65 (.65) & .38$_{\pm.01}$ & 0.65 (.65) & .37 & 0.65 (.65) & .37$_{\pm.01}$ & 0.65 (.65) \\
MMEE & .48$_{\pm.01}$ & 0.70 (.70) & .46 & 0.70 (.70) & .44$_{\pm.01}$ & 0.70 (.70) & .42 & 0.71 (.71) & .39$_{\pm.02}$ & 0.71 (.71) \\
\arrayrulecolor{black!20}\midrule
\arrayrulecolor{black}
Joint-AFA & .38$_{\pm.03}$ & \textcolor{burgundy}{\textbf{0.07}} (\textcolor{burgundy}{\textbf{.07}}) & .37$_{\pm.03}$ & \textcolor{burgundy}{\textbf{0.07}} (\textcolor{burgundy}{\textbf{.07}}) & .37$_{\pm.02}$ & \textcolor{burgundy}{\textbf{0.06}} (\textcolor{burgundy}{\textbf{.06}}) & .36$_{\pm.02}$ & \textcolor{burgundy}{\textbf{0.06}} (\textcolor{burgundy}{\textbf{.06}}) & .35$_{\pm.02}$ & \textcolor{burgundy}{\textbf{0.06}} (\textcolor{burgundy}{\textbf{.06}}) \\
\arrayrulecolor{black!20}\midrule
\arrayrulecolor{black}
\backbone{} (w/o \policy{}) & .57$_{\pm.01}$ & 0.58 (.58) & .55 & 0.52 (.52) & .54$_{\pm.01}$ & 0.46 (.46) & .51 & 0.40 (.40) & .49$_{\pm.01}$ & 0.35 (.35) \\
\rowcolor{tolblue!10}
\textbf{\method{} (Ours)} & \textcolor{burgundy}{\textbf{.56}}$_{\pm.01}$ & 0.43 (.43) & \underline{.53}$_{\pm.01}$ & 0.37 (.37) & \underline{.51}$_{\pm.02}$ & 0.31 (.31) & \textcolor{burgundy}{\textbf{.49}}$_{\pm.03}$ & 0.26 (.26) & \textcolor{burgundy}{\textbf{.46}}$_{\pm.02}$ & \underline{0.22} (\underline{.22}) \\
\bottomrule
\end{tabular}%
}
\end{table*}

\begin{table*}[t]
\centering
\scriptsize
\setlength{\tabcolsep}{3.0pt}
\renewcommand{\arraystretch}{1.04}
\caption{Macro-F1 ($\uparrow$; mean$_{\pm\mathrm{std}}$ where available) and end-to-end GFLOPs ($\downarrow$; fusion cost $\mathrm{GF}_{\mathrm{f}}$ in parentheses) on CZU-MHAD under 0--40\% test-time modality missingness. Per column, \textcolor{burgundy}{\textbf{bold}} and \underline{underlining} mark the best and second-best distinct displayed values (F1 highest; GF and $\mathrm{GF}_{\mathrm{f}}$ lowest); ties share markings; the backbone reference row is not ranked. The 0\% primary-method entries use the recomputed three-fold summaries.}
\label{tab:missingness-czumhad}
\resizebox{\textwidth}{!}{%
\begin{tabular}{@{}l*{4}{cc!{\color{black!25}\vrule width 0.35pt}}cc@{}}
\toprule
\multirow{2}{*}{\textbf{Method}} & \multicolumn{2}{c}{\textbf{0\%}} & \multicolumn{2}{c}{\textbf{10\%}} & \multicolumn{2}{c}{\textbf{20\%}} & \multicolumn{2}{c}{\textbf{30\%}} & \multicolumn{2}{c}{\textbf{40\%}} \\
\cmidrule(lr){2-3}\cmidrule(lr){4-5}\cmidrule(lr){6-7}\cmidrule(lr){8-9}\cmidrule(lr){10-11}
& F1$\uparrow$ & $\mathrm{GF}\,(\mathrm{GF}_{\mathrm{f}})\downarrow$ & F1$\uparrow$ & $\mathrm{GF}\,(\mathrm{GF}_{\mathrm{f}})\downarrow$ & F1$\uparrow$ & $\mathrm{GF}\,(\mathrm{GF}_{\mathrm{f}})\downarrow$ & F1$\uparrow$ & $\mathrm{GF}\,(\mathrm{GF}_{\mathrm{f}})\downarrow$ & F1$\uparrow$ & $\mathrm{GF}\,(\mathrm{GF}_{\mathrm{f}})\downarrow$ \\
\midrule
MBT & .82$_{\pm.11}$ & \textcolor{burgundy}{\textbf{6091}} (.56) & .75 & \textcolor{burgundy}{\textbf{5482}} (.56) & .69$_{\pm.10}$ & \textcolor{burgundy}{\textbf{4873}} (.56) & .65 & \textcolor{burgundy}{\textbf{4264}} (.56) & \underline{.57}$_{\pm.10}$ & \textcolor{burgundy}{\textbf{3656}} (.56) \\
MAESTRO & .25$_{\pm.16}$ & \textcolor{burgundy}{\textbf{6091}} (.30) & .22 & \textcolor{burgundy}{\textbf{5482}} (.30) & .20$_{\pm.11}$ & \textcolor{burgundy}{\textbf{4873}} (.30) & .17 & \textcolor{burgundy}{\textbf{4264}} (.30) & .14$_{\pm.07}$ & \textcolor{burgundy}{\textbf{3656}} (.30) \\
ShaSpec & .59$_{\pm.06}$ & 6093 (2.17) & .54 & 5484 (2.17) & .49$_{\pm.04}$ & \underline{4874} (2.17) & .47 & 4266 (2.17) & .42$_{\pm.05}$ & 3658 (2.17) \\
DecAlign & .36$_{\pm.10}$ & 6114 (23.37) & .34 & 5505 (23.37) & .29$_{\pm.07}$ & 4896 (23.37) & .30 & 4287 (23.37) & .24$_{\pm.05}$ & 3679 (23.37) \\
\arrayrulecolor{black!20}\midrule
\arrayrulecolor{black}
MultiModN & .47$_{\pm.09}$ & \underline{6092} (1.08) & .44 & \textcolor{burgundy}{\textbf{5482}} (1.08) & .41$_{\pm.09}$ & \underline{4874} (1.08) & .38 & \underline{4265} (1.08) & .34$_{\pm.05}$ & \underline{3657} (1.08) \\
\arrayrulecolor{black!20}\midrule
\arrayrulecolor{black}
GMS & .44$_{\pm.10}$ & \textcolor{burgundy}{\textbf{6091}} (\textcolor{burgundy}{\textbf{.13}}) & .41$_{\pm.07}$ & 5487 (\underline{.12}) & .38$_{\pm.07}$ & 4887 (\textcolor{burgundy}{\textbf{.11}}) & .35$_{\pm.07}$ & 4320 (\underline{.10}) & .32$_{\pm.06}$ & 3749 (\textcolor{burgundy}{\textbf{.08}}) \\
AdaMML & \underline{.83}$_{\pm.09}$ & \textcolor{burgundy}{\textbf{6091}} (\underline{.17}) & \underline{.78} & \textcolor{burgundy}{\textbf{5482}} (\textcolor{burgundy}{\textbf{.09}}) & \underline{.72}$_{\pm.08}$ & \textcolor{burgundy}{\textbf{4873}} (.18) & \textcolor{burgundy}{\textbf{.69}} & \textcolor{burgundy}{\textbf{4264}} (\textcolor{burgundy}{\textbf{.09}}) & .56$_{\pm.07}$ & \textcolor{burgundy}{\textbf{3656}} (.19) \\
DynMM & .75$_{\pm.08}$ & \underline{6092} (1.54) & .67 & \underline{5483} (1.58) & .62$_{\pm.10}$ & \underline{4874} (1.61) & .57 & \underline{4265} (1.70) & .48$_{\pm.05}$ & \underline{3657} (1.75) \\
DyMo & .65$_{\pm.12}$ & \textcolor{burgundy}{\textbf{6091}} (.68) & .61 & \textcolor{burgundy}{\textbf{5482}} (.68) & .56$_{\pm.12}$ & \textcolor{burgundy}{\textbf{4873}} (.68) & .52 & \textcolor{burgundy}{\textbf{4264}} (.68) & .46$_{\pm.10}$ & \textcolor{burgundy}{\textbf{3656}} (.68) \\
MMEE & .57$_{\pm.09}$ & \textcolor{burgundy}{\textbf{6091}} (.73) & .52 & \textcolor{burgundy}{\textbf{5482}} (.73) & .52$_{\pm.10}$ & \textcolor{burgundy}{\textbf{4873}} (.73) & .47 & \textcolor{burgundy}{\textbf{4264}} (.73) & .42$_{\pm.08}$ & \textcolor{burgundy}{\textbf{3656}} (.73) \\
\arrayrulecolor{black!20}\midrule
\arrayrulecolor{black}
Joint-AFA & .28$_{\pm.05}$ & \textcolor{burgundy}{\textbf{6091}} (\textcolor{burgundy}{\textbf{.13}}) & .27$_{\pm.04}$ & 5487 (\underline{.12}) & .26$_{\pm.04}$ & 4887 (\underline{.12}) & .25$_{\pm.04}$ & 4320 (.11) & .23$_{\pm.04}$ & 3750 (\underline{.10}) \\
\arrayrulecolor{black!20}\midrule
\arrayrulecolor{black}
\backbone{} (w/o \policy{}) & .92$_{\pm.07}$ & 6091 (.59) & .84 & 5482 (.53) & .76$_{\pm.06}$ & 4873 (.47) & .71 & 4264 (.41) & .61$_{\pm.03}$ & 3656 (.36) \\
\rowcolor{tolblue!10}
\textbf{\method{} (Ours)} & \textcolor{burgundy}{\textbf{.91}}$_{\pm.10}$ & \textcolor{burgundy}{\textbf{6091}} (.23) & \textcolor{burgundy}{\textbf{.83}}$_{\pm.08}$ & 5487 (.22) & \textcolor{burgundy}{\textbf{.75}}$_{\pm.07}$ & 4888 (.20) & \underline{.68}$_{\pm.07}$ & 4321 (.19) & \textcolor{burgundy}{\textbf{.61}}$_{\pm.05}$ & 3750 (.17) \\
\bottomrule
\end{tabular}%
}
\end{table*}

\begin{table*}[t]
\centering
\scriptsize
\setlength{\tabcolsep}{3.0pt}
\renewcommand{\arraystretch}{1.04}
\caption{Macro-F1 ($\uparrow$; mean$_{\pm\mathrm{std}}$ where available) and end-to-end GFLOPs ($\downarrow$; fusion cost $\mathrm{GF}_{\mathrm{f}}$ in parentheses) on EAV under 0--40\% test-time modality missingness. Per column, \textcolor{burgundy}{\textbf{bold}} and \underline{underlining} mark the best and second-best distinct displayed values (F1 highest; GF and $\mathrm{GF}_{\mathrm{f}}$ lowest); ties share markings; the backbone reference row is not ranked. The 0\% primary-method entries use the recomputed three-fold summaries.}
\label{tab:missingness-eav}
\resizebox{\textwidth}{!}{%
\begin{tabular}{@{}l*{4}{cc!{\color{black!25}\vrule width 0.35pt}}cc@{}}
\toprule
\multirow{2}{*}{\textbf{Method}} & \multicolumn{2}{c}{\textbf{0\%}} & \multicolumn{2}{c}{\textbf{10\%}} & \multicolumn{2}{c}{\textbf{20\%}} & \multicolumn{2}{c}{\textbf{30\%}} & \multicolumn{2}{c}{\textbf{40\%}} \\
\cmidrule(lr){2-3}\cmidrule(lr){4-5}\cmidrule(lr){6-7}\cmidrule(lr){8-9}\cmidrule(lr){10-11}
& F1$\uparrow$ & $\mathrm{GF}\,(\mathrm{GF}_{\mathrm{f}})\downarrow$ & F1$\uparrow$ & $\mathrm{GF}\,(\mathrm{GF}_{\mathrm{f}})\downarrow$ & F1$\uparrow$ & $\mathrm{GF}\,(\mathrm{GF}_{\mathrm{f}})\downarrow$ & F1$\uparrow$ & $\mathrm{GF}\,(\mathrm{GF}_{\mathrm{f}})\downarrow$ & F1$\uparrow$ & $\mathrm{GF}\,(\mathrm{GF}_{\mathrm{f}})\downarrow$ \\
\midrule
MBT & .61$_{\pm.03}$ & 6230 (.26) & .58 & 5609 (.26) & .55$_{\pm.02}$ & 5001 (.26) & .53 & 4417 (.26) & .50$_{\pm.01}$ & 3871 (.26) \\
MAESTRO & .53$_{\pm.02}$ & 6230 (.15) & .48 & 5609 (.15) & .45$_{\pm.01}$ & 5001 (.15) & .41 & 4417 (.15) & .36$_{\pm.01}$ & 3871 (.15) \\
ShaSpec & .60$_{\pm.05}$ & 6231 (.97) & .58 & 5610 (.97) & .56$_{\pm.04}$ & 5001 (.97) & .52 & 4418 (.97) & .50$_{\pm.02}$ & 3872 (.97) \\
DecAlign & .62$_{\pm.02}$ & 6232 (1.77) & .58 & 5611 (1.77) & .55$_{\pm.01}$ & 5002 (1.77) & .51 & 4419 (1.77) & .48$_{\pm.01}$ & 3872 (1.77) \\
\arrayrulecolor{black!20}\midrule
\arrayrulecolor{black}
MultiModN & .65$_{\pm.04}$ & 6230 (.49) & \underline{.61} & 5609 (.49) & \underline{.59}$_{\pm.04}$ & 5001 (.49) & \underline{.56} & 4417 (.49) & \underline{.54}$_{\pm.02}$ & 3871 (.49) \\
\arrayrulecolor{black!20}\midrule
\arrayrulecolor{black}
GMS & \underline{.66}$_{\pm.04}$ & \textcolor{burgundy}{\textbf{139}} (\textcolor{burgundy}{\textbf{.09}}) & \underline{.61}$_{\pm.03}$ & \textcolor{burgundy}{\textbf{126}} (\textcolor{burgundy}{\textbf{.08}}) & .57$_{\pm.02}$ & \textcolor{burgundy}{\textbf{114}} (\textcolor{burgundy}{\textbf{.07}}) & .52$_{\pm.03}$ & \textcolor{burgundy}{\textbf{101}} (\textcolor{burgundy}{\textbf{.06}}) & .48$_{\pm.02}$ & \textcolor{burgundy}{\textbf{88.70}} (\textcolor{burgundy}{\textbf{.06}}) \\
AdaMML & .64$_{\pm.07}$ & 6230 (.17) & .60 & 5609 (.18) & .58$_{\pm.06}$ & 5001 (.17) & .52 & 4417 (.18) & .51$_{\pm.04}$ & 3871 (.17) \\
DynMM & \underline{.66}$_{\pm.01}$ & 6230 (.54) & \underline{.61} & 5609 (.54) & .58$_{\pm.01}$ & 5001 (.54) & .53 & 4417 (.54) & .49$_{\pm.01}$ & 3871 (.55) \\
DyMo & .60$_{\pm.04}$ & 6230 (.28) & .58 & 5609 (.28) & .56$_{\pm.04}$ & 5001 (.28) & .53 & 4417 (.28) & .50$_{\pm.03}$ & 3871 (.28) \\
MMEE & .60$_{\pm.07}$ & 6230 (.18) & .56 & 5609 (.18) & .52$_{\pm.06}$ & 5001 (.18) & .48 & 4417 (.18) & .44$_{\pm.06}$ & 3871 (.18) \\
\arrayrulecolor{black!20}\midrule
\arrayrulecolor{black}
Joint-AFA & .54$_{\pm.03}$ & 3321 (.16) & .53$_{\pm.02}$ & 3235 (.15) & .50$_{\pm.02}$ & 3162 (.14) & .47$_{\pm.02}$ & 3022 (.13) & .46$_{\pm.01}$ & 2873 (.12) \\
\arrayrulecolor{black!20}\midrule
\arrayrulecolor{black}
\backbone{} (w/o \policy{}) & .67$_{\pm.03}$ & 6230 (.27) & .64 & 5609 (.25) & .61$_{\pm.03}$ & 5001 (.22) & .58 & 4417 (.20) & .56$_{\pm.02}$ & 3871 (.17) \\
\rowcolor{tolblue!10}
\textbf{\method{} (Ours)} & \textcolor{burgundy}{\textbf{.69}}$_{\pm.02}$ & \underline{719} (\underline{.10}) & \textcolor{burgundy}{\textbf{.65}}$_{\pm.03}$ & \underline{1075} (\underline{.10}) & \textcolor{burgundy}{\textbf{.62}}$_{\pm.02}$ & \underline{1311} (\underline{.10}) & \textcolor{burgundy}{\textbf{.59}}$_{\pm.03}$ & \underline{1557} (\underline{.10}) & \textcolor{burgundy}{\textbf{.56}}$_{\pm.01}$ & \underline{1736} (\underline{.10}) \\
\bottomrule
\end{tabular}%
}
\end{table*}

\begin{table*}[t]
\centering
\scriptsize
\setlength{\tabcolsep}{3.0pt}
\renewcommand{\arraystretch}{1.04}
\caption{Macro-F1 ($\uparrow$; mean$_{\pm\mathrm{std}}$ where available) and end-to-end GFLOPs ($\downarrow$; fusion cost $\mathrm{GF}_{\mathrm{f}}$ in parentheses) on UTD-MHAD under 0--40\% test-time modality missingness. Per column, \textcolor{burgundy}{\textbf{bold}} and \underline{underlining} mark the best and second-best distinct displayed values (F1 highest; GF and $\mathrm{GF}_{\mathrm{f}}$ lowest); ties share markings; the backbone reference row is not ranked. The 0\% primary-method entries use the recomputed three-fold summaries.}
\label{tab:missingness-utdmhad}
\resizebox{\textwidth}{!}{%
\begin{tabular}{@{}l*{4}{cc!{\color{black!25}\vrule width 0.35pt}}cc@{}}
\toprule
\multirow{2}{*}{\textbf{Method}} & \multicolumn{2}{c}{\textbf{0\%}} & \multicolumn{2}{c}{\textbf{10\%}} & \multicolumn{2}{c}{\textbf{20\%}} & \multicolumn{2}{c}{\textbf{30\%}} & \multicolumn{2}{c}{\textbf{40\%}} \\
\cmidrule(lr){2-3}\cmidrule(lr){4-5}\cmidrule(lr){6-7}\cmidrule(lr){8-9}\cmidrule(lr){10-11}
& F1$\uparrow$ & $\mathrm{GF}\,(\mathrm{GF}_{\mathrm{f}})\downarrow$ & F1$\uparrow$ & $\mathrm{GF}\,(\mathrm{GF}_{\mathrm{f}})\downarrow$ & F1$\uparrow$ & $\mathrm{GF}\,(\mathrm{GF}_{\mathrm{f}})\downarrow$ & F1$\uparrow$ & $\mathrm{GF}\,(\mathrm{GF}_{\mathrm{f}})\downarrow$ & F1$\uparrow$ & $\mathrm{GF}\,(\mathrm{GF}_{\mathrm{f}})\downarrow$ \\
\midrule
MBT & \underline{.81}$_{\pm.06}$ & 7306 (.35) & \underline{.79}$_{\pm.06}$ & 6576 (.34) & \underline{.73}$_{\pm.05}$ & 5848 (.34) & \underline{.67}$_{\pm.06}$ & 5129 (.34) & \underline{.63}$_{\pm.06}$ & 4431 (.34) \\
MAESTRO & .70$_{\pm.02}$ & 7306 (.19) & .65$_{\pm.01}$ & 6576 (.19) & .57$_{\pm.02}$ & 5848 (.19) & .54$_{\pm.02}$ & 5129 (.19) & .49$_{\pm.01}$ & 4430 (.19) \\
ShaSpec & .71$_{\pm.13}$ & 7307 (1.29) & .66$_{\pm.11}$ & 6577 (1.29) & .62$_{\pm.08}$ & 5849 (1.29) & .54$_{\pm.08}$ & 5130 (1.29) & .47$_{\pm.06}$ & 4432 (1.29) \\
DecAlign & .75$_{\pm.08}$ & 7310 (4.21) & .71$_{\pm.09}$ & 6580 (4.21) & .62$_{\pm.08}$ & 5852 (4.21) & .59$_{\pm.10}$ & 5133 (4.21) & .56$_{\pm.06}$ & 4434 (4.21) \\
\arrayrulecolor{black!20}\midrule
\arrayrulecolor{black}
MultiModN & .66$_{\pm.10}$ & 7306 (.64) & .62$_{\pm.08}$ & 6576 (.64) & .55$_{\pm.06}$ & 5848 (.64) & .46$_{\pm.08}$ & 5130 (.64) & .41$_{\pm.07}$ & 4431 (.64) \\
\arrayrulecolor{black!20}\midrule
\arrayrulecolor{black}
GMS & .61$_{\pm.10}$ & 5185 (\underline{.17}) & .51$_{\pm.08}$ & 4706 (\underline{.15}) & .42$_{\pm.06}$ & \underline{4183} (\underline{.14}) & .34$_{\pm.04}$ & \underline{3695} (\underline{.12}) & .29$_{\pm.02}$ & \underline{3259} (\underline{.11}) \\
AdaMML & .75$_{\pm.06}$ & 7306 (.19) & .73$_{\pm.05}$ & 6576 (.19) & .69$_{\pm.06}$ & 5848 (.19) & \underline{.67}$_{\pm.05}$ & 5129 (.19) & .62$_{\pm.03}$ & 4430 (.20) \\
DynMM & .78$_{\pm.05}$ & 7308 (1.91) & .73$_{\pm.08}$ & 6577 (1.91) & .66$_{\pm.05}$ & 5850 (1.91) & .63$_{\pm.06}$ & 5131 (1.91) & .56$_{\pm.06}$ & 4432 (1.91) \\
DyMo & .60$_{\pm.06}$ & 7306 (.38) & .59$_{\pm.06}$ & 6576 (.38) & .57$_{\pm.05}$ & 5848 (.38) & .55$_{\pm.04}$ & 5129 (.38) & .52$_{\pm.04}$ & 4431 (.38) \\
MMEE & .65$_{\pm.04}$ & 7306 (.44) & .63$_{\pm.06}$ & 6576 (.44) & .58$_{\pm.03}$ & 5848 (.44) & .55$_{\pm.02}$ & 5129 (.44) & .50$_{\pm.02}$ & 4431 (.44) \\
\arrayrulecolor{black!20}\midrule
\arrayrulecolor{black}
Joint-AFA & .30$_{\pm.04}$ & \textcolor{burgundy}{\textbf{755}} (\textcolor{burgundy}{\textbf{.04}}) & .28$_{\pm.03}$ & \textcolor{burgundy}{\textbf{1272}} (\textcolor{burgundy}{\textbf{.06}}) & .27$_{\pm.03}$ & \textcolor{burgundy}{\textbf{1587}} (\textcolor{burgundy}{\textbf{.07}}) & .26$_{\pm.04}$ & \textcolor{burgundy}{\textbf{1860}} (\textcolor{burgundy}{\textbf{.08}}) & .24$_{\pm.03}$ & \textcolor{burgundy}{\textbf{2171}} (\textcolor{burgundy}{\textbf{.09}}) \\
\arrayrulecolor{black!20}\midrule
\arrayrulecolor{black}
\backbone{} (w/o \policy{}) & .86$_{\pm.08}$ & 7306 (.36) & .83$_{\pm.07}$ & 6576 (.32) & .77$_{\pm.10}$ & 5848 (.29) & .73$_{\pm.10}$ & 5129 (.25) & .70$_{\pm.09}$ & 4430 (.22) \\
\rowcolor{tolblue!10}
\textbf{\method{} (Ours)} & \textcolor{burgundy}{\textbf{.83}}$_{\pm.10}$ & \underline{4706} (.19) & \textcolor{burgundy}{\textbf{.81}}$_{\pm.09}$ & \underline{4543} (.18) & \textcolor{burgundy}{\textbf{.75}}$_{\pm.09}$ & 4223 (.16) & \textcolor{burgundy}{\textbf{.69}}$_{\pm.09}$ & 3924 (.15) & \textcolor{burgundy}{\textbf{.67}}$_{\pm.08}$ & 3553 (.13) \\
\bottomrule
\end{tabular}%
}
\end{table*}

\section{Complete Component Replacement Study Results for Table~\ref{tab:component-swap}} \label{app:comp-replace-complete}

Table~\ref{tab:comp-replace-complete} reports per-dataset macro-F1,
modality utilization (MU), and latency for the configurations in
Table~\ref{tab:component-swap}, whose means weight all six datasets equally.

\paragraph{Latency.}
Component-replacement latency is measured end-to-end on one NVIDIA L40 GPU, on an otherwise idle server, with batch size~1 and no feature caching. For each sample $i$, the controller determines the acquired set $S_i$ and thus the raw-data extraction path; different controllers can request different raw modalities for the same sample. The result counts extraction only for those acquired modalities, together with the executed fusion steps, readouts, and selector calls. We average sample results within each dataset and weight the six dataset means equally. 

\paragraph{Replacing the controller.}
\emph{\backbone{} + GMS} provides a static-selection comparison using the same fusion backbone. At the reported operating points, \method{} achieves both higher mean macro-F1 and lower mean latency across the six datasets (Table~\ref{tab:component-swap}).
This comparison supports the benefit of our controller beyond the choice of backbone, although the advantage varies by dataset and acquisition budget. EAMA and Joint-AFA replace \policy{} with pre-backbone selectors. Both select modalities from raw previews of the inputs acquired so far and never read an intermediate \backbone{} prediction. Joint-AFA acquires one modality per decision, whereas EAMA acquires a nonempty subset. After acquisition terminates, the frozen \backbone{} processes the final acquired set once. Their action spaces and stopping rules are detailed in Appendix~\ref{app:baseline_adaptations}. 

\paragraph{Replacing the backbone.}
In MultiModN + \policy{}, the sequential fusion backbone \backbone{} is replaced by MultiModN with \policy{} stays the decider for modality acquisition and stop action.

\begin{table*}[t]
\centering
\scriptsize
\setlength{\tabcolsep}{2.6pt}
\renewcommand{\arraystretch}{1.04}
\caption{Complete component-replacement results: macro-F1 ($\uparrow$),
modality utilization MU ($\downarrow$) and GPU latency $t$ in ms
($\downarrow$) per dataset. Table~\ref{tab:component-swap} reports the row
means. EAMA and Joint-AFA are the pre-backbone adaptations of Appendix~\ref{app:baseline_adaptations}; EAMA uses uniform costs and the hybrid optimizer. MultiModN F1 and MU are recomputed from the main-table three-fold source. F1 can be compared across backbones on the same task and test set, but only shared-backbone comparisons isolate controller changes.}
\label{tab:comp-replace-complete}
\resizebox{\textwidth}{!}{%
\begin{tabular}{@{}l*{5}{ccc!{\color{black!25}\vrule width 0.35pt}}ccc@{}}
\toprule
\multirow{2}{*}{\textbf{Configuration}}
& \multicolumn{3}{c}{\textbf{IEMOCAP}} & \multicolumn{3}{c}{\textbf{MM-Fi}}
& \multicolumn{3}{c}{\textbf{CZU-MHAD}} & \multicolumn{3}{c}{\textbf{EAV}}
& \multicolumn{3}{c}{\textbf{CMI}} & \multicolumn{3}{c}{\textbf{UTD-MHAD}} \\
\cmidrule(lr){2-4}\cmidrule(lr){5-7}\cmidrule(lr){8-10}\cmidrule(lr){11-13}\cmidrule(lr){14-16}\cmidrule(lr){17-19}
& F1 & MU & Latency & F1 & MU & Latency & F1 & MU & Latency & F1 & MU & Latency & F1 & MU & Latency & F1 & MU & Latency \\
\midrule
\rowcolor{tolblue!10}
\textbf{\method{}}
& .669 & .65 & 745 & .950 & .21 & 187 & .914 & .37 & 1301 & .689 & .37 & 217 & .556 & .75 & 12.6 & .827 & .52 & 659 \\
\arrayrulecolor{black!20}\midrule\arrayrulecolor{black}
MultiModN + \policy{}
& .535 & .33 & 199 & .377 & .20 & 4004 & .485 & .16 & 1294 & .647 & .39 & 319 & .386 & .28 & 4.4 & .734 & .40 & 653 \\
MultiModN
& .580 & 1.00 & 1530 & .397 & 1.00 & 4234 & .467 & 1.00 & 1324 & .647 & 1.00 & 2032 & .447 & 1.00 & 15.5 & .660 & 1.00 & 1036 \\
\arrayrulecolor{black!20}\midrule\arrayrulecolor{black}
\backbone{} + EAMA
& .650 & .53 & 338 & .945 & .22 & 282 & .899 & .23 & 1296 & .687 & .41 & 216 & .523 & .45 & 6.8 & .757 & .49 & 642 \\
\backbone{} + Joint-AFA
& .621 & .34 & 332 & .945 & .49 & 2485 & .911 & .61 & 1311 & .683 & .51 & 603 & .517 & .62 & 10.0 & .817 & .72 & 780 \\
\backbone{} + GMS
& .676 & .83 & 95.9 & .759 & .20 & 2677 & .917 & .57 & 1306 & .688 & .33 & 26.7 & .567 & .71 & 10.9 & .793 & .50 & 516 \\
\bottomrule
\end{tabular}%
}
\end{table*}

\section{Complete Ablation Results for Figure~\ref{fig:ablation}} \label{app:ablate-complete}

Table~\ref{tab:iemocap-ablation} gives the complete IEMOCAP ablation behind
Figure~\ref{fig:ablation}. Every row changes exactly one component of
\method{} and is evaluated on the same three test folds under full modality
availability. \emph{\backbone{} ablations} retrain the backbone with the listed
change and then fit \policy{} on the retrained backbone with the deployed
per-fold seed and training price $\lambda$, the prior priced at the same $\lambda$; the
column ``\backbone{} F1'' reports the retrained backbone alone, consuming all
six modalities and averaged over the same 24 modality orders for every
checkpoint, so that an order-fragile variant is not judged on a single
favourable or unfavourable order. \emph{\policy{} ablations} keep the deployed
backbone and change only the controller, again at the deployed per-fold seed
and $\lambda$. The relative change is taken with respect to the
reference macro-F1 of $.669$, as in Figure~\ref{fig:ablation}.

The variants are defined as follows. \emph{Fixed modality order}: \backbone{}
is trained on one fixed modality order instead of a freshly sampled order per
batch. \emph{No dropout curriculum}: every training sample keeps all
modalities. \emph{Bottleneck tokens last}: the fusion blocks append the
bottleneck after the modality tokens instead of prepending it; this row
averages three training seeds, since the two placements are the same function
at initialization and differ only through optimization. \emph{Learnable
aggregation}: the running-mean aggregation is replaced
by a learned element-wise gate. $\mathcal L_{\mathrm{CE}}$: prefix supervision
and prefix distillation are removed, leaving only the cross-entropy of the full
retained prefix. \emph{Shapley + conf.\ exit}: no learned controller;
modalities are acquired in the population Shapley order and acquisition stops
once $\max_c p_{s,c}\ge0.8$. \emph{Zero value init.}: the prior is removed
($Q_0\equiv0$) and the network learns the full action value by fitted-$Q$
iteration. \emph{No STOP}: the \textsc{stop} action is removed, so every
available modality is acquired and \policy{} only chooses the order.
\emph{Conf.\ STOP}: the learned \textsc{stop} is replaced by the confidence exit
above. The two \emph{state} rows feed the residual network only the masks
$\mathbf1_{S_s}\Vert\mathbf1_{\mathcal A}$, or only the posterior summary
$p_s\Vert H(p_s)\Vert m_s$, instead of the full state of
Equation~\ref{eq:rlstate}.

Three observations complete the ablation discussion of Section~\ref{sec:experiments}. First,
learning the controller matters more than any single part of it: the static
Shapley order with a confidence exit, which uses the same marginal-gain
information without learning, loses $3.68\%$ at a low cost (MU $.33$), and
replacing only the learned \textsc{stop} by the confidence exit loses $3.44\%$.
Without the prior, fitted-$Q$ iteration still learns a usable controller but
under-acquires (MU $.56$ versus $.65$) and loses $1.47\%$, so the prior mainly
sets how much evidence is bought before stopping. Second, both parts of the
state matter: masks alone acquire more (MU $.70$) and still lose $3.44\%$,
while the posterior summary alone stays within $0.49\%$ at a higher cost
(MU $.69$), so the posterior tells the controller when to stop and the masks
mainly refine what to acquire. Third, only two changes do not reduce
macro-F1. The learnable aggregation gate changes it by $+0.14\%$, well within
its fold standard deviation of $.040$. Removing \textsc{stop} gains $0.65\%$,
but it acquires every modality and raises $\mathrm{GF}_{\mathrm f}$ by $53\%$,
which is precisely the cost \method{} is designed to avoid. On the backbone
side, the dropout curriculum has the largest effect once the controller is
attached ($-4.76\%$, the largest drop in the table), even though the backbone
alone loses only $1.4$ points, because without it the prefixes that \policy{}
stops on are poorly calibrated.

\begin{table}[t]
\centering
\small
\captionsetup{font=normalsize,skip=2pt}
\caption{Complete \backbone{} and \policy{} ablations on IEMOCAP under full
modality availability (mean over three test folds; subscripts are standard
deviations over folds). ``\backbone{} F1'' is the retrained backbone alone on
all modalities, averaged over 24 orders; 'F1' is the backbone and policy both retrained; $\Delta$ is the relative change of
the controlled macro-F1 with respect to $.669$.
\textcolor{burgundy}{\textbf{Bold}} marks the best value per metric.}
\label{tab:iemocap-ablation}
\setlength{\tabcolsep}{4pt}
\begin{tabular}{@{}llccccc@{}}
\toprule
\textbf{Category} & \textbf{Variant}
& \textbf{\backbone{} F1} & \textbf{F1}$\uparrow$ & $\Delta$
& $\mathrm{GF}_{\mathrm f}\downarrow$ & \textbf{MU}$\downarrow$ \\
\midrule
\multicolumn{7}{@{}l}{\emph{\backbone{} ablations}} \\
\multirow{2}{*}{Training}
& Fixed modality order       & .653 & .644$_{\pm.033}$ & $-3.69\%$ & .34 & .79 \\
& No dropout curriculum      & .661 & .637$_{\pm.043}$ & $-4.76\%$ & .30 & .70 \\
\multirow{2}{*}{Fusion}
& Bottleneck tokens last     & .651 & .645$_{\pm.018}$ & $-3.54\%$ & .29 & .67 \\
& Learnable aggregation      & \textcolor{burgundy}{\textbf{.675}} & .670$_{\pm.040}$ & $+0.14\%$ & .27 & .62 \\
Loss
& $\mathcal{L}_{\mathrm{CE}}$ only & .656 & .652$_{\pm.014}$ & $-2.46\%$ & .31 & .72 \\
\arrayrulecolor{black!20}\midrule
\arrayrulecolor{black}
\multicolumn{7}{@{}l}{\emph{\policy{} ablations}} \\
Static controller
& Shapley + conf.\ exit      & -- & .644$_{\pm.014}$ & $-3.68\%$ & \textcolor{burgundy}{\textbf{.14}} & \textcolor{burgundy}{\textbf{.33}} \\
Value init.
& Zero ($Q_0\equiv0$)        & -- & .659$_{\pm.012}$ & $-1.47\%$ & .24 & .56 \\
\multirow{2}{*}{Action}
& No STOP (acquire all)      & -- & \textcolor{burgundy}{\textbf{.673}}$_{\pm.012}$ & $+0.65\%$ & .43 & 1.00 \\
& Conf.\ STOP                & -- & .646$_{\pm.017}$ & $-3.44\%$ & \textcolor{burgundy}{\textbf{.14}} & \textcolor{burgundy}{\textbf{.33}} \\
\multirow{2}{*}{State}
& $\mathbf{1}_{S_s}\Vert\mathbf{1}_{\mathcal A}$ only
                             & -- & .646$_{\pm.047}$ & $-3.44\%$ & .30 & .70 \\
& $p_s\Vert H(p_s)\Vert m_s$ only
                             & -- & .666$_{\pm.014}$ & $-0.49\%$ & .30 & .69 \\
\arrayrulecolor{black}\midrule
\rowcolor{tolblue!10}
\multicolumn{2}{@{}l}{\textbf{\method{}}}
                             & \textcolor{burgundy}{\textbf{.675}} & .669$_{\pm.022}$ & -- & .28 & .65 \\
\bottomrule
\end{tabular}
\end{table}

\section{Complete Results for Figure~\ref{fig:diag-reward}} \label{app:reward-sensitivity}

Figure~\ref{fig:diag-reward} varies only the relative acquisition costs $w_m$
of Equation~\ref{eq:rlreturn}. We derive them from each modality's measured
per-sample extraction time $t_m$ on the same idle NVIDIA L40 server used for
all latency measurements,
\begin{equation}\label{eq:wgamma}
w_m(\gamma)=\Big(\frac{t_m}{\min_{m'}t_{m'}}\Big)^{\gamma},
\end{equation}
and price an acquisition exactly as in Equations~\ref{eq:rlreturn}
and~\ref{eq:target}, $\lambda w_j/\sum_{m\in\mathcal A}w_m$. Thus $\gamma=0$ is
the uniform price $\lambda/|\mathcal A|$ of \method{}, $\gamma=1$ splits
$\lambda$ in proportion to extraction time, and larger $\gamma$ concentrates
the price on the most expensive modality. Acquiring every available modality
always costs $\lambda$, whatever $\gamma$.

\paragraph{Protocol.} Everything except $\gamma$ is held at the settings of the
deployed uniform-cost \method{} policy, fold by fold and dataset by dataset:
the per-fold seed, the per-fold training price $\lambda_{\mathrm{train}}$
($0.5$ on every fold except MM-Fi fold~2, $1.0$), the fitted-$Q$ schedule and
the availability augmentation (Appendix~\ref{app:rl_policy}). As in the
deployed policy, the prior $Q_0$ is priced at $\lambda_{\mathrm{train}}$ both
during fitting and at inference, and actions are selected greedily. For every $\gamma>0$, \policy{} is refitted on the
deployed backbone with these settings; for $\gamma=0$ the price reduces to
$\lambda/|\mathcal A|$ and we use the deployed networks themselves, so the
$\gamma=0$ results are those of Table~\ref{tab:full-results}. Latency is measured end to end per sample on the idle L40, from raw-file
reading to returned logits, including extraction, fusion, and policy decisions,
using the timing protocol in Appendix~\ref{app:energy-latency-complete}.

Table~\ref{tab:wgamma-weights} attributes every weight and price to its
modality. Extraction times on IEMOCAP span three orders of magnitude, from
1.1\,ms for the head MoCap stream to 1{,}433\,ms for video (DINOv3-L on 100
frames), so video carries 69\% of the total price already at $\gamma=0.5$ and
86\% at $\gamma=0.75$, while text and the MoCap streams become almost free.

\begin{table}[t]
\centering
\small
\captionsetup{font=normalsize,skip=2pt}
\caption{Per-modality extraction time $t_m$ (idle L40, ms per sample), relative
cost $w_m(\gamma)$ of Equation~\ref{eq:wgamma}, and the resulting price of
acquiring each modality at $\lambda=0.5$ with all modalities available
($\lambda w_m/\sum_{m'}w_{m'}$) on IEMOCAP.}
\label{tab:wgamma-weights}
\setlength{\tabcolsep}{3.5pt}
\begin{tabular}{@{}lrrrrrr@{}}
\toprule
& video & audio & text & MoCap hand & MoCap head & MoCap rot. \\
\midrule
$t_m$ (ms) & 1432.6 & 64.7 & 4.1 & 2.6 & 1.1 & 18.4 \\
\midrule
\multicolumn{7}{@{}l}{\emph{Relative cost $w_m(\gamma)$}} \\
$\gamma=0$    & 1 & 1 & 1 & 1 & 1 & 1 \\
$\gamma=0.25$ & 5.99 & 2.76 & 1.39 & 1.23 & 1 & 2.02 \\
$\gamma=0.5$  & 35.9 & 7.62 & 1.93 & 1.52 & 1 & 4.07 \\
$\gamma=0.75$ & 215 & 21.0 & 2.68 & 1.88 & 1 & 8.20 \\
$\gamma=1$    & 1287 & 58.1 & 3.72 & 2.31 & 1 & 16.5 \\
$\gamma=1.5$  & $4.6{\times}10^{4}$ & 443 & 7.18 & 3.52 & 1 & 67.2 \\
$\gamma=2$    & $1.7{\times}10^{6}$ & 3376 & 13.9 & 5.35 & 1 & 273 \\
\midrule
\multicolumn{7}{@{}l}{\emph{Price at $\lambda=0.5$}} \\
$\gamma=0$    & .083 & .083 & .083 & .083 & .083 & .083 \\
$\gamma=0.25$ & .208 & .096 & .048 & .043 & .035 & .070 \\
$\gamma=0.5$  & .345 & .073 & .019 & .015 & .010 & .039 \\
$\gamma=0.75$ & .430 & .042 & .005 & .004 & .002 & .016 \\
$\gamma=1$    & .470 & .021 & .001 & .001 & .000 & .006 \\
$\gamma=1.5$  & .494 & .005 & .000 & .000 & .000 & .001 \\
$\gamma=2$    & .499 & .001 & .000 & .000 & .000 & .000 \\
\bottomrule
\end{tabular}
\end{table}

Table~\ref{tab:wgamma-results} lists the complete results behind
Figure~\ref{fig:diag-reward}. As the price moves onto video, \policy{} stops
acquiring it, from 47\% of samples at $\gamma=0$ to 11\% at $\gamma=0.5$ and
1\% at $\gamma=0.75$, and acquires the nearly free MoCap streams on 80--97\%
of samples instead. At $\gamma=0.75$ this keeps macro-F1 unchanged
($.670$ versus $.669$) while end-to-end latency falls from 745 to 87\,ms
($-88.3\%$), because video is by far the most expensive modality and adds
little once text and audio are held. Modality utilization rises from $.65$ to
$.75$ at the same time, which illustrates that MU counts modalities rather than
cost: the extra modalities are the ones that cost a few milliseconds. Beyond
$\gamma=0.75$ the cheap modalities are effectively free, so the price no longer
discourages any of them and macro-F1 decreases slightly ($.651$--$.658$). The
large latency spreads at $\gamma=0.25$ and $\gamma=2$ come from single folds
that still acquire video, on 43\% and 73\% of samples (folds 0 and 2) at
$\gamma=0.25$ and on 29\% (fold 0) at $\gamma=2$.

\begin{table}[t]
\centering
\scriptsize
\captionsetup{font=normalsize,skip=2pt}
\caption{Complete results for Figure~\ref{fig:diag-reward} on IEMOCAP
($\lambda=0.5$ with the prior priced at $\lambda$, deployed seeds; $\gamma=0$ is the deployed \method{} policy): macro-F1 and end-to-end
GPU latency (mean$_{\pm\mathrm{std}}$ over three test folds), modality
utilization, post-extraction GFLOPs, the extraction part of the latency, and
the fraction of test samples on which each modality is acquired.}
\label{tab:wgamma-results}
\setlength{\tabcolsep}{3pt}
\begin{tabular}{@{}lcccrrcccccc@{}}
\toprule
& & & & & & \multicolumn{6}{c}{\textbf{Acquisition frequency}} \\
\cmidrule(l){7-12}
$\gamma$ & \textbf{F1}$\uparrow$ & \textbf{MU} & $\mathrm{GF}_{\mathrm f}\downarrow$
& \textbf{Lat.\ (ms)}$\downarrow$ & \textbf{Extr.}
& vid. & aud. & text & hand & head & rot. \\
\midrule
\rowcolor{tolblue!10}
0     & .669$_{\pm.022}$ & .65 & .28 & 745$_{\pm132}$ & 740 & .47 & .72 & 1.00 & .56 & .56 & .59 \\
0.25  & .672$_{\pm.007}$ & .79 & .34 & 630$_{\pm519}$ & 624 & .39 & .81 & 1.00 & .94 & .94 & .67 \\
0.5   & .658$_{\pm.019}$ & .72 & .30 & 225$_{\pm141}$ & 219 & .11 & .71 & 1.00 & .87 & .88 & .74 \\
0.75  & .670$_{\pm.011}$ & .75 & .31 & 87$_{\pm26}$ & 81 & .01 & .75 & 1.00 & .97 & .95 & .80 \\
1     & .653$_{\pm.007}$ & .71 & .30 & 81$_{\pm17}$ & 76 & .00 & .81 & .90 & .80 & .86 & .88 \\
1.5   & .651$_{\pm.013}$ & .71 & .30 & 82$_{\pm17}$ & 76 & .00 & .83 & .87 & .85 & .86 & .87 \\
2     & .658$_{\pm.020}$ & .76 & .32 & 230$_{\pm234}$ & 224 & .10 & .97 & .88 & .79 & .86 & .94 \\
\bottomrule
\end{tabular}
\end{table}

Table~\ref{tab:wgamma-alldata} repeats the study on all six datasets. The
effect follows how unevenly each dataset's extraction cost is spread. Where one
modality dominates the cost and is not needed by every sample, cost weighting
removes most of the latency at little or no accuracy cost: at $\gamma=0.75$,
IEMOCAP keeps $.670$ at 87 instead of 745\,ms, EAV improves to $.694$ at 114
instead of 217\,ms (29\,ms at $\gamma=0.25$), and MM-Fi keeps $.952$ at 69
instead of 187\,ms. On UTD-MHAD, where the two camera streams together make up
nearly all of the cost, latency falls by 19\% (659 to 537\,ms) for a 3.9-point
loss. CZU-MHAD's predictions rest on its depth stream, which \policy{}
acquires at every $\gamma$, so latency stays at 1.3\,s while macro-F1 rises
slightly (.914 to .922 at $\gamma\ge1.5$). CMI's seven sensor streams all cost
0.7--1.5\,ms, so the weights are nearly equal and $\gamma$ has no systematic
effect (.548--.556). Averaged over the six datasets, $\gamma=0.75$ lowers
latency by 32\% (520 to 355\,ms) at $-0.6$ points of macro-F1.

\begin{table}[t]
\centering
\scriptsize
\captionsetup{font=normalsize,skip=2pt}
\caption{Macro-F1 / end-to-end GPU latency (ms) for every $\gamma$ on the six
datasets, each fold at its deployed uniform policy's seed and
$\lambda_{\mathrm{train}}$, with the prior priced at $\lambda_{\mathrm{train}}$
 (mean over three test folds). $\gamma=0$ is the deployed
\method{} policy.}
\label{tab:wgamma-alldata}
\setlength{\tabcolsep}{3pt}
\begin{tabular}{@{}lccccccc@{}}
\toprule
& $\gamma=0$ & $0.25$ & $0.5$ & $0.75$ & $1$ & $1.5$ & $2$ \\
\midrule
IEMOCAP   & .669 / 745 & .672 / 630 & .658 / 225 & .670 / 87 & .653 / 81 & .651 / 82 & .658 / 230 \\
MM-Fi     & .950 / 187 & .946 / 58 & .954 / 62 & .952 / 69 & .949 / 139 & .947 / 134 & .949 / 97 \\
CMI       & .556 / 12.6 & .548 / 10.1 & .550 / 10.4 & .552 / 12.2 & .555 / 12.1 & .554 / 12.2 & .548 / 12.4 \\
CZU-MHAD  & .914 / 1301 & .914 / 1301 & .920 / 1307 & .916 / 1314 & .920 / 1313 & .922 / 1315 & .922 / 1315 \\
EAV       & .689 / 217 & .690 / 29 & .697 / 156 & .694 / 114 & .691 / 56 & .691 / 55 & .692 / 100 \\
UTD-MHAD  & .827 / 659 & .817 / 617 & .800 / 599 & .788 / 537 & .791 / 544 & .791 / 539 & .797 / 566 \\
\midrule
Mean      & .767 / 520 & .764 / 441 & .763 / 393 & .762 / 355 & .760 / 358 & .759 / 356 & .761 / 387 \\
\bottomrule
\end{tabular}
\end{table}

Finally, Table~\ref{tab:wgamma-lambda} varies the training price on IEMOCAP,
$\lambda\in\{0.05,0.1,0.2,0.5,1,2\}$, with the prior priced at the training
$\lambda$ in every cell, so each cell differs from its
neighbours in exactly one knob; the $\lambda=0.5$ column is the deployed
setting of Table~\ref{tab:wgamma-results}. The weights matter only
when $\lambda$ is large enough for the price to bind. For $\lambda\le0.1$ even
the most expensive modality costs at most $\lambda$, which rarely outweighs its
expected gain, so latency stays between 0.53 and 1.25\,s for every $\gamma$;
from $\lambda=0.2$ the weighted price starts to bind, and at $\lambda=0.5$
$\gamma=0.75$ reaches the macro-F1 of the uniform cost at under an eighth of
its latency. Across $\lambda\le0.2$ macro-F1 stays within $.648$--$.678$ for
every $\gamma$. At the highest prices the uniform cost fails: with $\lambda=1$
or $2$ it charges every modality alike, so \policy{} stops after one or two
modalities (MU $.36$ and $.17$) and macro-F1 falls to $.624$ and $.598$,
whereas $\gamma\ge0.75$ concentrates the price on video and keeps macro-F1 at
$.640$--$.672$ within 42--92\,ms. In short, nonuniform costs let a single
$\lambda$ remove the expensive modality without also discouraging the cheap
informative ones.

\begin{table}[t]
\centering
\small
\captionsetup{font=normalsize,skip=2pt}
\caption{Macro-F1 / end-to-end GPU latency (ms) on IEMOCAP for every
$(\gamma,\lambda)$ pair (mean over three test folds, deployed seeds, prior
priced at $\lambda$). The shaded cell is the deployed \method{}
policy.}
\label{tab:wgamma-lambda}
\setlength{\tabcolsep}{3.5pt}
\begin{tabular}{@{}lcccccc@{}}
\toprule
$\gamma\ \backslash\ \lambda$ & 0.05 & 0.1 & 0.2 & 0.5 & 1 & 2 \\
\midrule
0    & .661 / 1053 & .675 / 994 & .668 / 792 & \cellcolor{tolblue!10}.669 / 745 & .624 / 238 & .598 / 5.5 \\
0.25 & .648 / 1207 & .670 / 733 & .672 / 877 & .672 / 630 & .620 / 23 & .606 / 23 \\
0.5  & .664 / 751 & .672 / 569 & .678 / 818 & .658 / 225 & .640 / 44 & .622 / 29 \\
0.75 & .678 / 1093 & .676 / 535 & .678 / 483 & .670 / 87 & .662 / 63 & .640 / 42 \\
1    & .672 / 1097 & .677 / 677 & .671 / 93 & .653 / 81 & .656 / 89 & .660 / 65 \\
1.5  & .670 / 695 & .676 / 1082 & .671 / 273 & .651 / 82 & .658 / 92 & .647 / 81 \\
2    & .668 / 1245 & .674 / 1082 & .672 / 633 & .658 / 230 & .672 / 90 & .656 / 89 \\
\bottomrule
\end{tabular}
\end{table}


\section{Theoretical Support for Backbone--Controller Co-Design}
\label{app:theory}

This appendix proves Proposition~\ref{thm:order_robustness}, supplies the accompanying same-forward distillation bound, derives Corollary~\ref{thm:ais_control} from the finite-horizon AIS result of \citet{subramanian2022ais}, and establishes the two interpretations of the acquisition prior in Section~\ref{sec:m_rl}. Throughout, logarithms are natural and $\operatorname{TV}(P,Q)=\sup_E|P(E)-Q(E)|$; for categorical distributions, this equals $\tfrac12\sum_k|P(k)-Q(k)|$. Expectations in each result use the sampling distribution specified there.

\subsection{Order robustness of \backbone{}}
\label{app:order}

We condition on a fixed available set $\mathcal A$ and evaluate a fixed backbone. Let $\mathcal B\subseteq\mathcal A$ be a nonempty dropout-retained set, and let $\Pi(S)$ denote the permutations of a nonempty acquired set $S$. The retained set is sampled independently of $(X,Y)$ conditional on $\mathcal A$; conditional on that retained set, its permutation is sampled uniformly and independently of the sample. A prefix prediction depends only on the acquired inputs and their processing order. We suppress conditioning on $\mathcal A$ in expectations below.

\paragraph{Uniform sampling of prefix subsets and orders.}
Fix $\mathcal B$, let $n=|\mathcal B|$, and draw $\omega$ uniformly from $\Pi(\mathcal B)$. For nonempty $S\subseteq\mathcal B$ with $|S|=s$, exactly $s!(n-s)!$ full orders have $S$ as their first-$s$ set. Therefore,
\begin{equation}
\mathbb P\!\left(\{\omega_1,\ldots,\omega_s\}=S\mid\mathcal B\right)=\frac{s!(n-s)!}{n!}=\binom ns^{-1}.
\label{eq:app_uniform_prefix}
\end{equation}
Conditional on this event, each ordered prefix in $\Pi(S)$ has $(n-s)!$ completions and hence probability $1/s!$. This proves the main-text statement that randomized ordering samples acquired subsets uniformly conditional on the retained set and prefix length, and samples their internal orders uniformly. It does not require different retained sets or prefix lengths to receive equal training weight.

\begin{proof}[Proof of Proposition~\ref{thm:order_robustness}]
Fix a nonempty $S\subseteq\mathcal A$, and define the common label posterior $\eta_S(k)=\mathbb P(Y=k\mid X_S,\mathcal A)$. Let $\sigma$ be uniform on $\Pi(S)$ and independent of the sample. Conditioning on $(X_S,\sigma)$ gives
\begin{align}
\mathcal R_S:=\mathbb E[-\log p_S^\sigma(Y)]
&=\mathbb E\!\left[-\sum_{k=1}^{C}\eta_S(k)\log p_S^\sigma(k)\right]\notag\\
&=H(Y\mid X_S,\mathcal A)+\mathbb E\!\left[\operatorname{KL}(\eta_S\|p_S^\sigma)\right].
\label{eq:app_prefix_risk}
\end{align}
Thus, the excess prefix cross-entropy in the theorem is $\epsilon_S=\mathbb E[\operatorname{KL}(\eta_S\|p_S^\sigma)]$. For two independent orders $\sigma,\tau$ evaluated on the same acquired evidence, the triangle inequality gives
\begin{equation*}
\operatorname{TV}(p_S^\sigma,p_S^\tau)\leq\operatorname{TV}(p_S^\sigma,\eta_S)+\operatorname{TV}(\eta_S,p_S^\tau).
\end{equation*}
The two terms have the same expectation. Applying Pinsker's inequality~\citep{cover2006elements} and then Jensen's inequality yields
\begin{align}
D_S:=\mathbb E\!\left[\operatorname{TV}(p_S^\sigma,p_S^\tau)\right]
&\leq 2\mathbb E\!\left[\operatorname{TV}(\eta_S,p_S^\sigma)\right]\notag\\
&\leq 2\mathbb E\!\left[\sqrt{\tfrac12\operatorname{KL}(\eta_S\|p_S^\sigma)}\right]\notag\\
&\leq\sqrt{2\epsilon_S}.
\label{eq:app_order_bound}
\end{align}
The same argument also gives
\begin{equation}
\mathbb E\!\left[\operatorname{TV}(p_S^\sigma,\eta_S)\right]\leq\sqrt{\epsilon_S/2},
\label{eq:app_posterior_bound}
\end{equation}
which establishes the main-text implication that low excess prefix risk brings predictions toward a common, order-independent posterior on average.
\end{proof}

The dropout curriculum changes the retained-set distribution, while random ordering and prefix supervision expose each retained set's prefixes to the common label target. The theorem concerns the population prefix risks attained by a fixed predictor; it does not assert that optimizing the combined classification and distillation objective attains the Bayes minimum for every subset. Its expectations are over the specified sample/order distribution, not the distribution induced by conditioning on an adaptive controller visiting $S_s=S$.

\paragraph{Supporting bound for same-forward distillation.}
Fix nonempty $S\subseteq\mathcal B\subseteq\mathcal A$, with $s=|S|$, and consider the full orders
\begin{equation*}
\Omega(\mathcal B,S)=\{\omega\in\Pi(\mathcal B):\{\omega_1,\ldots,\omega_s\}=S\}.
\end{equation*}
Draw $\omega,\nu$ independently and uniformly from $\Omega(\mathcal B,S)$, independently of the sample, and evaluate both orders on the same full retained input $X_{\mathcal B}$. Write $p_S^{\omega_{1:s}}$ for the prefix prediction and $q^\omega=\operatorname{sg}(p_{\mathcal B}^\omega)$ for its same-forward full-prefix teacher. Stop-gradient leaves the numerical teacher distribution unchanged. Define
\begin{align}
K_{S\mid\mathcal B}
&=\mathbb E_{X_{\mathcal B},\omega}\!\left[\operatorname{KL}\!\left(q^\omega\|p_S^{\omega_{1:s}}\right)\right],\notag\\
D_{\mathrm{full},S\mid\mathcal B}
&=\mathbb E_{X_{\mathcal B},\omega,\nu}\!\left[\operatorname{TV}(q^\omega,q^\nu)\right].
\label{eq:app_distillation_terms}
\end{align}
Both quantities use the same prefix-set-conditioned order ensemble and average over the full retained inputs. By independent retention and the uniform-prefix argument above, the induced prefix orders are independent and uniform on $\Pi(S)$, and their expected disagreement is the same $D_S$ as in Equation~\ref{eq:app_order_bound}. We claim that
\begin{equation}
D_S\leq D_{\mathrm{full},S\mid\mathcal B}+\sqrt{2K_{S\mid\mathcal B}}.
\label{eq:app_distillation_bound}
\end{equation}

\begin{proof}[Proof of Equation~\ref{eq:app_distillation_bound}]
Insert the two teachers between the prefix predictions:
\begin{align*}
\operatorname{TV}\!\left(p_S^{\omega_{1:s}},p_S^{\nu_{1:s}}\right)
&\leq\operatorname{TV}\!\left(p_S^{\omega_{1:s}},q^\omega\right)+\operatorname{TV}(q^\omega,q^\nu)\\
&\qquad+\operatorname{TV}\!\left(q^\nu,p_S^{\nu_{1:s}}\right).
\end{align*}
The middle term has expectation $D_{\mathrm{full},S\mid\mathcal B}$. By symmetry, Pinsker's inequality, and Jensen's inequality, the expectations of the outer terms sum to at most
\begin{equation*}
2\mathbb E\!\left[\sqrt{\tfrac12\operatorname{KL}\!\left(q^\omega\|p_S^{\omega_{1:s}}\right)}\right]\leq\sqrt{2K_{S\mid\mathcal B}}.
\end{equation*}
This proves the claim.
\end{proof}

Together with Proposition~\ref{thm:order_robustness}, this gives
\begin{equation}
D_S\leq\min\!\left\{\sqrt{2\epsilon_S},\;D_{\mathrm{full},S\mid\mathcal B}+\sqrt{2K_{S\mid\mathcal B}}\right\}.
\label{eq:app_combined_order_bound}
\end{equation}
The first bound follows from prefix supervision alone; the second supports the accompanying claim that same-forward distillation transfers full-prefix stability when its conditional teacher disagreement and prefix-to-teacher divergence are small. An unconditional full-order risk cannot replace $D_{\mathrm{full},S\mid\mathcal B}$ without an additional argument. These results concern average predictive distributions, not identical hidden states or AIS transition sufficiency. Aggregate macro-F1 variation is complementary empirical evidence rather than a measurement of these bounds.

\subsection{AIS guarantee for modality acquisition}
\label{app:ais}

This subsection proves Corollary~\ref{thm:ais_control}. Let $H_t=(\mathcal A,\sigma_{1:t},X_{\sigma_1},\ldots,X_{\sigma_t})$ denote the complete acquired history at a decision step $t\geq1$. It excludes the label and unacquired inputs. With the backbone frozen, $\bm\psi_t$ is a deterministic function of $H_t$; a fixed rule resolves prediction ties.

Let $\widehat r_{\mathrm{stop}}$ and $\widehat P$ be the stopping-reward function and acquisition kernel of the compact model. Assume that, for every history reachable under any admissible acquisition policy,
\begin{align}
\left|\mathbb E[R_t\mid H_t]-\widehat r_{\mathrm{stop}}(\bm\psi_t)\right|
&\leq\epsilon_{\mathrm r},\notag\\
\operatorname{TV}\!\left(\mathcal L(\bm\psi_{t+1}\mid H_t,a),\widehat P(\cdot\mid\bm\psi_t,a)\right)
&\leq\epsilon_{\mathrm p}
\label{eq:app_ais_conditions}
\end{align}
for every legal acquisition $a=\operatorname{acq}(j)$, where the second line is the controlled law obtained by taking $a$. Both processes preserve $\mathcal A$, use the same acquired-set update $S_{t+1}=S_t\cup\{j\}$, and assign the same acquisition reward $-\lambda w_j/\sum_{m\in\mathcal A}w_m$. STOP is always admissible, is the only action when $S_t=\mathcal A$, and its compact-model reward satisfies $0\leq\widehat r_{\mathrm{stop}}\leq\beta_{\max}$.

\paragraph{Proof strategy.}
The finite-horizon AIS framework uses a common action space. We therefore first rewrite the stopping problem as an equivalent fixed-horizon process with a common action alphabet. We then verify its reward and transition approximation errors and apply Theorem~9 of \citet{subramanian2022ais}.

\begin{proof}[Proof of Corollary~\ref{thm:ais_control}]
Fix a reachable history $H_s$ and let $n=|\mathcal A|-s$. Values below include the terminal reward and future acquisition costs; costs already incurred at $H_s$ are common to both policies.

\emph{Step 1: Construct an equivalent fixed-horizon process.}
Use $n+1$ decision stages. Acquisitions and their costs are unchanged. If STOP is selected early, move to a marked stopped state, take zero-reward no-ops until the final stage, and then pay the reward associated with the prediction retained at stopping. A trajectory that acquires all remaining modalities takes its forced STOP at the final stage. Because the objective is undiscounted, this padding leaves every policy's return unchanged.

To obtain a common action space, define otherwise-illegal active action labels as duplicates of STOP, and define all action labels in stopped states as duplicates of the forced no-op. These duplicate labels do not change values or the behavior of any legal action. Thus the padded process is value-equivalent to the original acquisition problem and has the form required by the finite-horizon AIS theorem.

\emph{Step 2: Verify the AIS errors and apply the theorem.}
During the first $n$ stages, rewards are exact acquisition costs or zero, so the reward-approximation error is zero. Only the final stopping reward contributes $\epsilon_{\mathrm r}$. Acquisition transitions satisfy the bound in Equation~\ref{eq:app_ais_conditions}; STOP and stopped-state no-ops have identical deterministic compact-state transitions.

At every stage, the compact model's optimal continuation value satisfies $0\leq\widehat V_\ell^*\leq\beta_{\max}$: stopping avoids further acquisition costs and yields a nonnegative reward, while the single terminal reward is bounded above by $\beta_{\max}$. Hence $\operatorname{span}(\widehat V_\ell^*)\leq\beta_{\max}$. Under our TV convention,
\begin{equation*}
\left|\int f\,dP-\int f\,dQ\right|\leq\operatorname{span}(f)\operatorname{TV}(P,Q),
\qquad\operatorname{span}(f)=\sup f-\inf f.
\end{equation*}
Therefore, Theorem~9 of \citet{subramanian2022ais} gives
\begin{equation}
\alpha\leq\epsilon_{\mathrm r}+\sum_{\ell=1}^{n}\operatorname{span}(\widehat V_{\ell+1}^*)\epsilon_{\mathrm p}
\leq\epsilon_{\mathrm r}+n\beta_{\max}\epsilon_{\mathrm p}.
\label{eq:app_ais_value_error}
\end{equation}
The source uses the full-$L_1$ convention for TV; its transition error is therefore $2\epsilon_{\mathrm p}$ and its multiplier is $\tfrac12\operatorname{span}(\widehat V_{\ell+1}^*)$, yielding the same product.

Let $\widehat V^*(\bm\psi_s)$ be the compact model's optimal remaining value and let $\widehat\pi$ be one of its optimal policies, executed in the original process through $\bm\psi$. The same AIS result gives
\begin{equation*}
\left|V^\star(H_s)-\widehat V^*(\bm\psi_s)\right|\leq\alpha,
\qquad
\left|V^{\widehat\pi}(H_s)-\widehat V^*(\bm\psi_s)\right|\leq\alpha.
\end{equation*}
Combining these inequalities with optimality of the full-history policy yields
\begin{equation}
0\leq V^\star(H_s)-V^{\widehat\pi}(H_s)
\leq2\alpha
\leq2\left[\epsilon_{\mathrm r}+(|\mathcal A|-s)\beta_{\max}\epsilon_{\mathrm p}\right].
\label{eq:app_ais_policy_bound}
\end{equation}
Value equivalence between the padded and original processes gives the stated corollary.
\end{proof}
\subsection{Shapley interpretation of the acquisition prior}
\label{app:qprior}

This subsection supports the acquisition-prior interpretation in Section~\ref{sec:m_rl}. We first establish an exact empirical connection to Shapley attribution, then state the conditions under which posterior-weighted gains equal conditional expected gains. Condition on the fitted backbone and cache. The Shapley identity uses a common offline cohort for which $v_i(S)$ is defined for every $S\subseteq[M]$, with at least one sample from each class represented in the tables. The empirical class average uses the same samples and weights for every coalition. We set $v_i(\varnothing)=0$ solely to complete the attribution game; this does not define a backbone prediction at the empty prefix or a first-acquisition rule. Recall that $\Delta_{i,j}(S)=v_i(S\cup\{j\})-v_i(S)$ and $[\bm T_S]_{k,j}=\mathbb E_i[\Delta_{i,j}(S)\mid y_i=k]$.

For the attribution identity, the empty-set convention also defines $[\bm T_{\varnothing}]_{k,j}=\mathbb E_i[v_i(\{j\})\mid y_i=k]$; this entry does not specify the controller's first acquisition. Let $\phi_{i,j}$ denote the Shapley value of modality $j$ for the set utility $v_i$. Then
\begin{equation}\label{eq:shapley_relation}
\mathbb E_i[\phi_{i,j}\mid y_i=k]
=
\sum_{S\subseteq[M]\setminus\{j\}}
\frac{|S|!(M-|S|-1)!}{M!}\,[\bm T_S]_{k,j}.
\end{equation}

\begin{proof}[Proof of Equation~\ref{eq:shapley_relation}]
For the set utility $v_i$, the Shapley value~\citep{shapley1953value} is
\begin{equation*}
\phi_{i,j}=\sum_{S\subseteq[M]\setminus\{j\}}\frac{|S|!\,(M-|S|-1)!}{M!}\,\Delta_{i,j}(S).
\end{equation*}
Interchanging the finite class average and coalition sum gives
\begin{align*}
\mathbb E_i[\phi_{i,j}\mid y_i=k]
&=\sum_{S\subseteq[M]\setminus\{j\}}\frac{|S|!\,(M-|S|-1)!}{M!}\,\mathbb E_i[\Delta_{i,j}(S)\mid y_i=k]\\
&=\sum_{S\subseteq[M]\setminus\{j\}}\frac{|S|!\,(M-|S|-1)!}{M!}\,[\bm T_S]_{k,j}.
\end{align*}
This proves Equation~\ref{eq:shapley_relation} exactly for the empirical tables. It requires neither posterior calibration nor exact order invariance: the order average already defines a set utility.
\end{proof}

 The common-cohort requirement matters: averaging each coalition over a different eligible sample population does not automatically preserve this sample-level identity. If the attribution game is restricted to a fixed available set $\mathcal A$, the same proof applies to $j\in\mathcal A$ using a common cohort, summing over $S\subseteq\mathcal A\setminus\{j\}$ and replacing $M$ by $|\mathcal A|$ in the Shapley weights.

\paragraph{Conditional expected-gain interpretation.}
For a random sample from the population of interest, let $\Delta_j(S)$ be its order-averaged set-utility difference. For a legal acquisition $j\in\mathcal A\setminus S_s$, suppose
\begin{equation}
p_s(k)=\mathbb P(Y=k\mid\bm\psi_s),
\qquad
\mathbb E[\Delta_j(S_s)\mid\bm\psi_s,Y=k]=[\bm T_{S_s}]_{k,j},
\label{eq:app_gain_assumptions}
\end{equation}
where the second condition is required for classes with positive conditional probability. Because $S_s$ is determined by $\bm\psi_s$, the law of total expectation yields
\begin{align}
\mathbb E[\Delta_j(S_s)\mid\bm\psi_s]
&=\sum_{k=1}^{C}\mathbb P(Y=k\mid\bm\psi_s)\,\mathbb E[\Delta_j(S_s)\mid\bm\psi_s,Y=k]\notag\\
&=\sum_{k=1}^{C}p_s(k)[\bm T_{S_s}]_{k,j}
=[p_s^\top\bm T_{S_s}]_j.
\label{eq:app_posterior_gain}
\end{align}
This establishes the conditional expected-gain interpretation under Equation~\ref{eq:app_gain_assumptions}. Its conditional-mean assumption includes the availability and acquired-evidence information contained in $\bm\psi_s$; conditioning only on class and set when defining the table does not establish that assumption.

To distinguish estimation from this interpretation, the natural population table is $[\bm T_S^{\mathrm{pop}}]_{k,j}=\mathbb E[\Delta_j(S)\mid Y=k]$. The same proof applies with $\bm T^{\mathrm{pop}}$ when the corresponding conditional-mean assumption holds. The implemented $\bm T$ estimates such class means from the cache. The empirical Shapley identity is exact, but neither empirical averaging nor Equation~\ref{eq:shapley_relation} establishes equality with population means or correctness of $p_s$; using the empirical table in Equation~\ref{eq:app_posterior_gain} requires its stated assumptions.

Finally, $\Delta_j(S)$ compares order-averaged set utilities, so it need not equal the realized gain from appending $j$ to the current order. These utilities use unweighted correctness, and $\operatorname{conf}(p_s)$ is an unweighted confidence proxy. The prior is therefore an initialization; the residual is trained against the class-weighted, multistep return defined in Equation~\ref{eq:rlreturn}. The identities above do not make $Q_0$ a Bellman-optimal value function or imply that fitted-$Q$ learning eliminates its approximation errors.

\end{document}